\documentclass[a4paper, oneside, 11pt]{article}
\usepackage{mathstyle}
\usepackage{defs}
\usepackage{dutchcal}

\NewDocumentCommand{\e}{m m m}{e_{#1, #2}^{(#3)}}

\newboolean{isOnBorisLaptop}
\IfFileExists{~/boris_machine_marker}{
  \setboolean{isOnBorisLaptop}{false}
}{%
  \setboolean{isOnBorisLaptop}{false}
}

\usepackage{crossreftools}
\makeatletter
\def\crt@cref@extractfields#1#2#3\@nil#4{%
  \csname crt@cref@splitter@#4\endcsname{#1}{#2}%
}
\renewcommand{\@@crtextr@ct@cref}[2]{%
  \expandafter\ifx\csname r@#2@\@@crt@@cref@@suffix\endcsname\relax
    \crt@refundefined
  \else
    \expandafter\expandafter\expandafter\crt@cref@extractfields
      \csname r@#2@\@@crt@@cref@@suffix\endcsname\@nil{#1}%
  \fi
}
\makeatother
\usepackage[
backend=biber,
natbib=true,
hyperref,
backref,
doi=false,
url=false,
eprint=false,
]{biblatex}

\renewbibmacro{in:}{}

\newbibmacro{string+doiurlisbn}[1]{%
  \ifthenelse{\boolean{isOnBorisLaptop}}{%
    \iffieldundef{file}{%
        \iffieldundef{doi}{%
            \iffieldundef{url}{%
                \iffieldundef{isbn}{%
                    \iffieldundef{issn}{%
                        #1%
                      }{%
                        \href{http://books.google.com/books?vid=ISSN\thefield{issn}}{#1}%
                      }%
                    }{%
                      \href{http://books.google.com/books?vid=ISBN\thefield{isbn}}{#1}%
                    }%
                  }{%
                    \href{\thefield{url}}{#1}%
                  }%
                }{%
                  \href{https://doi.org/\thefield{doi}}{#1}%
                }%
              }{%
                \href{file:../../../../bibliography/\thefield{entrykey}.pdf}{#1}%
              }%
            }{%
              \iffieldundef{doi}{%
                  \iffieldundef{url}{%
                      \iffieldundef{isbn}{%
                          \iffieldundef{issn}{%
                              #1%
                            }{%
                              \href{http://books.google.com/books?vid=ISSN\thefield{issn}}{#1}%
                            }%
                          }{%
                            \href{http://books.google.com/books?vid=ISBN\thefield{isbn}}{#1}%
                          }%
                        }{%
                          \href{\thefield{url}}{#1}%
                        }%
                      }{%
                        \href{https://doi.org/\thefield{doi}}{#1}%
                      }%
                    }%
                  }

                  \DeclareFieldFormat{title}{\usebibmacro{string+doiurlisbn}{\mkbibemph{#1}}}
                  \DeclareFieldFormat[article,incollection,inproceedings]{title}%
                  {\usebibmacro{string+doiurlisbn}{\mkbibquote{#1}}}
\usepackage{forest}
\forestset{
  */.style={
    delay+={append={[]},}
  },
  rooted tree/.style={
    for tree={
      grow'=90,
      s sep=2.5pt,
      delay={
        if content={*}{
          content=,
          append={[]}
        }{}
      }
    },
    before typesetting nodes={
      for tree={
        circle,
        fill,
        minimum width=3pt,
        inner sep=0pt,
      },
    },
    before computing xy={
      for tree={
        l=5pt,
      }
    }
  },
  big rooted tree/.style={
    for tree={
      grow'=90,
      inner sep=1pt,
      minimum width=5pt,
      font=\tiny,
      s sep=2.5pt,
      if level=0{
        baseline
      }{},
      delay={
        if content={*}{
          content=,
          append={[]}
        }{}
      }
    },
    edge={
        semithick,
        -Latex
    },
    before computing xy={
      for tree={
        l=15pt,
      }
    }
  },
  my node/.style={
    circle, fill=#1
  },
  gray node/.style={
    my node=gray, inner sep=#1mm,
  },
  red node/.style={
    my node=red, inner sep=#1mm,
  },
}
\DeclareDocumentCommand\rootedtree{o}{\Forest{rooted tree [#1]}}
\DeclareDocumentCommand\bigrootedtree{o}{\Forest{big rooted tree [#1]}}

\author{
  Matias D. Cattaneo\thanks{Authors are in alphabetic order.} \\
  Princeton University\\
  \texttt{cattaneo@princeton.edu} \\
  \and
  Boris Shigida\footnotemark[1] \\
  Princeton University\\
  \texttt{bs1624@princeton.edu} \\
}

\title{Modified Loss of Momentum Gradient Descent:\\Fine-Grained Analysis}

\begin{document}

\maketitle
\begin{center}
    October 6, 2026
\end{center}

\bigskip

\begin{abstract}
  We analyze gradient descent with \citet{polyak1964some} heavy-ball momentum (HB) whose fixed momentum hyperparameter $\beta \in (0, 1)$ provides exponential decay of memory. Building on \citet{kovachki2021continuous}, we prove that on an exponentially attractive invariant manifold the algorithm is exactly plain gradient descent with a modified loss, provided that the step size $h$ is small enough. Although the modified loss does not admit a closed-form expression, we describe it up to $O(h^{\ord})$-errors for arbitrary finite order $\ord$, and prove global (finite ``time'' horizon) trajectory approximation bounds $O(h^{\ord})$. We then conduct a fine-grained analysis of the combinatorics underlying the memoryless approximations of HB, in particular, finding a rich family of polynomials in $\beta$ hidden inside which include and lie coefficient-wise in between Eulerian and Narayana polynomials. We prove that these polynomials are $h$-polynomials of certain graph-associahedra. As corollaries of the main results, we derive continuous modified equations of arbitrary finite approximation order (with rigorous bounds) and the principal flow that approximates the HB dynamics, generalizing \citet{rosca2023on}. Approximation theorems cover both full-batch and mini-batch HB. The results shed new light on the main features of HB and outline a roadmap for similar analysis of other optimization algorithms.
\end{abstract}

\clearpage
\tableofcontents

\clearpage
\section{Introduction}

Gradient descent with \citet{polyak1964some} heavy-ball momentum (HB) is a well-known optimization algorithm used in practice. Given a loss function $L\colon \mathbb{R}^d \mapsto \mathbb{R}$, and initial conditions $\btheta^{(0)} \in \mathbb{R}^d$ and $\bv^{(0)} \in \mathbb{R}^d$, its full-batch iteration is
\begin{equation}\label{eq:hb}
  \begin{aligned}
    &\begin{pmatrix}
      \btheta^{(n + 1)}\\
      \bv^{(n + 1)}
    \end{pmatrix} = \begin{pmatrix}
      \btheta^{(n)} - h \nabla L(\btheta^{(n)}) + h \beta \bv^{(n)}\\
      \beta \bv^{(n)} - \nabla L(\btheta^{(n)})
    \end{pmatrix},\quad n \in \mathbb{Z}_{\geq 0},
  \end{aligned}
\end{equation}
or, more compactly,
\begin{equation}\label{eq:hb-iter-compact}
\btheta^{(n + 1)} = \btheta^{(n)} - h \sum_{k = 0}^n \beta^k \nabla L(\btheta^{(n - k)}) + h \beta^{n + 1} \bv^{(0)},
\end{equation}
where $\beta \in (0, 1)$ is a momentum hyperparameter, which we treat as a fixed constant throughout this article. The next iterate $\btheta^{(n + 1)}$ depends on the whole history $\crl{\btheta^{(s)}}_{s = 0}^n$, rather than just the current iterate $\btheta^{(n)}$, which we interpret as having ``memory''. This algorithm or its variants, often referred to as gradient descent (GD) with momentum, are widely used in modern machine learning \citep{goodfellow2016deep,prince2023understanding,krizhevsky2012imagenetclassificationdeep,sutskever2013importanceinitializationmomentum,he2016deep,xie2017aggregated,he2019bag}. A mini-batch version of this algorithm is often used when training large-scale deep learning models.

This article studies the theoretical properties of HB, considering both full-batch and mini-batch implementations, and establishes four main results. First, we prove that for a fixed full-batch loss, on an exponentially attractive invariant manifold the algorithm is exactly plain gradient descent with a modified loss, provided that the step size $h$ is small enough. Second, we approximate the gradient of the full-batch modified loss with uniform error $O(h^{\ord})$ and prove global (finite ``time'' horizon) approximation bounds $O(h^{\ord})$ for any finite order $\ord \geq 2$. Third, we analyze the combinatorics hidden in the modified loss, and find a rich family of tree-parametrized polynomials with a geometric characterization, including Narayana and Eulerian polynomials. Finally, we derive continuous modified equations of arbitrary approximation order (with rigorous bounds) and the principal flow that approximates the optimization dynamics of the algorithm. The results not only shed new light on the main properties of HB, but also outline a roadmap for similar extensive analysis of other popular optimization algorithms.

\subsection{HB is GD on an Invariant Manifold}\label{sec:intro-hb-is-gd-on-manifold}

\Citet{kovachki2021continuous} proved that there exists a function $\bg_h(\btheta)$ such that $\crl{(\btheta, \bv): \bv = - (1 - \beta)^{-1} \nabla L(\btheta) + h \boldsymbol{g}_h(\btheta)}$
is an exponentially attractive invariant manifold. This perspective draws on the theory of attractive invariant manifolds \citep{fenichel1971persistence,hirsch1970invariant,wiggins2013normally}. By definition, on this manifold
\begin{equation*}
  \btheta^{(n + 1)} = \btheta^{(n)} - h \nabla L(\btheta^{(n)}) + h \beta \prn[\bigg]{- \frac{\nabla L(\btheta^{(n)})}{1 - \beta} + h \boldsymbol{g}_h(\btheta^{(n)})},
\end{equation*}
and therefore we obtain an algorithm representation without memory
\begin{equation}\label{eq:qcUUpn}
  \btheta^{(n + 1)} = \btheta^{(n)} - h \frac{1}{1 - \beta} \nabla L(\btheta^{(n)}) + h^2 \beta \boldsymbol{g}_h(\btheta^{(n)}).
\end{equation}

Using this perspective, we further show that $\boldsymbol{g}_h(\cdot)$ has a particular structure that makes \eqref{eq:qcUUpn} plain gradient descent.

\begin{contribution}\label{contrib:manifold}
    \cref{th:manifold} establishes that $\boldsymbol{g}_h$ is a gradient, that is, $\boldsymbol{g}_h = \nabla G_h$ for some $G_h$. This implies that on an exponentially attractive invariant manifold, HB is plain gradient descent:
    \begin{equation*}
        \btheta^{(n + 1)} = \btheta^{(n)} - \frac{h}{1 - \beta} \nabla \crl[\big]{\underbrace{L - h \beta (1 - \beta) G_h}_{\text{modified loss}}}(\btheta^{(n)}). \qedhere
    \end{equation*}
\end{contribution}

We then conduct a fine-grained analysis of how the loss is modified by memory.
From the definition of $\boldsymbol{g}_h$, it is possible to write a fixed point equation that gives a formal power series expansion for $\boldsymbol{g}_h$ in $h$ (as we do in \cref{sec:matching}), and therefore automatically for $G_h$ as well.
We go further because there are two wishes that we would like to fulfill.
First, we would like to have precise approximation guarantees rather than just a formal series expansion.
Second, it is useful to cover mini-batch HB, where the loss function can be different at each iteration, rather than just full-batch.
To do that, we take a different perspective.

\subsection{Finite-Order Memoryless Approximation}\label{sec:intro-memoryless-approximation}

Consider the mini-batch version of \eqref{eq:hb-iter-compact} with the typical setting $\bv^{(0)} = \boldsymbol{0}$:
\begin{equation}\label{eq:hb-iter-minibatch}
\btheta^{(n + 1)} = \btheta^{(n)} - h \sum_{k = 0}^n \beta^k \nabla L^{(n - k)}(\btheta^{(n - k)}),
\end{equation}
where $\crl{L^{(s)}}_{s \in \mathbb{Z}_{\geq 0}}$ are mini-batch losses.
In practice, $L^{(s)}$ is the loss function obtained by taking the $s$th mini-batch of samples.
(For now, we are agnostic to how exactly the samples are batched.)
Mini-batch training usually achieves higher test accuracies \citep{keskar2017on,masters2018revisiting,smith2020generalizationbenefitnoise},
and is therefore widely used in practice \citep{bottou2018optimization,prince2023understanding}.
\Cref{prop:exact-minibatch-gradient} proves that mini-batch HB is exactly plain gradient descent
with (implicitly defined) step-dependent losses, as long as the step size $h$ is small enough (providing a sharp bound on $h$).
However, this does not yield concrete explicit correction terms arising from memory.
\citet{cattaneo2025howmemoryoptimization} introduced a technique for identifying such corrections, converting a numerical optimization method with decaying memory into a memoryless one up to $O(h^2)$ error, that is, a second-order approximation. We generalize their technique and prove an approximation bound to any desired order.

\begin{contribution}\label{contrib:approx}
    For any approximation order $\ordDef \in \mathbb{Z}_{\geq 2}$, \cref{th:approx} establishes a memoryless approximation
    \begin{equation}\label{eq:memoryless-approx-contrib}
        \tilde{\btheta}^{(n + 1)} = \tilde{\btheta}^{(n)} + \sum_{j = 1}^{\ord} h^j \nomemcoef{j}{n}(\tilde{\btheta}^{(n)})
    \end{equation}
    with the global guarantee
    \begin{equation}\label{eq:global-guarantee-contrib}
        \sup_{n \in \range{0}{\lfloor T / h \rfloor}} \norm{\btheta^{(n)} - \tilde{\btheta}^{(n)}} = O(h^{\ord}),
    \end{equation}
    where $T$ is any fixed ``time'' horizon.
\end{contribution}

This contribution provides a higher-order, more detailed understanding of HB and its memoryless representation, which can be used to make qualitative conclusions on the dynamics of the algorithm, in particular in the context of implicit regularization. The first two coefficients of the memoryless iteration \eqref{eq:memoryless-approx-contrib} are
\begin{equation*}
  \nomemcoef{1}{n}(\btheta) = - \sum_{b = 0}^n \beta^b \nabla L^{(n - b)}(\btheta),
  \quad
  \nomemcoef{2}{n}(\btheta) = - \beta \sum_{b = 0}^{n - 1} \beta^{b} \sum_{l' = 1}^{b + 1} \sum_{b' = 0}^{n - l'} \beta^{b'} \nabla^2 L^{(n - 1 - b)}(\btheta) \nabla L^{(n - l' - b')}(\btheta).
\end{equation*}
In the mini-batch case, such an expression can be informative after averaging over permutations of samples \citep{smith2021on,beneventano2024trajectoriessgdwithout,cattaneo2025howmemoryoptimization}. We will illustrate it as follows. Assume that there are $n + 1$ batches in an epoch, with each batch consisting of $\batchsizeDef$ samples, and the $k$th mini-batch loss is given by
\begin{equation*}
    L^{(k)}(\btheta) = \frac{1}{\batchsize} \sum_{r = k \batchsize + 1}^{k \batchsize + \batchsize} \ell_{\pi(r)}(\btheta),\quad k \in \range{0}{n},
\end{equation*}
where $\crl{\ell_s}_{s = 1}^{(n + 1) \batchsize}$ are per-sample losses and $\pi$ is a random permutation of $\range{1}{(n + 1) \batchsize}$, distributed uniformly over all $((n + 1) \batchsize)!$ such permutations. This corresponds to sampling without replacement as common in practice. Denote $L(\btheta) = [(n + 1) \batchsize]^{-1} \sum_{s = 1}^{(n + 1) \batchsize} \ell_s(\btheta)$ and let
\begin{equation}\label{eq:empirical-covariance}
    \Sigma_{i j} := \frac{1}{(n + 1) \batchsize} \sum_{p = 1}^{(n + 1) \batchsize} \partial_i (\ell_p - L) \partial_j (\ell_p - L)
\end{equation}
be the elements of the empirical covariance matrix $\mSigma(\btheta)$ of per-sample gradients.
Then
\begin{align*}
  &\Eletpi \prn[\big]{\nomemcoef{1}{n}(\btheta) + h \nomemcoef{2}{n}(\btheta)}\\
  &\quad = - \frac{1}{1 - \beta} \nabla \prn[\bigg]{
    \prn[\big]{1 + o_n(1)} L(\btheta) + \underbrace{h \frac{\beta + o_n(1)}{2 (1 - \beta)^2} \norm{\nabla L(\btheta)}^2}_{\text{regularization by memory}} + \underbrace{h \frac{\beta + o_n(1)}{2 (1 - \beta) (1 + \beta)} \frac{\trace \mSigma(\btheta)}{\batchsize}}_{\text{regularization by stochasticity}}
    },
\end{align*}
where $\Eletpi$ denotes the expectation over $\pi$ with $\vtheta$ held fixed, $o_n(1)$ are terms that go to zero as $n \to \infty$ (for fixed $\beta$).

In the full-batch case $L^{(s)} \equiv L$ and after neglecting $o_n(1)$ terms (that is, taking $n \to \infty$), the first three coefficients are
\begin{equation}\label{eq:kTgATX}
\begin{aligned}
  \nomemcoef{1}{\infty}(\btheta)
  = &- \frac{1}{1 - \beta} \nabla L(\btheta),
  \\
  \nomemcoef{2}{\infty}(\btheta)
    = &- \frac{\beta}{(1 - \beta)^3} \nabla^2 L(\btheta) \nabla L(\btheta),
  \\
  \nomemcoef{3}{\infty}(\btheta)
  =
  &- \frac{\beta (1 + \beta)}{(1 - \beta)^5} \nabla^2 L (\btheta) \nabla^2 L
    (\btheta) \nabla L (\btheta)
  \\
  &- \frac{\beta (1 + \beta)}{2 (1 - \beta)^5} \nabla^3 L (\btheta)
    \brk[\big]{\nabla L(\btheta), \nabla L(\btheta)},
\end{aligned}
\end{equation}
so the limiting memoryless approximation of order $\ord = 3$ is
\begin{equation}\label{eq:SZADJf}
  \tilde{\btheta}^{(n + 1)} = \tilde{\btheta}^{(n)}
  - \frac{h}{1 - \beta} \nabla \underbrace{\prn[\bigg]{
    L + \frac{h \beta}{2 (1 - \beta)^2} \norm{\nabla L}^2
    + \frac{h^2 \beta (1 + \beta)}{4 (1 - \beta)^4} \prn[\big]{\nabla \norm{\nabla L}^2 \cdot \nabla L}}}_{\text{modified loss}}(\tilde{\btheta}^{(n)}),
\end{equation}
where the dot denotes the scalar product. We see two ``implicit regularization'' terms added to the loss by memory: the rescaled squared gradient norm and the rescaled directional derivative of $\norm{\nabla L}^2$ along $\nabla L$.

The loss differentials from \eqref{eq:kTgATX} are parametrized
by rooted trees, e.\,g.,
\begin{equation*}
  \nabla^2 L(\btheta)\nabla^2 L(\btheta)\nabla L(\btheta)
  = \nabla^{\rootedtree[[[]]]} L(\btheta),
  \quad
  \nabla^3 L(\btheta)
  \brk[\big]{\nabla L(\btheta),\nabla L(\btheta)}
  = \nabla^{\rootedtree[[][]]} L(\btheta).
\end{equation*}
In general, we can write the memoryless iteration coefficients as the sum over such rooted trees:
\begin{equation}\label{eq:TdPJZo}
\nomemcoef{j}{\infty}(\vtheta)
 = -\frac{\beta}{(1 - \beta)^{2 j - 1}}
    \sum_{\tau: |\tau| = j}
    \frac{e_{\tau, 1}(\beta)}{\sigma(\tau)}
    \nabla^\tau L(\vtheta),
    \quad j \geq 2,
\end{equation}
where $|\tau|$ denotes the number of vertices in the tree, $\sigma(\tau)$ is the symmetry coefficient defined later, and $\crl{e_{\tau, 1}(\beta)}$ are polynomials of $\beta$ (the subscript $1$ will be clear later). We completely characterize this family of polynomials.

\begin{contribution}\label{contrib:combinatorics}
\Cref{th:nomem-coef} establishes a tractable form of the memoryless iteration coefficients $\nomemcoef{j}{n}(\btheta)$ involving a sum over unlabeled rooted trees with $j$ vertices, in particular leading to \eqref{eq:TdPJZo}. \Cref{thm:e-h-poly} writes the polynomials $e_{\tau, 1}(\beta)$ as $h$-polynomials of certain graph-associahedra. \Cref{cor:e-properties} infers that they are palindromic, monic, and unimodal, include Eulerian and Narayana polynomials and lie coefficient-wise between the two.
\end{contribution}

The characterization of the polynomials (in particular their independence of the choice of the root of $\tau$) allows us to construct explicitly the modified loss of the limiting full-batch memoryless approximation to HB at any finite order (\cref{prop:finite-order-modified-loss}).
This fine-grained analysis leads to interesting findings.
First, using the natural heuristic that $h$ multiplied by a high-order derivative (higher than two) of the loss is small, but $h$ times the Hessian ($h \nabla^2 L$) does not have to be small mid-training \citep{cohen2021gradient,andreyev2025edgestochasticstability}, we can ask what the ``principal'' part of \eqref{eq:memoryless-approx-contrib} looks like, that is, the part remaining after neglecting the derivatives of the loss of order higher than two (always multiplied by some positive power of $h$). This leads to what we can call the \textit{principal iteration}
\begin{equation*}
    \tilde{\btheta}^{(n + 1)} = \tilde{\btheta}^{(n)} - h \sigma_{\beta}\prn[\big]{h \nabla^2 L\prn[\big]{\tilde{\btheta}^{(n)}}} \nabla L\prn[\big]{\tilde{\btheta}^{(n)}} + \np,
\end{equation*}
after taking $\ord$ formally to infinity and in the limit $n \to \infty$, where $\sigma_{\beta}(\cdot)$ is a power series expansion (in powers of $z$) of
\begin{equation*}
    \sigma_{\beta}(z) = \frac{2}{1 - \beta + z + \sqrt{(1 - \beta - z)^2 - 4 \beta z}}
\end{equation*}
(\cref{cor:principal-coef}),
and $\np$ means ``non-principal terms'' (with $\nabla^3 L$, $\nabla^4 L$ or higher).
The term ``principal iteration'' comes from the analogous term \textit{principal flow} coined in \citet{rosca2023on} for continuous modified equations of plain GD. Second, we combine our framework with continuous approximations (discussed next) to derive the principal flow for HB (\cref{cor:principal-flow}). See \citet{rosca2023on} for a discussion of the importance of such (complex) flows: they capture oscillatory behavior and divergence, in contrast to standard continuous approximations. We provide an illustration in \cref{sec:principle-flow}.

In the case of a quadratic loss, the principal iteration and principal flow are both exact (if HB is initialized on the manifold introduced above; see \cref{sec:principle-flow}). To illustrate the implicit regularization effect further, consider the example of least-squares regression. Let $\bX \in \mathbb{R}^{N \times d}$, $\by \in \mathbb{R}^N$ and
\begin{equation*}
L(\btheta) = \frac{1}{2 N} \norm{\bX \btheta - \by}^2.
\end{equation*}
Assume that $\frac{h}{N} \norm{\bX\trans \bX} < (1 - \sqrt{\beta})^2$ to ensure the convergence of the power series defining $\sigma_{\beta}(\frac{h}{N} \bX\trans \bX)$.
Additionally, let $\btheta^{\star}$ satisfy the normal equations $\bX\trans \bX \btheta^{\star} = \bX\trans \by$; in particular, $L(\btheta) = \frac{1}{2 N} (\btheta - \btheta^{\star})\trans \bX\trans \bX (\btheta - \btheta^{\star}) + C$ (where $C$ does not depend on $\btheta$). Then, the memoryless iteration is the gradient descent
\begin{equation*}
\tilde{\btheta}^{(n + 1)} = \tilde{\btheta}^{(n)} - h \nabla \tilde{L}\prn[\big]{\tilde{\btheta}^{(n)}}\quad\text{with}\quad\tilde{L}(\btheta) = \frac{1}{2 N} (\btheta - \btheta^{\star})\trans \sigma_{\beta}\prn[\bigg]{\frac{h}{N} \bX\trans \bX} \bX\trans \bX (\btheta - \btheta^{\star}).
\end{equation*}
Introducing $\sigma_{\beta}\prn[\big]{\frac{h}{N} \bX\trans \bX}$ rescales the learning rate (by $(1 - \beta)^{-1}$) and applies a spectral filter to the Hessian.

\subsection{Modified Equations}

The discrete iteration without memory \eqref{eq:memoryless-approx-contrib} has the following advantage over HB: it only has one evolving variable $\btheta^{(n)} \in \mathbb{R}^d$, and it
can be approximated by a continuous ODE solution without sacrificing anything in terms of the approximation guarantee. (Enlarging the phase space by treating HB as an evolution of $(\btheta^{(n)}, \bv^{(n)}) \in \mathbb{R}^{2 d}$ or, equivalently, considering higher-order ODEs does not by itself guarantee a higher approximation order \citep{kovachki2021continuous}.) Then, finding continuous trajectories closely tracking discrete numerical iterations allows us to analyze qualitatively the finite-$h$ behavior of the iteration. This approach is called the method of modified equations, backward error analysis (BEA), or high-resolution ODE approximation. It is widely used in numerical analysis and deep learning theory; see, for example, works
\citep{wilkinson1960error,warming1974modified,griffiths1986scope,feng1991formal,beyn1991numerical,eirola1993aspects,Hairer1994Backward,calvo1994modified,reich1999backward,chartier2007numerical,shi2022understanding,farazmand2020multiscale,kovachki2021continuous} and many others; textbook references include \citet{stuart1998dynamical,hairer2006}.
Recently, this method was used in deep learning theory for finding implicit biases of gradient-based algorithms
\citep{barrett2021implicit,smith2021on,miyagawa2022toward,ghosh2023implicit,pmlr-v235-cattaneo24a,rosca2023on}. For us, it is just a corollary of the approximation \eqref{eq:memoryless-approx-contrib}.

\begin{contribution}\label{contrib:modified-eq}
    For any approximation order $\ord \in \mathbb{Z}_{\geq 2}$, \cref{cor:modified-eq} finds a set of functions $\crl{\beacoef{j}{n}(\btheta): j \in \range{1}{\ord}, n \in \mathbb{Z}_{\geq 0}}$ such that the (unique) continuous solution to the piecewise ODE
    \begin{equation*}
    \dot{\btheta}(t) = \sum_{j = 1}^{\ord} h^{j - 1} \beacoef{j}{n}(\btheta(t)), \quad t \in [n h, (n + 1)h],
    \end{equation*}
    with an appropriate initial condition, has a global approximation guarantee
    \begin{equation}\label{eq:ode-approx-contrib}
    \sup_{n \in \range{0}{\lfloor T / h \rfloor}} \norm{\btheta^{(n)} - \btheta(n h)} = O(h^{\ord}),
    \end{equation}
    where $T$ is any fixed time horizon.
\end{contribution}

The ODE is defined piecewise because of the dependence of $\nomemcoef{j}{n}(\cdot)$ on $n$, and the solution will be continuous but not necessarily smooth. Nonetheless, in the full-batch case, the ODEs corresponding to neighboring pieces are the same up to an exponentially decaying error: we do not lose any information compared to having an ODE only on the attractive manifold but globally defined, as done for $\ord = 2$ in \citet{kovachki2021continuous}. The idea of using piecewise ODEs in this setting is due to \citet{ghosh2023implicit}, who derived the modified equations for $\ord = 2$. In an independent and concurrent work, \citet{lyu2025heavyball} derived piecewise modified equations of arbitrary order for full-batch HB. Although their construction and proof use a different method, they yield the same values for $\crl{\beacoef{j}{n}}$ as the full-batch case of \cref{cor:modified-eq} (differently characterized).

In contrast to the discrete case, the right-hand side of the modified equation does \textit{not} (in general) become a gradient for large $n$, even for full-batch HB. Therefore, HB does \textit{not} approximately follow a smooth gradient flow trajectory with a modified loss (with approximation order higher than $h^2$),
as we make explicit in \cref{rem:not-gf}.

\subsection{Organization}

We continue as follows. In an attempt to make the paper self-contained, we introduce the relevant mathematical concepts and notation (e.\,g., labeled and unlabeled rooted trees, their symmetry coefficients, marking vertices to split the tree into a forest) in \cref{sec:Technical Background and Notation}.
\cref{sec:full-batch-hb-is-gd} states the theorem and briefly outlines the argument corresponding to \cref{contrib:manifold}.
\Cref{sec:step-dependent-losses} shows a related representation of mini-batch HB as plain GD with step-dependent losses.
\Cref{sec:mini-batch-hb-approx-theorem} is devoted to the main approximation theorem corresponding to \cref{contrib:approx}, removing memory from HB and controlling the error while doing so. In particular, we introduce in \cref{sec:approx-main-bound} a somewhat complicated-looking recursive definition for the memoryless iteration coefficients $\nomemcoef{j}{n}(\btheta)$.
In \cref{sec:simpler-form}, we prove that a much simpler characterization holds which involves the sum over rooted trees. We use the notation introduced there to derive the principal iteration.
In \cref{sec:combinatorics}, we analyze the combinatorial structure of $\nomemcoef{j}{n}(\btheta)$, in particular proving that the polynomials $e_{\tau, 1}(\beta)$ are $h$-polynomials of graph-associahedra.
\Cref{sec:matching} is devoted to matching the memoryless iteration perspective and the consequences of the fixed-point argument yielding $\vg_h$, in particular proving the equality of formal power series and proving $O(h^{\ord})$ bounds for the
discrete trajectories arising from the two perspectives.
\Cref{sec:continuous} derives the modified equations corresponding to \cref{contrib:modified-eq} and the principal flow.
Proof strategies are explained in the main text, and omitted technical details are provided in the Appendix.

\subsection{Technical Background and Notation}\label{sec:Technical Background and Notation}

We start with well-known definitions of labeled and unlabeled rooted trees. See, for example, \citet{faris2021rootedtreegraphs,Kerber1999} for details. For a fixed non-empty finite set $V$ and a point $r \in V$, a \textit{labeled rooted tree} $\tau$ with root $r$ is a pair $(G, r)$, where $G = (V, E)$ is a connected acyclic graph with $V$ as the vertex set (where $E$ is the edge set; $\abs{E} = \abs{V} - 1$). Taking a different point of view, $\tau$ is a function $\tau\colon V \setminus \crl{r} \to V$ with no non-empty invariant subset (mapping a vertex to its parent). Let $\labrooted[V]$ denote all labeled rooted trees with vertex set $V$. We will write $\abs{\tau} := \abs{V}$ for the number of vertices.
It will also be convenient to use the notation $\labrooted := \bigcup_{m = 1}^{\infty} \labrooted\range{1}{m}$ and $\labrooted_{\varnothing} := \crl{\varnothing} \cup \labrooted$.

Two labeled rooted trees $\tau\colon V \setminus \crl{r} \to V$ and $\tau'\colon V' \setminus \crl{r'} \to V'$ are called isomorphic if there is a bijection $\mathcal{b}\colon V \to V'$ such that $\mathcal{b}(r) = r'$ and $\tau' = \mathcal{b} \circ \tau \circ \left.\mathcal{b}^{-1}\right|_{V' \setminus \crl{r'}}$. An unlabeled rooted tree with $m$ vertices is an object that corresponds to a class of isomorphic labeled rooted trees with $m$ vertices. Specifically, we can fix a canonical set of $m$ elements $\range{1}{m}$ and see an \textit{unlabeled rooted tree} as an orbit of the permutation group $S\range{1}{m}$ acting on $\labrooted\range{1}{m}$: the action of $\pi \in S\range{1}{m}$ on $\tau \in \labrooted\range{1}{m}$ is $\pi \tau = \pi \circ \tau \circ \left.\pi^{-1}\right|_{\range{1}{m} \setminus \crl{\pi(r)}}$.
The set of unlabeled rooted trees with $m$ vertices will be denoted $\unlabrooted[m]$, and the set of unlabeled rooted trees with any number of vertices by $\unlabrooted := \bigcup_{m = 1}^{\infty} \unlabrooted[m]$; in addition, put $\unlabrooted_{\varnothing} := \crl{\varnothing} \cup \unlabrooted$. For example, the following orbit (consisting of three labeled rooted trees) is the unlabeled rooted tree $\rootedtree[[][]]$:
\begin{equation}\label{eq:kljqh}
\bigrootedtree[1[2][3]]\quad \bigrootedtree[2[1][3]]\quad \bigrootedtree[3[1][2]]
\end{equation}
For any such orbit, the order of the stabilizer group of each element is the same. (The stabilizer group is the subgroup of permutations leaving a labeled rooted tree intact. For example, the first tree in \eqref{eq:kljqh} has stabilizer group $\crl{\text{id}, 2 \leftrightarrow 3}$.) Hence, for an unlabeled rooted tree $\tau$ we can define the \textit{symmetry coefficient} $\sigma(\tau)$ as the order of the stabilizer group of each element. By the well-known theorem, the length of the orbit is the order of the group $m!$ divided by the order of each stabilizer group $\sigma(\tau)$; so, there are $m! / \sigma(\tau)$ labeled rooted trees (on a fixed vertex set of size $m$) corresponding to an unlabeled rooted tree $\tau$.

Whenever there is no chance of confusion, we will use the same terms and symbols when working with labeled and unlabeled rooted trees, for example, we will write $\sigma(\tau)$ regardless of whether $\tau \in \labrooted$ or $\tau \in \unlabrooted$ (because each labeled rooted tree has a corresponding unlabeled rooted tree); we will say an unlabeled rooted tree $\tau \in \unlabrooted[m]$ ``has $\abs{\tau} = m$ vertices'' even though there is no such thing as a vertex of an unlabeled rooted tree, etc.

Since the ordering of subtrees does not matter, there is a one-to-one correspondence between $\tau \in \unlabrooted[m]$ and a multiset $[\tau_1, \ldots, \tau_{\ell}]$ of unlabeled rooted trees whose sum of vertices is $m - 1$ (the subtrees rooted at the children of $\tau$'s root). We will sometimes write simply $\tau = [\tau_1, \ldots, \tau_{\ell}]$ to reflect this fact. Fixing some canonical ordering in each set $\unlabrooted[s]$, we will denote by $\crl[\big]{\mu_1^s(\tau), \ldots, \mu_{\abs{\unlabrooted[s]}}^s(\tau)}_{s = 1}^{m - 1}$ the \textit{multiplicities} of such subtrees: this means that the first unlabeled subtree with $s$ vertices appears $\mu_1^s(\tau)$ times, the second one appears $\mu_2^s$ times, and so on. In particular, $\sum_{s = 1}^{m - 1} \prn[\big]{\mu_1^s(\tau) + \ldots + \mu_{\abs{\unlabrooted[s]}}^s(\tau)} = \ell$. It is a standard fact that
\begin{equation}\label{eq:recursive-formula-for-sym-coef}
\sigma(\tau) = \sigma(\tau_1) \ldots \sigma(\tau_{\ell}) \prod_{s = 1}^{m - 1} \prod_{t = 1}^{\abs{\unlabrooted[s]}} \mu_t^s(\tau)!.
\end{equation}

In addition, it will be convenient to call $\mathcal{c}_m \in \unlabrooted[m]$ the \textit{chain} with $m$ vertices if either $m = 1$ or it is the (unique) unlabeled rooted tree corresponding to any element of $\labrooted\range{1}{m}$ whose root has degree 1 and all other vertex degrees do not exceed 2.

Similarly to the rooted variants, we denote by $\lab[V]$ the set of \textit{labeled unrooted trees}
with (non-empty) vertex set $V$, that is, connected acyclic graphs $(V, E)$. Two labeled unrooted trees $G = (V, E)$ and $G' = (V', E')$ are
\textit{isomorphic} if there is a bijection $\mathcal{b}\colon V \to V'$
such that $\crl{u,v} \in E$ if and only if
$\crl{\mathcal{b}(u),\mathcal{b}(v)} \in E'$ for every pair of distinct
vertices $u, v \in V$. An \textit{unlabeled unrooted tree} with $m$ vertices is an isomorphism
class of labeled unrooted trees with $m$ vertices. Denote their set by $\unlab[m]$.
Equivalently, its elements are the orbits of $S\range{1}{m}$ acting on $\lab\range{1}{m}$ by relabeling vertices, thus $\unlab[m] = \lab\range{1}{m} / S\range{1}{m}$. We also write $\lab := \bigcup_{m = 1}^{\infty} \lab\range{1}{m}$ and $\unlab := \bigcup_{m = 1}^{\infty} \unlab[m]$.

\paragraph{Marking Vertices (via Admissible Cuts).}

For a labeled rooted tree $\tau \in \labrooted[V]$, define a \textit{marking} $\mrk$ of that tree as a subset of non-root vertices of $\tau$ (interpreted as \textit{marked}) such that if $v \in V$ is marked, then no other vertices in the subtree of $v$ are marked. In other words, the marked vertices are the upper (farther from the root) vertices of an admissible cut (e.\,g. \citet{Bai2024PostGroups}),
or the roots of the forest obtained after removing a subtree with the same root
\citep{faris2021rootedtreegraphs,chartier2010algebraic}. Hence, marking $\abs{\mrk}$ vertices is the same as selecting $\abs{\mrk}$ disjoint subtrees $\tau_1, \ldots, \tau_{\abs{\mrk}}$ (rooted at those vertices), and it splits the tree $\tau$ into $\abs{\mrk} + 1$ labeled rooted trees $\tau_0^{\mrk} = \tau \setminus (\tau_1^{\mrk} \cup \ldots \cup \tau^{\mrk}_{\abs{\mrk}})$, $\tau_1^{\mrk}, \ldots, \tau^{\mrk}_{\abs{\mrk}}$. So, we will sometimes write $\mrk = (\tau_0^{\mrk}, \crl{\tau_1^{\mrk}, \ldots, \tau^{\mrk}_{\abs{\mrk}}})$ to reflect this fact. Denote $\mrkset_{\tau, i}$ the set of all markings of $\tau$ with $\abs{\mrk} = i$ marked vertices, and $\mrkset_{\tau} := \bigcup_{i = 0}^{\abs{V} - 1} \mrkset_{\tau, i}$ the set of all markings of $\tau$.
For example, the following tree has marking $\mrk = \crl{2, 5}$ highlighted in red:
\begin{equation*}
\tau = \bigrootedtree[1[2,text=red[4]][3[5,text=red]]].
\end{equation*}
Cutting the edges 1--2 and 3--5 gives
\begin{equation*}
\tau_0^{\mrk} = \bigrootedtree[1[3]],
\qquad
\tau_1^{\mrk} = \bigrootedtree[2,text=red[4]],
\qquad
\tau_2^{\mrk} = \bigrootedtree[5,text=red].
\end{equation*}
Here $\tau_0^{\mrk}$ contains the original root, and $\tau_1^{\mrk}$ and $\tau_2^{\mrk}$ are the subtrees rooted at the marked vertices.
In contrast, $\crl{2, 4}$ is not a marking of $\tau$ since $4$ is a descendant of $2$.

\paragraph{Other Notation.}

We denote by $L(\btheta)$ the full-batch loss and by $L^{(s)}(\btheta)$ the $s$th mini-batch loss, where $\btheta \in \mathbb{R}^d$ is the evolving vector of parameters of fixed dimension $d$. We denote by $\nabla^k$ the $k$th derivative tensor, for example, $\nabla^3 L(\btheta) \brk[\big]{\nabla L(\btheta), \nabla L(\btheta)}$ is a vector whose $j$th component is
$\sum_{i, l = 1}^d \partial_{j i l} L(\btheta) \partial_i L(\btheta) \partial_l L(\btheta)$.
We write $\nabla^{\tau} L$ for the elementary differential of $L$, defined recursively by $\nabla^{\varnothing} L(\btheta) = \btheta$, $\nabla^{\rootedtree[]} L(\btheta) = \nabla L(\btheta)$ and $\nabla^{\tau} L(\btheta) = \nabla^{\ell + 1} L(\btheta) \brk[\big]{\nabla^{\tau_1} L(\btheta), \ldots, \nabla^{\tau_{\ell}} L(\btheta)}$ for $\tau = [\tau_1, \ldots, \tau_{\ell}]$.
The notation for the norm $\norm{\cdot}$ without indices will always mean the Euclidean (operator) norm. The set of integers no smaller than some $k$ will be denoted by $\mathbb{Z}_{\geq k} := [k, + \infty) \cap \mathbb{Z}$, and the set of integers between $a$ and $b$ inclusive by $\range{a}{b} := [a, b] \cap \mathbb{Z}$. A sum over an empty set is by definition $0$, and a product is $1$. \Cref{sec:mini-batch-hb-approx-theorem} defines the set $\mathcal{K}_{i, l}$, the memoryless iteration coefficients $\nomemcoef{j}{n}(\btheta)$, the history coefficients $\histcoef{m}{n}{a}(\btheta)$; \cref{sec:modified-equation} defines the backward error analysis coefficients $\beacoef{j}{n}(\btheta)$. For any operator $\ba_{n, h}$ with a defined norm $\norm{\ba_{n, h}}$ and any nonnegative scalar function $g_{n, h}$, we write $\ba_{n, h} = O(g_{n, h})$ if $\norm{\ba_{n, h}} \leq C g_{n, h}$ on the set of admissible $n$ and $h$ which is clear from context, with some constant $C$ not depending on $n$ or $h$.

\section{Exact Representations}\label{sec:exact-representations}

We present two results in this section.
The first one says that for small enough $h$, full-batch HB is GD with one modified loss on an attractive invariant manifold.
If HB is not initialized on that manifold ($\vv^{(0)}$ is not appropriately chosen), the dynamics get closer to GD because of the exponential attraction but never exactly coincide with it. The second result says that HB in the general (mini-batch) case, even with $\vv^{(0)} = \boldsymbol{0}$ (as done in practice), can be viewed as GD if we allow the loss to change at each iteration. Moreover, the resulting family of losses does not depend on the initialization $\vtheta^{(0)}$.
These results are exact but the losses are not explicitly constructed.

\subsection{Full-Batch HB is GD with a Modified Loss on an Invariant Manifold}\label{sec:full-batch-hb-is-gd}

\Citet{kovachki2021continuous} proved that for full-batch HB \eqref{eq:hb} there exists a function $\boldsymbol{g}_h(\btheta)$ such that $\crl{(\btheta, \bv): \bv = - (1 - \beta)^{-1} \nabla L(\btheta) + h \boldsymbol{g}_h(\btheta)}$
is an invariant manifold, which means
\begin{equation*}
\bv^{(n)} = - \frac{\nabla L(\btheta^{(n)})}{1 - \beta} + h \boldsymbol{g}_h(\btheta^{(n)}) \quad \Rightarrow \quad \bv^{(n + 1)} = - \frac{\nabla L(\btheta^{(n + 1)})}{1 - \beta} + h \boldsymbol{g}_h(\btheta^{(n + 1)}).
\end{equation*}

On this invariant manifold, on the one hand,
\begin{align*}
  \bv^{(n + 1)} &= - (1 - \beta)^{-1} \nabla L(\btheta^{(n + 1)}) + h \boldsymbol{g}_h(\btheta^{(n + 1)})\\
                &= - \frac{1}{1 - \beta} \nabla L \prn[\bigg]{\btheta^{(n)} - \frac{h}{1 - \beta} \nabla L(\btheta^{(n)}) + h^2 \beta \boldsymbol{g}_h(\btheta^{(n)})}\\
  &\quad + h \boldsymbol{g}_h \prn[\bigg]{\btheta^{(n)} - \frac{h}{1 - \beta} \nabla L(\btheta^{(n)}) + h^2 \beta \boldsymbol{g}_h(\btheta^{(n)})},
\end{align*}
while, on the other hand,
\begin{align*}
  \bv^{(n + 1)} &= \beta \bv^{(n)} - \nabla L(\btheta^{(n)}) = - \frac{\beta}{1 - \beta} \nabla L(\btheta^{(n)}) + h \beta \boldsymbol{g}_h(\btheta^{(n)}) - \nabla L(\btheta^{(n)})\\
  &= - \frac{1}{1 - \beta} \nabla L(\btheta^{(n)}) + h \beta \boldsymbol{g}_h(\btheta^{(n)}).
\end{align*}
This must hold for any $\btheta^{(n)}$, so $\boldsymbol{g}_h(\cdot)$ must satisfy
\begin{equation}\label{eq:uruw}
\begin{aligned}
  &- \frac{1}{1 - \beta} \nabla L \prn[\bigg]{\btheta - \frac{h}{1 - \beta} \nabla L(\btheta) + h^2 \beta \boldsymbol{g}_h(\btheta)} + h \boldsymbol{g}_h \prn[\bigg]{\btheta - \frac{h}{1 - \beta} \nabla L(\btheta) + h^2 \beta \boldsymbol{g}_h(\btheta)}\\
  &\quad = - \frac{1}{1 - \beta} \nabla L(\btheta) + h \beta \boldsymbol{g}_h(\btheta),\quad \btheta \in \mathbb{R}^d.
\end{aligned}
\end{equation}

To prove that such a function exists, they define a mapping $T\colon \Gamma \to \Gamma$, where $\Gamma$ is an appropriately chosen closed subset of $C(\mathbb{R}^d, \mathbb{R}^d)$ with the usual sup-norm, by
\begin{equation}\label{eq:def-of-t}
T \boldsymbol{g}(\bzeta) = \frac{1}{h (1 - \beta)} \crl[\big]{\nabla L(\bzeta) - \nabla L(\btheta)} + \beta \boldsymbol{g}(\btheta),
\end{equation}
where $\btheta \leftrightarrow \bzeta = \btheta - h (1 - \beta)^{-1} \nabla L(\btheta) + h^2 \beta \boldsymbol{g}(\btheta)$ is a bijection between $\mathbb{R}^d$ and itself for any $\boldsymbol{g} \in \Gamma$ (this fact is \cref{lem:ljrhh}). A fixed point of such a mapping $T$ would satisfy~\eqref{eq:uruw}. They prove that $T$ is a contraction on $\Gamma$, so the contraction mapping principle applies.

We follow the same conceptual approach, but apply some technical tweaks (different metric, different $\Gamma$), to prove that the additional requirement $\bg_h \in C^1(\mathbb{R}^d, \mathbb{R}^d)$ with symmetric $\nabla \bg_h$ is satisfiable as well, which implies that such $\bg_h$ is a gradient. Specifically, define the set
\begin{equation*}
    \Gamma := \crl[\big]{\boldsymbol{g} \in C^1(\mathbb{R}^d, \mathbb{R}^d): \sup_{\btheta \in \mathbb{R}^d} \norm{\boldsymbol{g}} \leq \gamma, \sup_{\btheta \in \mathbb{R}^d} \norm{\nabla \boldsymbol{g}} \leq \delta, \nabla \boldsymbol{g}\text{ is symmetric and $\lambda$-Lipschitz}},
\end{equation*}
with the norm
\begin{equation}\label{eq:gamma-norm}
    \| \boldsymbol{g} \|_{\Gamma} := \sup_{\btheta \in \mathbb{R}^d} \| \boldsymbol{g}(\btheta) \| + \sup_{\btheta \in \mathbb{R}^d} \| \nabla \boldsymbol{g}(\btheta) \|,
\end{equation}
where both the vector norm and the matrix norm in the right-hand side are Euclidean. The space $\crl[\big]{\boldsymbol{g} \in C^1(\mathbb{R}^d, \mathbb{R}^d): \| \boldsymbol{g} \|_{\Gamma} < \infty}$ with norm \eqref{eq:gamma-norm} is a Banach space, and $\Gamma$ is a closed subset of it; therefore $\Gamma$ can be seen as a complete metric space with the metric induced by the norm $\| \cdot \|_{\Gamma}$. The constants $\gamma$, $\delta$, $\lambda$ are chosen in \cref{sec:manifold-details}.

The assumptions of the theorem are essentially the same as in \citet{kovachki2021continuous}, except we (naturally) need one more derivative of the loss to be Lipschitz.

\begin{theorem}\label{th:manifold}
  Assume $L(\cdot) \in C^3(\mathbb{R}^d, \mathbb{R})$ with constants $D_1, D_2, D_3, D_4$ such that
  \begin{equation*}
    \sup_{\btheta} \norm{\nabla^j L(\btheta)} \leq D_j\text{ for $j \in \range{1}{3}$}, \quad \sup_{\btheta_1 \neq \btheta_2} \frac{\norm{\nabla^3 L(\btheta_1) - \nabla^3 L(\btheta_2)}}{\norm{\btheta_1 - \btheta_2}} \leq D_4,
  \end{equation*}
  where the norms are Euclidean operator norms. Then, if $h$ is small enough, there exists a unique $\bg_h(\cdot) \in \Gamma$ satisfying \eqref{eq:uruw}, and the following exponential attractivity property holds:
  \begin{equation*}
    \norm[\bigg]{\bv^{(n)} + \frac{1}{1 - \beta} \nabla L(\btheta^{(n)}) - h \bg_h(\btheta^{(n)})} \leq \rho_h^n \norm[\bigg]{\bv^{(0)} + \frac{1}{1 - \beta} \nabla L(\btheta^{(0)}) - h \bg_h(\btheta^{(0)})},
  \end{equation*}
  where
  \begin{equation*}
    \rho_h := \beta \prn[\bigg]{1 + \frac{h D_2}{1 - \beta} + h^2 \delta} < 1.
  \end{equation*}
In particular, $\vg_h$ is a gradient, that is, there exists $G_h \in C^2(\mathbb{R}^d, \mathbb{R})$ such that
$\vg_h = \nabla G_h$, and HB initialized on the invariant manifold $\crl{(\btheta, \bv): \bv = - (1 - \beta)^{-1} \nabla L(\btheta) + h \boldsymbol{g}_h(\btheta)}$ is exactly gradient descent:
\begin{equation*}
\vtheta^{(n + 1)} = \vtheta^{(n)} - \frac{h}{1 - \beta} \nabla L_h(\vtheta^{(n)}), \quad L_h := L - h \beta (1 - \beta) G_h.
\end{equation*}
\end{theorem}

The proof is a bounding argument very heavy on long equations, so it is moved to \cref{sec:manifold-details}. \Cref{lem:t-maps-from-gamma-to-gamma} proves that indeed $T$ maps $\Gamma$ to itself, and \cref{lem:contraction-of-t} proves that $T$ is a contraction on $\Gamma$. Since $\Gamma$ is a complete metric space with the metric $\norm{\bg_1 - \bg_2}_{\Gamma}$, the contraction mapping principle implies that there is a unique fixed point $\bg_h \in \Gamma$ of the operator $T$, i.e., $T \bg_h = \bg_h$. Exponential attractivity is tackled in \cref{lem:exponential-attractivity}.

\begin{corollary}[Critical points]\label{cor:critical-points}
In the setting of \cref{th:manifold} and as long as $h$ is sufficiently small,
\begin{equation*}
\crl{\vtheta: \nabla L(\vtheta) = 0} = \crl{\vtheta: \nabla L_h(\vtheta) = 0}.
\end{equation*}
Moreover, at every common critical point $\vtheta_{*}$, the Hessians $\nabla^2 L_h(\vtheta_{*})$ and
$\nabla^2 L(\vtheta_{*})$ have the same numbers of positive, negative, and zero eigenvalues.
\end{corollary}

\begin{proof}
From the definition of $\vg_h$,
\begin{equation}\label{eq:EWKPPP}
\brk[\big]{\nabla L_h - (1 - \beta) \nabla L}
\prn[\bigg]{\vtheta - \frac{h}{1 - \beta} \nabla L_h(\vtheta)}
= \beta \nabla L_h(\vtheta).
\end{equation}
Thus, $\nabla L_h(\vtheta) = 0$ implies $\nabla L(\vtheta) = 0$.
Conversely, if $\nabla L(\vtheta) = \boldsymbol{0}$, then $(\vtheta, \boldsymbol{0})$ is a fixed state of HB.
The exponential attractivity bound in \cref{th:manifold} forces any fixed state to lie on the invariant manifold, where the update of $\vtheta$ is $-h (1 - \beta)^{-1} \nabla L_h(\vtheta)$, so $\nabla L_h(\vtheta) = \boldsymbol{0}$.

Fix a critical point $\vtheta_{*}$. Since $\vg_h \in \Gamma$, for sufficiently small $h$ we have
\begin{equation}\label{eq:jnbqNE}
h \norm{\nabla^2 L_h(\vtheta_*)}
\leq h D_2 + h^2 \beta (1 - \beta) \delta < (1 - \beta)^2.
\end{equation}
Differentiating \eqref{eq:EWKPPP} at $\vtheta_{*}$ and rearranging yields
\begin{equation*}
\nabla^2 L(\vtheta_*) \brk[\bigg]{I - \frac{h}{1 - \beta} \nabla^2 L_h(\vtheta_*)}
= \nabla^2 L_h(\vtheta_*)
\brk[\bigg]{I - \frac{h}{(1 - \beta)^2} \nabla^2 L_h(\vtheta_*)}.
\end{equation*}
We see that each eigenvector of $\nabla^2 L_h(\vtheta_{*})$ with eigenvalue $\mu$ is an eigenvector of $\nabla^2 L(\vtheta_{*})$ with eigenvalue
\begin{equation*}
\mu \frac{1 - h \mu / (1 - \beta)^2}{1 - h \mu / (1 - \beta)}.
\end{equation*}
Both the numerator and denominator of this fraction are positive by \eqref{eq:jnbqNE}, which means the eigenvalues have the same sign. Applying this fact to an orthonormal eigenbasis of $\nabla^2 L_h(\vtheta_{*})$ proves the claim.
\end{proof}

\subsection{Mini-Batch HB is GD with Step-Dependent Losses}\label{sec:step-dependent-losses}

In the previous section, we proved that full-batch HB is exactly GD with one modified loss, provided that the initialization is on the invariant manifold.
However, typically HB is initialized with $\vv^{(0)} = \boldsymbol{0}$, and a mini-batch version is used. In this section, we prove that in this case HB is also exactly GD but the loss is step-dependent. Note that while the modified loss depends on the step, it does not depend on the initialization $\vtheta^{(0)}$. Moreover, we provide a sharp condition on the step size for which this can be true.

\begin{proposition}\label{prop:exact-minibatch-gradient}
Assume that a family of loss functions $L^{(n)} \in C^2(\mathbb{R}^d, \mathbb{R})$ for $n \geq 0$ satisfies
\begin{equation*}
\sup_{n \geq 0} \sup_{\vtheta \in \mathbb{R}^d} \norm{\nabla^2 L^{(n)}(\vtheta)} \leq D_2.
\end{equation*}
If $h D_2 \leq (1 - \sqrt{\beta})^2$, the sequence of maps $\vF_n\colon \mathbb{R}^d \to \mathbb{R}^d$ given by
\begin{align}
  \vF_0(\vtheta) &= \vtheta - h \nabla L^{(0)}(\vtheta),\\
  \vF_n(\vtheta) &= (1 + \beta) \vtheta - h \nabla L^{(n)}(\vtheta) - \beta \vF_{n - 1}^{-1}(\vtheta), \quad n \geq 1, \label{eq:deqafB}
\end{align}
is well-defined and has a symmetric Jacobian $\nabla \vF_n$ for each $n$. The HB iteration \eqref{eq:hb-iter-minibatch} with the typical setting $\vv^{(0)} = \boldsymbol{0}$ satisfies
\begin{equation*}
\vtheta^{(n + 1)} = \vF_n(\vtheta^{(n)}), \quad n \geq 0.
\end{equation*}
As a consequence, this iteration is exactly gradient descent
\begin{equation*}
\vtheta^{(n + 1)} = \vtheta^{(n)} - h \nabla \tilde{L}^{(n)}(\vtheta^{(n)}), \quad n \geq 0,
\end{equation*}
with a step-dependent modified loss not depending on $\vtheta^{(0)}$:
\begin{equation*}
h \tilde{L}^{(n)}(\vtheta) = \frac{1}{2} \norm{\vtheta}^2 - \int_0^1 \vtheta\trans \vF_n(t \vtheta) \, \mathrm{d} t.
\end{equation*}

The step size condition $h D_2 \leq (1 - \sqrt{\beta})^2$ is sharp in the following sense. If $h D_2 > (1 - \sqrt{\beta})^2$, there is a scalar loss $L(\theta) = \lambda \theta^2 / 2$ with $0 < \lambda \leq D_2$ for which no family of maps $G_n\colon \mathbb{R} \to \mathbb{R}$ satisfies $\theta^{(n + 1)} = G_n(\theta^{(n)})$ for every initialization $\theta^{(0)} \in \mathbb{R}$ with $v^{(0)} = 0$.
\end{proposition}

\begin{proof}
We prove by induction that $\crl{\vF_n}$ are well-defined and have symmetric Jacobians satisfying
\begin{equation*}
\sqrt{\beta} I \preceq \nabla \vF_n(\vtheta) \preceq (1 + \beta + h D_2) I.
\end{equation*}
For $n = 0$, the bound
\begin{equation*}
(1 - h D_2) I \preceq \nabla \vF_0(\vtheta) \preceq (1 + h D_2) I, \quad 1 - h D_2 \geq \sqrt{\beta}
\end{equation*}
holds. Assume it holds for $n - 1$.
Then, the map
\begin{equation*}
\vtheta \mapsto \vtheta - \frac{1}{1 + \beta + h D_2} (\vF_{n - 1}(\vtheta) - \vzeta)
\end{equation*}
has a Jacobian with eigenvalues between zero and $1 - \frac{\sqrt{\beta}}{1 + \beta + h D_2}$,
thus it is a contraction on $\mathbb{R}^d$.
Then $\vF_{n - 1}(\vtheta) = \vzeta$ has a unique solution, which means $\vF_{n - 1}$ is invertible, and $\vF_n$ in \eqref{eq:deqafB} is well-defined.
Differentiating \eqref{eq:deqafB} gives
\begin{equation*}
\nabla \vF_n(\vtheta) = (1 + \beta) I - h \nabla^2 L^{(n)}(\vtheta) - \beta \brk[\big]{ \nabla \vF_{n - 1}(\vF_{n - 1}^{-1}(\vtheta)) }^{-1}.
\end{equation*}
This is a symmetric matrix since $\nabla^2 L^{(n)}(\vtheta)$ is symmetric and $\nabla \vF_{n - 1}$ is symmetric by the induction hypothesis.
Since the inverse matrix has nonnegative eigenvalues no larger than $1 / \sqrt{\beta}$, we have
\begin{align*}
  \nabla \vF_n(\vtheta) &\succeq (1 + \beta - h D_2 - \beta / \sqrt{\beta}) I \succeq \sqrt{\beta} I,
  \\
  \nabla \vF_n(\vtheta) &\preceq (1 + \beta + h D_2) I,
\end{align*}
concluding the induction step.

The symmetry of the Jacobian now implies that $\vF_n(\vtheta) = \vtheta - h \nabla \tilde{L}^{(n)}(\vtheta)$ by differentiating under the integral.
Since
\begin{equation*}
\vtheta^{(1)} = \vtheta^{(0)} - h \nabla L^{(0)}(\vtheta^{(0)}) = \vF_0(\vtheta^{(0)}),
\end{equation*}
and whenever $\vtheta^{(n)} = \vF_{n - 1}(\vtheta^{(n - 1)})$, we also have
\begin{align*}
  \vtheta^{(n + 1)} &= (1 + \beta) \vtheta^{(n)} - \beta \vtheta^{(n - 1)} - h \nabla L^{(n)}(\vtheta^{(n)})
  \\
                    &= (1 + \beta) \vtheta^{(n)} - \beta \vF_{n - 1}^{-1}(\vtheta^{(n)}) - h \nabla L^{(n)}(\vtheta^{(n)})
  \\
  &= \vF_n(\vtheta^{(n)}), \quad n \geq 1,
\end{align*}
we obtain the desired representation of the HB iteration.

We now prove the last claim.
Initialize with $\theta^{(0)} = 1$ and $v^{(0)} = 0$.
If $h \lambda = (1 - \sqrt{\beta})^2$, the solution is
\begin{equation*}
\theta^{(n)} = \prn[\big]{1 + n (1 - \sqrt{\beta})} \beta^{n / 2} > 0.
\end{equation*}
Choose $\lambda$ satisfying
\begin{equation*}
(1 - \sqrt{\beta})^2 < h \lambda < \min \crl{h D_2, 1}.
\end{equation*}
If all iterates $\theta^{(n)}$ were positive, we would have
\begin{align*}
  \frac{\theta^{(n)}}{\theta^{(n - 1)}} - \frac{\theta^{(n + 1)}}{\theta^{(n)}}
  &= \frac{\theta^{(n)}}{\theta^{(n - 1)}} + \beta \frac{\theta^{(n - 1)}}{\theta^{(n)}} - (1 + \beta - h \lambda)
  \\
  &\geq 2 \sqrt{\frac{\theta^{(n)}}{\theta^{(n - 1)}} \cdot \beta \frac{\theta^{(n - 1)}}{\theta^{(n)}}} - (1 + \beta - h \lambda)
  = h \lambda - (1 - \sqrt{\beta})^2 > 0,
\end{align*}
which does not decay with $n$. Thus the ratios decrease at least by a fixed amount, and will eventually become negative, contradiction. Fix such $n$ that $\theta^{(n)}$ is the first nonpositive iterate.
Decreasing $\lambda$ to $(1 - \sqrt{\beta})^2 / h$ would render $\theta^{(n)}$ positive; continuity in $\lambda$ implies that there is an intermediate value of $\lambda$ for which $\theta^{(n)} = 0$. Fix that $\lambda$ and note that
\begin{equation*}
\theta^{(n + 1)} = - \beta \theta^{(n - 1)} \neq 0,
\end{equation*}
because two consecutive zero iterates would imply that all previous iterates are zero as well, contradicting $\theta^{(0)} = 1$. Thus, if such $G_n$ exists, it must satisfy $G_n(0) = \theta^{(n + 1)} \neq 0$. But for a different initialization $\theta^{(0)} = 0$, we would have both $\theta^{(n)} = 0$ and $G_n(0) = \theta^{(n + 1)} = 0$, contradiction with the fact that $G_n$ works for all initializations.
\end{proof}

\section{Approximation}\label{sec:mini-batch-hb-approx-theorem}

In the form of \cref{prop:exact-minibatch-gradient}, the modified losses are difficult to analyze because they are defined recursively and not explicitly (involving the inverse of $\vF_{n - 1}$).
We now start the analysis of the memoryless approximations to HB, which will lead to more explicit descriptions of the (truncated) modified losses.

\subsection{Main Bound}\label{sec:approx-main-bound}

\citet{cattaneo2025howmemoryoptimization} introduced a method for removing memory in a class of numerical optimization algorithms with decaying memory, which can be applied to HB up to $O(h^2)$. Following their idea, we prove the approximation of HB by a memoryless iteration with the global error bound $O(h^{\ord})$ where $\ord \in \mathbb{Z}_{\geq 2}$. Considering higher-order approximations requires additional notation and technical work. Denoting
\begin{equation*}
\mathcal{K}_{i, l} = \crl[\big]{(k_0, \ldots, k_l) \in \mathbb{Z}_{\geq 0}^{l + 1}: k_0 + \ldots + k_l = i, k_1 + 2 k_2 + \ldots + l k_l = l},
\end{equation*}
define the \textit{memoryless iteration coefficients}
\begin{equation}\label{eq:nomem-coef-def}
  \begin{aligned}
    \nomemcoef{1}{n}(\btheta) &= - \sum_{k = 0}^n \beta^k \nabla L^{(n - k)}(\btheta),\\
    \nomemcoef{m}{n}(\btheta) &= - \sum_{k = 1}^n \beta^k \sum_{\substack{i, l \geq 0\\i + l = m - 1}} \sum_{(i_0, \ldots, i_l) \in \mathcal{K}_{i, l}} \frac{1}{i_0! \ldots i_l!}\\
    &\qquad \times \nabla^{i + 1} L^{(n - k)}(\btheta) \prn[\bigg]{\underbrace{\histcoef{1}{n}{k}(\btheta)}_{\text{$i_0$ times}}, \ldots, \underbrace{\histcoef{l + 1}{n}{k}(\btheta)}_{\text{$i_l$ times}}},\quad m \geq 2,
  \end{aligned}
\end{equation}
where the \textit{history terms} satisfy the iteration
\begin{equation}\label{eq:jfjl}
\histcoef{m}{n}{a}(\btheta) = - \sum_{s = 1}^a \sum_{\substack{j \geq 1, i, l \geq 0\\i + j + l = m}} \sum_{(k_0, \ldots, k_l) \in \mathcal{K}_{i, l}} \frac{1}{k_0! \ldots k_l!} \nabla^i \nomemcoef{j}{n - s}(\btheta) \prn[\bigg]{\underbrace{\histcoef{1}{n}{s}(\btheta)}_{\text{$k_0$ times}}, \ldots, \underbrace{\histcoef{l + 1}{n}{s}(\btheta)}_{\text{$k_l$ times}}},
\end{equation}
with $a \in \range{1}{n}$, and $m \in \mathbb{Z}_{\geq 1}$.

Although this recursion may look complex, we will use it as the main definition for $\crl{\nomemcoef{m}{n}(\btheta)}$ because it is convenient for proving \cref{th:approx}. It involves taking high-order derivatives of recursively defined quantities, specifically $\nabla^i \nomemcoef{j}{n - s}$ in \eqref{eq:jfjl}, and for this reason it is not ideal. We provide a different, more natural, form in \cref{sec:simpler-form}, where only the loss function is differentiated.

\begin{theorem}[Approximation by a memoryless iteration]\label{th:approx}
  Let a family of loss functions $\crl{L^{(n)}}_{n \geq 0}$ be defined on an open convex domain $\mathcal{D}$ in $\mathbb{R}^d$, and assume each loss function $L^{(n)}\colon \mathcal{D} \to \mathbb{R}$ is $2 \ord$-times continuously differentiable with bounded uniformly in $n$ derivatives up to order $2 \ord$, where $\ord \in \mathbb{Z}_{\geq 2}$. Let
$\crl[\big]{\btheta^{(n)}}_{n = 0}^{\infty} \subset \mathcal{D}$ be the HB iteration \eqref{eq:hb-iter-minibatch}
with initial condition $\btheta^{(0)} \in \mathcal{D}$,
and $\crl[\big]{\tilde{\btheta}^{(n)}}_{n = 0}^{\infty} \subset \mathcal{D}$ the iteration
\eqref{eq:memoryless-approx-contrib} with the same initial condition $\tilde{\btheta}^{(0)} = \btheta^{(0)}$. Let $T > 0$ be a fixed ``time'' horizon. Then, for each $h \in (0, 1 / 2)$,
\begin{equation}\label{eq:onestep-error-bound}
  \sup_{n \in \range{0}{\lfloor T / h \rfloor}} \norm[\bigg]{\tilde{\btheta}^{(n + 1)} - \tilde{\btheta}^{(n)} + h \sum_{k = 0}^n \beta^k \nabla L^{(n - k)}\prn[\big]{\tilde{\btheta}^{(n - k)}}} \leq \oC{clocal} h^{\ord + 1}
\end{equation}
and, as a consequence,
\begin{equation}\label{eq:global-error-bound}
\sup_{n \in \range{0}{\lfloor T / h \rfloor}} \norm[\big]{\btheta^{(n)} - \tilde{\btheta}^{(n)}} \leq \oC{cglobal} h^{\ord},
\end{equation}
where $\nC{clocal}$ and $\nC{cglobal}$ are some constants depending on $T$, $\beta$, $\ord$ and the uniform derivative bounds but are independent of $n$ and $h$.
\end{theorem}

To build intuition we offer a sketch of our proof next; the full argument can be found in \cref{sec:proof-of-approx-th}.

\subsubsection*{Proof Sketch}

By definition, \eqref{eq:onestep-error-bound} will follow from the following bound on the error introduced by removing memory:
\begin{equation}\label{eq:kwPSbq}
- \sum_{k = 0}^n \beta^k \nabla L^{(n - k)}\prn[\big]{\tilde{\btheta}^{(n - k)}} \overset{?}{=} \sum_{m = 0}^{\ord - 1} h^{m} \nomemcoef{m + 1}{n}\prn[\big]{\tilde{\btheta}^{(n)}} + O(h^{\ord}),
\end{equation}
that is, our task is to rewrite the left-hand side in such a way that instead of a function of all previous $\tilde{\btheta}^{(n - k)}$ it becomes just a function of $\tilde{\btheta}^{(n)}$.

It is shown by induction (\cref{lem:liuyf}) that $\nomemcoef{j}{n}(\btheta) = O(1)$ which implies $\tilde{\btheta}^{(n - k)} - \tilde{\btheta}^{(n)} = O(k h)$ (\cref{lem:qVNejQ}). This allows us to claim that the remainder is $\text{Rem}^{(n - k)} = O(k^{\ord} h^{\ord})$ in the Taylor expansion
\begin{equation}\label{eq:lqyfdV}
\begin{aligned}
  \nabla L^{(n - k)}\prn[\big]{\tilde{\btheta}^{(n - k)}} = &\sum_{i = 0}^{\ord - 1} \frac{1}{i!} \nabla^{i + 1} L^{(n - k)}\prn[\big]{\tilde{\btheta}^{(n)}} \prn[\big]{\underbrace{\tilde{\btheta}^{(n - k)} - \tilde{\btheta}^{(n)}, \ldots, \tilde{\btheta}^{(n - k)} - \tilde{\btheta}^{(n)}}_{\text{$i$ times}}}\\
  &+ \text{Rem}^{(n - k)},
\end{aligned}
\end{equation}
where
\begin{multline*}
  \text{Rem}^{(n - k)} = \frac{1}{(\ord - 1) !} \int_0^1 (1 - t)^{\ord - 1}\\
  \times \nabla^{\ord + 1} L^{(n - k)}\prn[\big]{\tilde{\btheta}^{(n)} + t \prn[\big]{\tilde{\btheta}^{(n - k)} - \tilde{\btheta}^{(n)}}} \prn[\big]{\underbrace{\tilde{\btheta}^{(n - k)} - \tilde{\btheta}^{(n)}, \ldots, \tilde{\btheta}^{(n - k)} - \tilde{\btheta}^{(n)}}_{\text{$\ord$ times}}}\,\mathrm{d} t.
\end{multline*}
This is partial progress: we rewrote $\nabla L^{(n - k)}\prn[\big]{\tilde{\btheta}^{(n - k)}}$ in such a way that it depends only multi-linearly on (copies of) $\tilde{\btheta}^{(n - k)} - \tilde{\btheta}^{(n)}$, whereas all loss derivatives are evaluated at the current iterate $\tilde{\btheta}^{(n)}$. Now we need to express $\tilde{\btheta}^{(n - k)}$ through $\tilde{\btheta}^{(n)}$. This is done in the following lemma, which clarifies the meaning of the history terms $\histcoef{m}{n}{k}(\btheta)$:
\begin{lemma}\label{lem:history-expansion}
We have for $r \in \range{1}{\ord}$, $k \in \range{1}{n}$
  \begin{equation}\label{eq:lkjru}
    \tilde{\btheta}^{(n - k)} = \tilde{\btheta}^{(n)} + \sum_{m = 1}^{r - 1} h^m \histcoef{m}{n}{k}\prn[\big]{\tilde{\btheta}^{(n)}} + O(k^r h^r).
  \end{equation}
\end{lemma}
The proof is by induction and given in \cref{sec:proof-of-approx-th}. Inserting this history expansion \eqref{eq:lkjru} into the main term in the right-hand side of \eqref{eq:lqyfdV} and carefully keeping track of the error gives
\begin{align*}
  &\nabla L^{(n - k)}\prn[\big]{\tilde{\btheta}^{(n - k)}}\\
  &\quad= \sum_{m = 0}^{\ord - 1} h^m \sum_{\substack{i, l \geq 0\\i + l = m}} \sum_{(i_0, \ldots, i_l) \in \mathcal{K}_{i, l}} \frac{1}{i_0! \ldots i_l!} \nabla^{i + 1} L^{(n - k)}\prn[\big]{\tilde{\btheta}^{(n)}} \prn[\bigg]{\underbrace{\histcoef{1}{n}{k}\prn[\big]{\tilde{\btheta}^{(n)}}}_{\text{$i_0$ times}}, \ldots, \underbrace{\histcoef{l + 1}{n}{k}\prn[\big]{\tilde{\btheta}^{(n)}}}_{\text{$i_l$ times}}}\\
  &\qquad + O(k^{\ord} h^{\ord}).
\end{align*}

It is left to sum this over $k$ with an exponentially decaying weight $\beta^k$.
The error $O(k^{\ord} h^{\ord})$ is polynomial in $k$ but it is not a problem because it will turn into $\sum_{k = 0}^n \beta^k O(k^{\ord} h^{\ord}) = O(h^{\ord})$ by exponential summation. The coefficients $\nomemcoef{m}{n}\prn[\big]{\btheta}$ are defined in such a way that the result is exactly \eqref{eq:kwPSbq} that we need, proving the local error bound \eqref{eq:onestep-error-bound}. A standard argument (\cref{lem:global-error-bound}) for converting a local error bound to a global error bound gives~\eqref{eq:global-error-bound}.

\subsection{Simpler Form of Memoryless Iteration Coefficients}\label{sec:simpler-form}

In this \lcnamecref{sec:simpler-form}, we will free ourselves from the complicated expressions used in \cref{sec:mini-batch-hb-approx-theorem} to prove \cref{th:approx}. We will write a much simpler recursion for the memoryless iteration coefficients $\nomemcoef{m}{n}$, which will be convenient for analyzing their properties. As a result of this freedom, we will almost automatically obtain the principal iteration and recognize Narayana polynomials in its additive terms.

Let us write the first few terms $\nomemcoef{m}{n}(\btheta)$ from~\eqref{eq:nomem-coef-def} and describe the pattern that can be used to generate them. To declutter notation, we will omit the argument $\btheta$ since it will be fixed. By definition, the first memoryless iteration coefficient is just the exponential average of past gradients:
\begin{equation*}
\nomemcoef{1}{n} = - \sum_{b = 0}^n \beta^b \nabla L^{(n - b)}.
\end{equation*}

Using the definition, we can also write the following expression for $\nomemcoef{2}{n}$:
\begin{equation*}
\nomemcoef{2}{n} = - \beta \sum_{b = 0}^{n - 1} \beta^{b} \nabla^2 L^{(n - 1 - b)} \sum_{l' = 1}^{b + 1} \sum_{b' = 0}^{n - l'} \beta^{b'} \nabla L^{(n - l' - b')}.
\end{equation*}
This triple sum on the right can be generated as follows.
Write the rooted tree consisting of two nodes $(1, 2)$ with corresponding parents $(\varnothing, 1)$. Let us introduce a variable $l$ with value $1$, which we will call the memory distance variable. Corresponding to the root $1$, write the symbolic ``sum''
\begin{equation}\label{eq:lkjlkfjy}
\sum_{b = 0}^{n - l} \beta^b \nabla^2 L^{(n - l - b)} \sum_{l' = 1}^{b + 1} \square = \sum_{b = 0}^{n - 1} \beta^b \nabla^2 L^{(n - 1 - b)} \sum_{l' = 1}^{b + 1} \square
\end{equation}
and call $l'$ the new memory distance variable (it is now being summed over, so it does not have a fixed value). We write the second derivative tensor $\nabla^2 L$ (matrix in this case) because the root has one child; later, the order of the derivative will be $\ell + 1$ where $\ell$ is the number of children.
Let us go down the tree and consider node $2$. Replace $\square$ in \eqref{eq:lkjlkfjy} with the corresponding expression:
\begin{equation*}
\sum_{b' = 0}^{n - l'} \beta^{b'} \nabla L^{(n - l' - b')}.
\end{equation*}
Again, the order of the derivative $\nabla L$ is equal to the number of children (zero) plus one. The upper limit of the sum is $n$ minus the current memory distance variable. We do not write a trailing sum at the end like in \eqref{eq:lkjlkfjy} because there are no children.

We will further illustrate this process by looking at $\nomemcoef{3}{n}$. Using $\eqref{eq:nomem-coef-def}$ again and after some algebra, we can write
\begin{align}
  \nomemcoef{3}{n} = &- \beta \sum_{b = 0}^{n - 1} \beta^b \nabla^2 L^{(n - 1 - b)} \sum_{l' = 1}^{b + 1} \sum_{b' = 0}^{n - l'} \beta^{b'} \nabla^2 L^{(n - l' - b')} \sum_{l'' = 1}^{l' + b'} \sum_{b'' = 0}^{n - l''} \beta^{b''} \nabla L^{(n - l'' - b'')}\label{eq:nomemcoef3first}\\
  &- \frac{\beta}{2} \sum_{b = 0}^{n - 1} \beta^b \nabla^3 L^{(n - b - 1)} \brk[\bigg]{\sum_{l' = 1}^{b + 1} \sum_{b' = 0}^{n - l'} \beta^{b'} \nabla L^{(n - l' - b')}, \sum_{l' = 1}^{b + 1} \sum_{b' = 0}^{n - l'} \beta^{b'} \nabla L^{(n - l' - b')}}.\label{eq:nomemcoef3second}
\end{align}

There are two non-isomorphic rooted trees with $3$ vertices. The first one is a ``chain'': the nodes $(1, 2, 3)$ have corresponding parents $(\varnothing, 1, 2)$. Let us now describe a sum that corresponds to this tree. As previously, the current memory distance is $l = 1$, and the root $1$ (having one child) generates a symbolic expression
\begin{equation*}
\sum_{b = 0}^{n - 1} \beta^b \nabla^2 L^{(n - l - b)} \sum_{l' = 1}^{b + 1} \square = \sum_{b = 0}^{n - 1} \beta^b \nabla^2 L^{(n - 1 - b)} \sum_{l' = 1}^{b + 1} \square.
\end{equation*}
The current memory distance variable is now $l'$ (with no fixed value).
Next comes node $2$ with one child, generating a symbolic expression
\begin{equation*}
\sum_{b' = 0}^{n - l'} \beta^{b'} \nabla^2 L^{(n - l' - b')} \sum_{l'' = 1}^{l' + b'} \square.
\end{equation*}
The current memory distance is $l''$ (with no fixed value). Finally, node $3$ has no children, so it closes the sum with the expression
\begin{equation*}
\sum_{b'' = 0}^{n - l''} \beta^{b''} \nabla L^{(n - l'' - b'')}.
\end{equation*}
Up to coefficient $-\beta$ in front, we have obtained \eqref{eq:nomemcoef3first}.

The second rooted tree with $3$ vertices is the tree consisting of nodes $(1, 2, 3)$ with corresponding parents $(\varnothing, 1, 1)$ (root with two children). The initial memory distance variable is denoted as $l$ and has value $1$. The root has two children, so the order of the derivative corresponding to it will be $3$, and we will write two $\square$ signs corresponding to two subtrees:
\begin{equation*}
\sum_{b = 0}^{n - l} \beta^b \nabla^3 L^{(n - b - 1)} \brk[\bigg]{\sum_{l' = 1}^{b + 1} \square, \sum_{l' = 1}^{b + 1} \square}.
\end{equation*}
Then we replace the first $\square$ with the expression corresponding to node $2$ (with initial memory variable $l'$):
\begin{equation*}
\sum_{b' = 0}^{n - l'} \beta^{b'} \nabla L^{(n - l' - b')},
\end{equation*}
and the second $\square$ with the same expression corresponding to node $3$. Up to coefficient $- \beta / 2$ in front, we have obtained \eqref{eq:nomemcoef3second}. The reason for the division by $2$ is that this tree has a symmetry coefficient $2$.

This consideration of special cases suggests a pattern in how the memoryless iteration coefficients are structured. We formalize this pattern in the next result.

\begin{theorem}[The form of memoryless iteration coefficients]
  \label{th:nomem-coef}
  For $m \geq 2$, the memoryless iteration coefficient $\nomemcoef{m}{n}$ is equal (up to a coefficient $- \beta$) to a sum over the set $\unlabrooted[m]$ of unlabeled rooted trees with $m$ vertices:
\begin{equation*}
\nomemcoef{m}{n} = - \beta \sum_{\tau \in \unlabrooted[m]} \frac{1}{\sigma(\tau)} \E{\tau}{1}{n},
\end{equation*}
where $\sigma(\tau)$ is the symmetry coefficient of the tree (background in \cref{sec:Technical Background and Notation}).
The expression $\EDef{\tau}{l}{n} := \E{\tau}{l}{n}(\beta)$, depending on the iteration number $n$ and memory distance variable $l$, is defined recursively by
\begin{equation}\label{eq:e-tau-def}
\E{\tau}{l}{n} = \sum_{b = 0}^{n - l} \beta^b \nabla^{\ell + 1} L^{(n - l - b)} \brk[\bigg]{\sum_{l_1 = 1}^{l + b} \E{\tau_1}{l_1}{n}, \ldots, \sum_{l_{\ell} = 1}^{l + b} \E{\tau_{\ell}}{l_{\ell}}{n}}, \quad l \in \range{1}{n},
\end{equation}
where $(\tau_1, \ldots, \tau_{\ell})$ are the subtrees rooted at the children of the root of $\tau$ (with $\ell \in \mathbb{Z}_{\geq 0}$). In particular, $\E{\rootedtree[]}{l}{n} = \sum_{b = 0}^{n - l} \beta^b \nabla L^{(n - l - b)}$; for $l > n$, the sum is empty and $\E{\tau}{l}{n} = 0$.
\end{theorem}

The $\E{\tau}{l}{n}$ expression is an exact translation of our informal observations above into the mathematical language. For example,
\begin{equation*}
\E{\rootedtree[[]]}{l}{n} = \sum_{b = 0}^{n - l} \beta^b \nabla^2 L^{(n - l - b)} \sum_{l_1 = 1}^{l + b} \sum_{b_1 = 0}^{n - l_1} \beta^{b_1} \nabla L^{(n - l_1 - b_1)}.
\end{equation*}
We outline the main ideas underlying the proof next, but the formal details are deferred to \cref{sec:omitted-details-nomem-coef}.

Before sketching the argument for establishing \cref{th:nomem-coef}, we introduce some preliminary results. We use the notation and concepts from \cref{sec:Technical Background and Notation}.

\subsubsection*{Auxiliary Quantity Related to \eqref{eq:e-tau-def}}

Fix a labeled rooted tree $\tau$ in $\labrooted\range{1}{m}$, and choose a marking $\mrk \in \mrkset_{\tau}$, where the marked vertices are $v_1, \ldots, v_{\abs{\mrk}}$ with corresponding subtrees $\tau^{\mathcal{m}}_1, \ldots, \tau^{\mathcal{m}}_{\abs{\mrk}}$ and the remaining subtree $\tau_0^{\mrk}$.
Let $p_1, \ldots, p_{\abs{\mrk}}$ be the parents of the marked vertices (not necessarily distinct). Consider the derivative of the loss corresponding to $p_1$ in the symbolic expression for $\E{\tau_0^{\mrk}}{a}{n - l}$. Add one to the order of the derivative and add $\sum_{l' = 1}^l \E{\tau_1^{\mrk}}{l'}{n}$ as an argument. Continue this process (possibly increasing the order of the same derivative more than once) until all marked vertices are processed. We will denote the resulting expression by
\begin{equation*}
\EmarkDef{\tau}{\mrk}{a}{n - l \to n}.
\end{equation*}

For example, consider the tree $\tau$ consisting of the root and two leaves, with one of the leaves marked:
\begin{equation*}
\bigrootedtree[1[2][3,text=red]]
\end{equation*}

Consider the loss derivative corresponding to vertex $p_1 = 1$ in the symbolic expression
\begin{equation*}
\E{\rootedtree[[]]}{a}{n - l} = \sum_{b = 0}^{n - l - a} \beta^b \underbrace{\nabla^2 L^{(n - l - b - a)}} \sum_{l_1 = 1}^{a + b} \sum_{b_1 = 0}^{n - l - l_1} \beta^{b_1} \nabla L^{(n - l - l_1 - b_1)}.
\end{equation*}
Increase the order of the derivative and insert $\sum_{l' = 1}^l \E{\rootedtree[]}{l'}{n}$, giving
\begin{equation*}
\sum_{b = 0}^{n - l - a} \beta^b \nabla^3 L^{(n - l - b - a)} \brk[\bigg]{ \sum_{l_1 = 1}^{a + b} \sum_{b_1 = 0}^{n - l - l_1} \beta^{b_1} \nabla L^{(n - l - l_1 - b_1)}, \sum_{l' = 1}^l \E{\rootedtree[]}{l'}{n}}.
\end{equation*}

\subsubsection*{Useful Properties of $\E{\tau}{l}{n}$}

We give two lemmas about $\E{\tau}{l}{n}$ that we will use in the argument. Both lemmas are proven in \cref{sec:omitted-details-nomem-coef}. The following important fact is the reason why the induction step in the main argument goes through.

\begin{lemma}\label{lem:sum-over-markings-of-one-tree}
Let $m \geq 1$, $n \geq 0$, $l \in \range{0}{n}$, $a \geq 1$. For any $\tau \in \labrooted\range{1}{m}$ we have
\begin{equation}\label{eq:sum-over-markings-of-one-tree}
\E{\tau}{l + a}{n} = \sum_{\mrk \in \mrkset_{\tau}} \Emark{\tau}{\mrk}{a}{n - l \to n}.
\end{equation}
\end{lemma}

The reason why we invoke marked trees is that differentiation naturally creates such trees. The following \lcnamecref{lem:kjjhhag} establishes a connection between the sum over marked trees in \eqref{eq:sum-over-markings-of-one-tree} and the high-order derivative tensor which arises in the main argument for \cref{th:nomem-coef}.

\begin{lemma}\label{lem:kjjhhag}
  We have for $m \geq 2$, $n \geq 0$, $l \in \range{0}{n}$, $a \geq 1$
\begin{align*}
  &\sum_{i = 0}^{m - 1} \sum_{j = 1}^{m - i} \sum_{(k_0, \ldots, k_{m - j - i}) \in \mathcal{K}_{i, m - j - i}} \frac{1}{k_0! \ldots k_{m - j - i}!} \sum_{\tau_0 \in \unlabrooted[j]} \frac{1}{\sigma(\tau_0)} \\
  &\qquad \times \nabla^i \E{\tau_0}{a}{n - l} \brk[\bigg]{
  \underbrace{\sum_{\tau \in \unlabrooted[1]}\frac{1}{\sigma(\tau)} \sum_{l' = 1}^l \E{\tau}{l'}{n}}_{\text{$k_0$ times}},
  \ldots,
  \underbrace{\sum_{\tau \in \unlabrooted[m - j - i + 1]}\frac{1}{\sigma(\tau)} \sum_{l' = 1}^l \E{\tau}{l'}{n}}_{\text{$k_{m - j - i}$ times}}
  } = \sum_{\tau \in \unlabrooted[m]} \frac{1}{\sigma(\tau)} \E{\tau}{l + a}{n}.
\end{align*}
\end{lemma}

\begin{remark}
Note that from the definition of $\E{\tau}{l}{n}$,
\begin{equation}\label{eq:nJdufq}
  \nabla^{\ell + 1} L^{(n - l)}
  \brk[\bigg]{
    \sum_{l' = 1}^l \E{\tau_1}{l'}{n},
    \ldots,
    \sum_{l' = 1}^l \E{\tau_\ell}{l'}{n}
  } + \beta \E{\tau}{l + 1}{n} = \E{\tau}{l}{n}.
\end{equation}

Replacing $\E{\tau}{l + 1}{n}$ with $\sum_{\mrk \in \mrkset_{\tau}} \Emark{\tau}{\mrk}{1}{n - l \to n}$ by \cref{lem:sum-over-markings-of-one-tree} and setting $l = 1$, we get the following alternative recursion for $n \geq 1$:
\begin{equation}\label{eq:lQbzcy}
\E{\tau}{1}{n} = \nabla^{\ell + 1} L^{(n - 1)} \brk[\big]{\E{\tau_1}{1}{n}, \ldots, \E{\tau_{\ell}}{1}{n}} + \beta \sum_{\mrk \in \mrkset_{\tau}} \Emark{\tau}{\mrk}{1}{n - 1 \to n}.
\end{equation}
The advantage of this form is that the memory distance variable is always $1$, but the disadvantage is that the right-hand side contains a sum over all markings of $\tau$.
\end{remark}

\subsubsection*{Proof Sketch of \cref{th:nomem-coef}}

The strategy is to prove the following two statements simultaneously by induction over $m \geq 2$:
\begin{align}
  \nomemcoef{m}{n} &= - \beta \sum_{\tau \in \unlabrooted[m]} \frac{1}{\sigma(\tau)} \E{\tau}{1}{n},\quad\text{and}
                     \label{eq:nomemcoef-m-n-ind}
  \\
  \histcoef{m}{n}{k} &= \sum_{l = 1}^k \sum_{\tau \in \unlabrooted[m]} \frac{1}{\sigma(\tau)} \E{\tau}{l}{n}.\label{eq:histcoef-m-n-k-ind}
\end{align}

For $m = 2$, they are easily verifiable. Note also that the second statement holds for $m = 1$ as well:
\begin{equation*}
\histcoef{1}{n}{k} = \sum_{l = 1}^k \E{\rootedtree[]}{l}{n}.
\end{equation*}

By definition,
\begin{align*}
  &\nomemcoef{m}{n}
  =
  - \beta \sum_{b = 0}^{n - 1} \beta^b \sum_{\ell = 0}^{m - 1} \sum_{(i_0, \ldots, i_{m - 1 - \ell}) \in \mathcal{K}_{\ell, m - 1 - \ell}} \frac{1}{i_0! \ldots i_{m - 1 - \ell}!} \times\\
  &\qquad \qquad \times \nabla^{\ell + 1} L^{(n - 1 - b)} \prn[\bigg]{\underbrace{\histcoef{1}{n}{b + 1}}_{\text{$i_0$ times}}, \ldots, \underbrace{\histcoef{m - \ell}{n}{b + 1}}_{\text{$i_{m - 1 - \ell}$ times}}}.
\end{align*}
Insert the induction hypothesis (recall that \eqref{eq:histcoef-m-n-k-ind} holds for $m = 1$ too):
\begin{align*}
  &\nomemcoef{m}{n}
  =
  - \beta \sum_{b = 0}^{n - 1} \beta^b \sum_{\ell = 0}^{m - 1} \sum_{(i_0, \ldots, i_{m - 1 - \ell}) \in \mathcal{K}_{\ell, m - 1 - \ell}} \frac{1}{i_0! \ldots i_{m - 1 - \ell}!}\\
  &\qquad \times \nabla^{\ell + 1} L^{(n - 1 - b)} \prn[\bigg]{\underbrace{\sum_{l = 1}^{b + 1} \sum_{\tau \in \unlabrooted[1]} \frac{1}{\sigma(\tau)} \E{\tau}{l}{n}}_{\text{$i_0$ times}}, \ldots, \underbrace{\sum_{l = 1}^{b + 1} \sum_{\tau \in \unlabrooted[m - \ell]} \frac{1}{\sigma(\tau)} \E{\tau}{l}{n}}_{\text{$i_{m - 1 - \ell}$ times}}}.
\end{align*}
Careful rearrangement is used to simplify this to
\begin{equation*}
  \nomemcoef{m}{n} = - \beta \sum_{b = 0}^{n - 1} \beta^b \sum_{\ell = 0}^{m - 1}
  \sum_{\substack{\tau = [\tau_1, \ldots, \tau_{\ell}] \in \unlabrooted[m]}} \frac{1}{\sigma(\tau)} \nabla^{\ell + 1} L^{(n - 1 - b)} \prn[\bigg]{\sum_{l = 1}^{b + 1} \E{\tau_1}{l}{n}, \ldots, \sum_{l = 1}^{b + 1} \E{\tau_{\ell}}{l}{n}},
\end{equation*}
and that is, by definition, equal to $- \beta \sum_{\tau \in \unlabrooted[m]} \frac{1}{\sigma(\tau)} \E{\tau}{1}{n}$. Hence, under the induction hypothesis for smaller $m$, \eqref{eq:nomemcoef-m-n-ind} holds.

By definition of the history terms in \eqref{eq:jfjl} and the induction hypothesis, to prove \eqref{eq:histcoef-m-n-k-ind} and complete the induction step, it is enough to show that
\begin{equation}\label{eq:aux-fact-in-nonmem-coef}
\begin{aligned}
  &-
  \sum_{j = 1}^m \sum_{i = 0}^{m - j}
  \sum_{(k_0, \ldots, k_{m - i - j}) \in \mathcal{K}_{i, m - i - j}} \frac{1}{k_0! \ldots k_{m - i - j}!}\\
  &\qquad \times \nabla^i \nomemcoef{j}{n - l}
  \brk[\bigg]{
    \underbrace{
      \sum_{l' = 1}^l \sum_{\tau \in \unlabrooted[1]} \frac{1}{\sigma(\tau)} \E{\tau}{l'}{n}
    }_{\text{$k_0$ times}},
    \ldots,
    \underbrace{
      \sum_{l' = 1}^l \sum_{\tau \in \unlabrooted[m - i - j + 1]} \frac{1}{\sigma(\tau)} \E{\tau}{l'}{n}
    }_{\text{$k_{m - i - j}$ times}}
  }
  \overset{?}{=} \sum_{\tau \in \unlabrooted[m]} \frac{1}{\sigma(\tau)} \E{\tau}{l}{n}.
\end{aligned}
\end{equation}

For $j \geq 2$, by the induction hypothesis and \eqref{eq:nomemcoef-m-n-ind} already proven, we can replace $\nabla^i \nomemcoef{j}{n - l}$ in the left-hand side of \eqref{eq:aux-fact-in-nonmem-coef} by $- \beta \nabla^i \sum_{\tau_0 \in \unlabrooted[j]} \frac{1}{\sigma(\tau_0)} \E{\tau_0}{1}{n - l}$;
for $j = 1$, by $-\nabla^{i+1}L^{(n-l)} -\beta \nabla^i \E{\rootedtree[]}{1}{n-l}$.
The result will involve precisely the big sum that we saw in \cref{lem:kjjhhag}, applying which we will simplify the left-hand side of \eqref{eq:aux-fact-in-nonmem-coef} to
\begin{align*}
\sum_{i = 0}^{m - 1}
    \sum_{\substack{\tau \in \unlabrooted[m]\\ \tau = [\tau_1, \ldots, \tau_i]}} \frac{1}{\sigma(\tau)} \nabla^{i + 1} L^{(n - l)}
    \brk[\bigg]{
    \sum_{l' = 1}^l \E{\tau_1}{l'}{n},
    \ldots,
    \sum_{l' = 1}^l \E{\tau_i}{l'}{n}
  }
  +
  \beta \sum_{\tau \in \unlabrooted[m]} \frac{1}{\sigma(\tau)} \E{\tau}{l + 1}{n}.
\end{align*}
Combining this with \eqref{eq:nJdufq}, we see that the left-hand side of \eqref{eq:aux-fact-in-nonmem-coef} is equal to the right-hand side of \eqref{eq:aux-fact-in-nonmem-coef}. This completes the induction step. Omitted technical details are given in \cref{sec:omitted-details-nomem-coef}.

\subsection{Specializing to the Full-Batch Case}\label{sec:specializing-full-batch-case}

A useful special case of \cref{th:nomem-coef} is when the loss is the same at all iterations, $L^{(n)} \equiv L$. In this case, the expression $\E{\tau}{l}{n}$ in \eqref{eq:e-tau-def} can be factored into a scalar coefficient and an elementary differential of the loss:
\begin{equation}\label{eq:full-batch-E-factorization}
  \E{\tau}{l}{n}
  = \frac{\e{\tau}{l}{n}(\beta)}{(1 - \beta)^{2\abs{\tau} - 1}}
  \nabla^\tau L,
  \quad 1 \leq l \leq n,
\end{equation}
where $\e{\tau}{l}{n}(\beta)$ are defined below.

\begin{definition}\label{def:e-tau-l-n}
For $\tau = [\tau_1, \ldots, \tau_{\ell}] \in \unlabrooted[m]$, $n \in \mathbb{Z}_{\geq 0}$, $l \in \mathbb{Z}_{\geq 1}$, define $\e{\tau}{l}{n}(\beta)$ recursively by
\begin{equation}\label{eq:finite-e-recursion}
  \e{\tau}{l}{n}(\beta)
  := (1 - \beta)^{\ell + 1}
  \sum_{b = 0}^{n - l} \beta^b
  \prod_{i = 1}^{\ell}
  \prn[\bigg]{\sum_{l_i = 1}^{l + b} \e{\tau_i}{l_i}{n}(\beta)}.
\end{equation}
The outer sum is zero when $l > n$, and the empty product is one.
In particular,
\begin{equation*}
  \e{\rootedtree[]}{l}{n}(\beta)
  = (1 - \beta) \sum_{b = 0}^{n - l} \beta^b
  = 1 - \beta^{n - l + 1},
  \quad 1 \leq l \leq n.\qedhere
\end{equation*}
\end{definition}

By induction over the size of the tree, $\e{\tau}{l}{n}$ is a polynomial in $\beta$.
By \cref{th:nomem-coef} applied to the full-batch case, we have
\begin{equation}\label{eq:full-batch-nomem-coef-finite}
  \begin{aligned}
    \nomemcoef{1}{n}
    &= -\frac{1 - \beta^{n + 1}}{1 - \beta}\nabla L,
    \\
    \nomemcoef{m}{n}
    &= -\frac{\beta}{(1 - \beta)^{2m - 1}}
    \sum_{\tau \in \unlabrooted[m]}
    \frac{\e{\tau}{1}{n}(\beta)}{\sigma(\tau)}
    \nabla^\tau L,
    \quad m \geq 2.
  \end{aligned}
\end{equation}

Another simple induction verifies that $\e{\tau}{l}{n}(\beta)$ is non-decreasing and bounded in $n$, specifically $0 \leq \e{\tau}{l}{n}(\beta) \leq C_{\tau, \beta} (1 + l)^{\abs{\tau} - 1}$; it will therefore be useful to define large-$n$ limits of the quantities above.

\begin{definition}\label{def:e-coef}
For $\tau = [\tau_1, \ldots, \tau_{\ell}]
\in \unlabrooted[m]$ and $l \in \mathbb{Z}_{\geq 1}$,
define $e_{\tau,l}(\beta) \equiv \e{\tau}{l}{\infty}(\beta) \equiv e_{\tau,l}$ recursively by
\begin{equation*}
  e_{\tau,l}(\beta)
  =
  (1-\beta)^{\ell+1}
  \sum_{b=0}^{\infty}\beta^b
  \prod_{i=1}^{\ell}
  \prn[\bigg]{
    \sum_{l_i=1}^{l+b}e_{\tau_i,l_i}(\beta)
  }.
\end{equation*}
In particular, $e_{\rootedtree[],l} = 1$.
In the full-batch case $L^{(n)} \equiv L$, the corresponding
\textit{limiting memoryless iteration coefficients} are
\begin{equation}\label{eq:nomem-coef-limiting}
  \nomemcoef{m}{\infty}
  := \lim_{n \to \infty}\nomemcoef{m}{n},
  \quad m \geq 1,
\end{equation}
where the limit is taken at fixed $\vtheta$.
\end{definition}

Explicitly,
\begin{equation}\label{eq:nomem-coef-limiting-explicit}
\begin{aligned}
  \nomemcoef{1}{\infty}
  &= -\frac{1}{1-\beta}\nabla L,
  \\
  \nomemcoef{m}{\infty}
  &= -\frac{\beta}{(1-\beta)^{2m - 1}}
  \sum_{\tau \in \unlabrooted[m]}
  \frac{e_{\tau,1}(\beta)}{\sigma(\tau)}
  \nabla^\tau L,
  \quad m \geq 2.
\end{aligned}
\end{equation}

Equipped with these notations, we now begin investigating the form of the memoryless iteration coefficients, starting with ``principal'' terms.

\subsection{Principal Iteration}\label{sec:principal-iter}

Consider full-batch HB \eqref{eq:hb-iter-compact}. Let us formally take $\ord$ to infinity in \eqref{eq:memoryless-approx-contrib} and write a formal series $\sum_{m = 1}^{\infty} h^m \nomemcoef{m}{n}$. It is an infinite sum of terms that are of two types: terms containing only derivatives of order no higher than two of the loss, which we will call \textit{principal terms}, and the remaining terms (containing derivatives of order at least three), which we will call \textit{non-principal} ones. For example, the term in \eqref{eq:nomemcoef3first} is principal and the term in \eqref{eq:nomemcoef3second} is non-principal.

Write
\begin{equation*}
  \sum_{m = 2}^{\infty} h^m \nomemcoef{m}{n}
  = - \beta \sum_{m = 2}^{\infty} \frac{\e{\mathcal{c}_m}{1}{n}(\beta)}{(1 - \beta)^{2 m - 1}} h^m \crl{\nabla^2 L}^{m - 1} \nabla L + \np,
\end{equation*}
where the notation $\np$ means ``non-principal terms'', $\crl[\Big]{\frac{\e{\mathcal{c}_m}{1}{n}(\beta)}{(1 - \beta)^{2 m - 1}}}$ are coefficients
corresponding to the chain $\mathcal{c}_m$ with $m$ vertices. (Since the derivatives of order higher than two are not allowed, no vertex is allowed to have more than one child, leaving only $\mathcal{c}_m$.)

\begin{corollary}[Principal iteration]\label{cor:principal-coef}
The chain-parametrized polynomials $\e{\mathcal{c}_m}{1}{n}(\beta)$ satisfy,
\begin{align}
  &\e{\mathcal{c}_1}{1}{n} = 1 - \beta^n, \quad n \geq 0, \quad \e{\mathcal{c}_m}{1}{0} = 0, \quad m \geq 1,
  \\
  &\e{\mathcal{c}_m}{1}{n}
    = (1 - \beta)^2 \e{\mathcal{c}_{m - 1}}{1}{n} + \beta \e{\mathcal{c}_m}{1}{n - 1}
    + \beta (1 - \beta)
    \sum_{j = 1}^{m - 1} \e{\mathcal{c}_j}{1}{n - 1} \e{\mathcal{c}_{m - j}}{1}{n}, \quad m \geq 2, n \geq 1. \label{eq:NMVJYZ}
\end{align}
Accordingly, their limits satisfy
\begin{align}
  &e_{\mathcal{c}_1, 1} = 1, \label{eq:nBkwSo}
  \\
  &e_{\mathcal{c}_m, 1}
    = (1 - \beta) e_{\mathcal{c}_{m - 1}, 1}
    + \beta \sum_{j = 1}^{m - 1} e_{\mathcal{c}_j, 1} e_{\mathcal{c}_{m - j}, 1}, \quad m
    \geq 2. \label{eq:nBkwSo1}
\end{align}
Consequently,
\begin{equation*}
  e_{\mathcal{c}_{m + 1}, 1} = N_m(\beta), \quad m \geq 1,
\end{equation*}
where
\begin{equation*}
  N_m(\beta)
  := \sum_{k = 1}^m \frac{1}{m} \binom{m}{k} \binom{m}{k - 1} \beta^{m - k}, \quad m
    \geq 1
\end{equation*}
are the Narayana polynomials, and the generating function $g_{\beta}(x) := \sum_{m = 0}^{\infty} \frac{e_{\mathcal{c}_{m + 1}, 1}(\beta)}{(1 - \beta)^{2 m + 1}} x^m$ is given by
\begin{equation}
  \label{eq:mDeSYH} g_{\beta}(x)
    = \frac{1 - \beta - x - \sqrt{(1 - \beta - x)^2 - 4 \beta x}}{2 \beta x} .
\end{equation}
\end{corollary}

\begin{proof}
Applying \cref{def:e-tau-l-n} to the chain $\mathcal{c}_m$ gives, for $m \geq 2$,
\begin{equation}\label{eq:vm-l-rec}
\e{\mathcal{c}_m}{l}{n} = (1 - \beta)^2 \sum_{b = 0}^{n - l} \beta^b \sum_{l_1 = 1}^{l + b} \e{\mathcal{c}_{m - 1}}{l_1}{n}.
\end{equation}
Applying \eqref{eq:nJdufq} with $l = 1$ gives
\begin{equation*}
  \nabla^2 L \E{\mathcal{c}_{m - 1}}{1}{n} + \beta \E{\mathcal{c}_m}{2}{n} = \E{\mathcal{c}_m}{1}{n}.
\end{equation*}
By \cref{lem:sum-over-markings-of-one-tree} with $l = a = 1$, we have also
\begin{equation*}
\E{\mathcal{c}_m}{2}{n} = \sum_{\mrk \in \mrkset_{\mathcal{c}_m}} \Emark{\mathcal{c}_m}{\mrk}{1}{n - 1 \to n}.
\end{equation*}
Combining and using the definition of $\Emark{\mathcal{c}_m}{\mrk}{1}{n - 1 \to n}$ gives \eqref{eq:NMVJYZ}. \Cref{eq:nBkwSo,eq:nBkwSo1} are immediate by taking the limit $n \to \infty$.

The Narayana polynomials satisfy the recurrence (e.\,g. \citet{coker2003enumerating})
\begin{equation*}
N_m(\beta) = (1 + \beta) N_{m - 1}(\beta) + \beta \sum_{k = 1}^{m - 2} N_k(\beta) N_{m - k - 1}(\beta)
\end{equation*}
for all $m \geq 3$, which means that $e_{\mathcal{c}_{m + 1}, 1}$ and $N_m(\beta)$ satisfy the same recurrence. Moreover, $e_{\mathcal{c}_2,1}= 1 = N_1(\beta)$ and $e_{\mathcal{c}_3,1}= 1 + \beta = N_2(\beta)$.
Hence $e_{\mathcal{c}_{m + 1},1} = N_m(\beta)$ for all $m \geq 1$.

From \cref{eq:nBkwSo,eq:nBkwSo1}, we get that the generating function $g_{\beta}(x)$ needs to satisfy the quadratic equation $g_{\beta}(x) - (1 - \beta)^{-1} = (1 - \beta)^{-1} x g_{\beta}(x) + \beta (1 - \beta)^{-1} x g_{\beta}(x)^2$ and have $g_{\beta}(0) = (1 - \beta)^{-1}$, which gives~\eqref{eq:mDeSYH}.
\end{proof}

Informally, \cref{cor:principal-coef} means
\begin{align*}
  \sum_{m = 1}^{\infty} h^m \nomemcoef{m}{n}
  &\approx - h \nabla L - \beta h g_{\beta}(h \nabla^2 L) \nabla L + \np,
\end{align*}
where $\approx$ hides the fact that $n$ is taken to infinity. Therefore, the ``full-order'' memoryless iteration is approximately
\begin{equation}\label{eq:aIctKj}
\tilde{\btheta}^{(n + 1)} = \tilde{\btheta}^{(n)} - h \nabla L\prn[\big]{\tilde{\btheta}^{(n)}} - \beta h g_{\beta}\prn[\big]{h \nabla^2 L\prn[\big]{\tilde{\btheta}^{(n)}}} \nabla L\prn[\big]{\tilde{\btheta}^{(n)}} + \np
\end{equation}
for large $n$. The series $g_{\beta}\prn[\big]{h \nabla^2 L\prn[\big]{\tilde{\btheta}^{(n)}}}$ converges in Euclidean operator norm when $\norm[\big]{\nabla^2 L\prn[\big]{\tilde{\btheta}^{(n)}}} < R_{\beta} / h$, where $R_{\beta} = (1 - \sqrt{\beta})^2$ is the convergence radius.\sloppy

\begin{remark}
We thank Boris Hanin for the following interesting observation: $g_{\beta}(x)$ is the Stieltjes transform of the standard Marchenko-Pastur law with parameter $\beta$, which we can write as
\begin{equation*}
g_{\beta}(x) = \Expectlet_{\xi \sim \mathrm{MP}(\beta)} \brk{(\xi - x)^{-1}}.
\end{equation*}
Hence, \eqref{eq:aIctKj} becomes
\begin{equation*}
\tilde{\btheta}^{(n + 1)} = \tilde{\btheta}^{(n)} - h \Expectlet_{\xi \sim \mathrm{MP}(\beta)} \brk[\big]{\bI + \beta \prn[\big]{\xi \bI - h \nabla^2 L\prn[\big]{\tilde{\btheta}^{(n)}}}^{-1}} \nabla L\prn[\big]{\tilde{\btheta}^{(n)}} + \np
\end{equation*}
as long as $\norm{\nabla^2 L\prn[\big]{\tilde{\btheta}^{(n)}}} < (1 - \sqrt{\beta})^2 / h$ as above.
\end{remark}

\section{Analyzing the Combinatorics of Memoryless HB Approximations}\label{sec:combinatorics}

Recall from \cref{cor:principal-coef} that the coefficients $e_{\mathcal{c}_m, 1}$, where $\mathcal{c}_m$ is a chain with $m$ vertices, are the Narayana polynomials.

\begin{remark}\label{rem:narayana}
\Cref{cor:principal-coef} and \eqref{eq:vm-l-rec} show in particular that the Narayana polynomials can be defined as $N_m(\beta) \equiv N_{m, 1}(\beta)$ where $\crl{N_{m, l}(\beta)}$ satisfy the recursion
\begin{equation*}
  N_{m, l}(\beta)
    = (1 - \beta)^2 \sum_{b = 0}^{\infty} \beta^b
  \sum_{l_1 = 1}^{l + b} N_{m - 1, l_1}(\beta), \quad m
    \geq 2, l \geq 1
\end{equation*}
with initial condition $N_{1, l}(\beta) = \beta + (1 - \beta) l$ for $l \geq 1$.
With the help of ChatGPT 6, we found that this recursion is equivalent,
after normalization, to \citet[Eq. (2.7)]{Blecher11082016}.
\end{remark}

Let us now write the coefficient before another type of trees, namely, the trees consisting only of the root and a number of leaves.

\begin{corollary}\label{cor:eulerian}
Let $\mathcal{r}_m$ be the rooted tree consisting of a root
and $m - 1$ leaves. For $m, l \geq 1$,
\begin{equation}\label{eq:GGeGvs}
  e_{\mathcal{r}_m, l}(\beta)
    = (1 - \beta)^m \sum_{b = 0}^{\infty} \beta^b(l + b)^{m - 1} .
\end{equation}
In particular,
\begin{equation*}
  e_{\mathcal{r}_{m + 1}, 1}(\beta) = A_m(\beta), \qquad m \geq 1,
\end{equation*}
where
\begin{equation*}
  A_m(\beta) := (1 - \beta)^{m + 1} \sum_{j = 1}^{\infty} j^m \beta^{j - 1}
\end{equation*}
are the Eulerian polynomials.
\end{corollary}

\begin{proof}
\Cref{eq:GGeGvs} follows from \cref{def:e-coef} since $e_{\rootedtree[], l} = 1$.
\end{proof}

By \cref{cor:principal-coef,cor:eulerian}, for $m \geq 2$
\begin{equation}\label{eq:hiFKzE}
e_{\mathcal{c}_m, 1} = N_{m - 1}(\beta),\quad e_{\mathcal{r}_m, 1} = A_{m - 1}(\beta).
\end{equation}
We see that $e_{\tau, 1} \equiv e_{\tau, 1}(\beta)$ in \cref{def:e-coef}
form a rich $\unlabrooted[m]$-parametrized family of polynomials of $\beta$,
containing both the Narayana polynomials and the Eulerian polynomials.
For $m = 6$, the family is
\begin{align*}
  e_{\mathcal{c}_6, 1}(\beta)
  &= 1 + 10 \beta + 20 \beta^2 + 10 \beta^3 + \beta^4 = N_5(\beta),
  \\
  e_{\rootedtree[[[[]]][][]], 1}(\beta)
  &= 1 + 13 \beta + 28 \beta^2 + 13 \beta^3 + \beta^4,
  \\
  e_{\rootedtree[[[]][[]][]], 1}(\beta)
  &= 1 + 14 \beta + 30 \beta^2 + 14 \beta^3 + \beta^4,
  \\
  e_{\rootedtree[[][][[][]]], 1}(\beta)
  &= 1 + 17 \beta + 40 \beta^2 + 17 \beta^3 + \beta^4,
  \\
  e_{\rootedtree[[[]][][][]], 1}(\beta)
  &= 1 + 19 \beta + 44 \beta^2 + 19 \beta^3 + \beta^4,
  \\
  e_{\mathcal{r}_6, 1}(\beta)
  &= 1 + 26 \beta + 66 \beta^2 + 26 \beta^3 + \beta^4 = A_5(\beta).
\end{align*}
It is easy to see a clear pattern: they are all palindromic, monic (the leading coefficient is $1$), unimodal (the coefficients increase up to a point and then decrease), and lie coefficient-wise between the Narayana and Eulerian polynomials.
In fact, they are not merely unimodal but $\gamma$-positive (which implies unimodality), that is, for $m = \abs{\tau} \geq 2$, they can be written as
\begin{equation*}
  e_{\tau, 1}(\beta)
    = \sum_{j = 0}^{\lfloor (m - 2) / 2 \rfloor} \gamma_{\tau, j}\, \beta^j
  (1 + \beta)^{m - 2 - 2 j}
\end{equation*}
with nonnegative reals $\crl{\gamma_{\tau, j}}$. For example,
\begin{equation*}
  e_{\rootedtree[[[]][][][]], 1}(\beta)
    = (1 + \beta)^4 + 15 \beta(1 + \beta)^2 + 8 \beta^2.
\end{equation*}
We will now prove that these properties hold in general.
Perhaps unsurprisingly, the argument is not direct. It uses a short excursion into geometric combinatorics: we will identify $e_{\tau,1}$ with the $h$-polynomial of a graph-associahedron.

Recall that a \textit{simplicial complex} on a finite vertex set $V$
is a collection $\Delta$ of subsets of $V$, called \textit{faces}, such that every
subset of a face is again a face. Its dimension is
$\dim \Delta := \max_{F \in \Delta}(|F| - 1)$.
If $f_{i - 1}$ denotes the number of faces of $\Delta$ with $i$ vertices (with $f_{-1} = 1$ corresponding to the empty face), its $h$-polynomial is
defined by
\begin{equation*}
  h(\Delta,\beta)
  :=
  \sum_{i=0}^{d}
  f_{i-1}\beta^i(1-\beta)^{d-i}
  =
  \sum_{i=0}^{d}h_i\beta^i, \quad d = \dim \Delta + 1.
\end{equation*}
The coefficient vector $(h_0,\ldots,h_d)$ is called the $h$-vector of
$\Delta$.
The Stanley-Reisner Hilbert series has the following two forms:
\begin{equation}\label{eq:qsUkQm}
\frac{h(\Delta, \beta)}{(1 - \beta)^d}
  :=
\sum_{F\in\Delta} \prn[\bigg]{\frac{\beta}{1- \beta}}^{|F|}
\equiv
\sum_{\mathcal M} \beta^{|\mathcal M|},
\end{equation}
where $\mathcal M$ ranges over finite multisets of vertices of
$\Delta$ whose support is a face of $\Delta$, and $|\mathcal{M}|$ denotes total multiplicity
(see, for example, \cite[Theorems~6.9 and 6.15]{FranciscoMerminSchweigSR}).

For any graph $G$, we will use $V(G)$ and $E(G)$ to denote its vertex and edge sets.
If $\tau \in \labrooted\range{1}{m}$ is a (labeled) rooted tree, let $L(\tau)$ be the line graph, that is, the graph whose vertices are the edges of $\tau$, with two vertices adjacent if and only if the corresponding edges share an endpoint.
Let
\begin{equation*}
  \mathcal{B}(L(\tau)) := \crl{\varnothing \neq I \subseteq V(L(\tau)): \text{$L(\tau)[I]$ is connected}}
\end{equation*}
be the \textit{graphical building set} of $L(\tau)$, which is the family of \textit{tubes}:
nonempty subsets of $V(L(\tau)) \equiv E(\tau)$ whose induced subgraph of $L(\tau)$ is connected.
Denoting by $\crl{\be_a: a \in E(\tau)}$ the standard basis of
$\mathbb R^{E(\tau)}$,
the \textit{graph-associahedron} of $L(\tau)$ is the nestohedron associated to
$\mathcal B(L(\tau))$, i.\,e., the Minkowski sum
\begin{equation*}
  P_{\mathcal{B}(L(\tau))}
  :=
  \sum_{I\in\mathcal{B}(L(\tau))} \operatorname{conv}\{\be_a: a\in I\} \subseteq \mathbb R^{E(\tau)}.
\end{equation*}
Its corresponding (dual) simplicial complex $\Delta_{\tau} \equiv \Delta_{\mathcal{B}(L(\tau))}$ is the \textit{nested set complex}, or \textit{tubing complex},
of $L(\tau)$.
Thus, the vertices of $\Delta_{\tau}$ are the proper tubes $\mathcal{B}(L(\tau)) \setminus \{ V(L(\tau)) \}$, and the faces can be defined \citep[Section 7.1]{PRW} as sets of pairwise compatible tubes. Two tubes $I, J \in \mathcal{B}(L(\tau))$ are \textit{compatible} if they are either nested or disjoint and non-adjacent in $L(\tau)$, equivalently, if either $I \subseteq J$, $J \subseteq I$, or else $I \cap J = \varnothing$ and $I \cup J \notin \mathcal{B}(L(\tau))$.

Assume $\abs{\tau} > 1$ to avoid the trivial case of a single vertex.
The cone over $\Delta_\tau$ is the simplicial complex
\begin{equation*}
  \widehat{\Delta}_\tau
  := \operatorname{cone}(\Delta_\tau) := \Delta_\tau
  \cup \crl{F \cup \{ E(\tau) \}:F \in \Delta_\tau} .
\end{equation*}
In words, we add one new vertex which is labelled by the full tube
$V(L(\tau)) \equiv E(\tau)$, and allow this vertex to be added to
every existing face.
The simplicial complex $\Delta_\tau$ has dimension $\abs{E(\tau)} - 2$, and thus $\widehat{\Delta}_\tau$ has dimension $\abs{E(\tau)} - 1$.
Applying \eqref{eq:qsUkQm} to $\widehat{\Delta}_\tau$ gives
\begin{equation*}
\frac{h(\widehat{\Delta}_\tau, \beta)}{(1 - \beta)^{\abs{E(\tau)}}} = \sum_{\mathcal M} \beta^{|\mathcal M|},
\end{equation*}
where $\mathcal M$ ranges over finite multisets of pairwise compatible tubes of $L(\tau)$, with the full tube $E(\tau)$ allowed. Since $\Delta_\tau$ is dual to the graph-associahedron $P_{\mathcal{B}(L(\tau))}$ and coning does not change the $h$-polynomial, we have further
\begin{equation}\label{eq:xuLaEb}
\frac{h(P_{\mathcal{B}(L(\tau))}, \beta)}{(1 - \beta)^{\abs{E(\tau)}}} := \frac{h(\Delta_\tau, \beta)}{(1 - \beta)^{\abs{E(\tau)}}} = \frac{h(\widehat{\Delta}_\tau, \beta)}{(1 - \beta)^{\abs{E(\tau)}}} = \sum_{\mathcal M} \beta^{|\mathcal M|}.
\end{equation}

For $\tau = \bullet$, we use the convention $\Delta_{\tau} = \widehat{\Delta}_{\tau} = \crl{\varnothing}$. Then $\dim \Delta_\tau = \dim \widehat{\Delta}_\tau = -1$,
$h(\widehat{\Delta}_\tau,\beta) = 1$, and the multiset sum is $1$.

The main result of this section will be
\begin{theorem}\label{thm:e-h-poly}
  For every rooted tree $\tau$,
\begin{equation*}
e_{\tau, 1}(\beta) = h(P_{\mathcal{B}(L(\tau))}, \beta).
\end{equation*}
In words, $e_{\tau, 1}(\beta)$ is the $h$-polynomial of the graph-associahedron of the line graph of $\tau$.
In particular, $e_{\tau, 1}(\beta)$ does not depend on the choice of root.
\end{theorem}

The proof strategy is identifying a bijection between functions $x\colon E(\tau) \to \mathbb{Z}_{> 0}$ (attaching natural numbers to edges of $\tau$) and multisets $\mathcal{M}$ of pairwise compatible tubes of $L(\tau)$, and suitably constructing a functional $J_1(x)$, such that, on the one hand
\begin{equation}\label{eq:ewlkKG}
\frac{e_{\tau, 1}(\beta)}{(1 - \beta)^{\abs{E(\tau)}}} = \sum_{x\colon E(\tau) \to \mathbb{Z}_{> 0}} \beta^{J_1(x)},
\end{equation}
and on the other hand
\begin{equation*}
J_1(x) = |\mathcal{M}|.
\end{equation*}
Since it is a bijection,
\begin{equation*}
\sum_{x\colon E(\tau) \to \mathbb{Z}_{> 0}} \beta^{J_1(x)} = \sum_{\mathcal M} \beta^{|\mathcal M|}.
\end{equation*}
Then \eqref{eq:ewlkKG} gives
\begin{equation*}
\frac{e_{\tau, 1}(\beta)}{(1 - \beta)^{\abs{E(\tau)}}} = \sum_{\mathcal M} \beta^{|\mathcal M|},
\end{equation*}
which combined with \eqref{eq:xuLaEb} proves the theorem.

We implement this now, starting by giving the definition of $J_1(x)$ (in fact $J_l(x)$ for natural $l$) and constructing the bijection.

\begin{lemma}\label{lem:YKOemB}
Let $\tau$ be a labeled rooted tree. For each function $x\colon E(\tau) \to \mathbb{Z}_{> 0}$ and integer $l \geq 1$, define
\begin{align*}
  &h_v(x) := \begin{cases}
    x_{e_{\text{in}(v)}}, & \text{if } v \text{ is not the root},\\
    l, & \text{if } v \text{ is the root},
  \end{cases}\\
  &J_l(x)
    :=
    \sum_{\substack{v\in V(\tau)\\ \operatorname{out}(v)\neq\varnothing}}
  \prn[\bigg]{
  \max_{e\in \operatorname{out}(v)} x_e - h_v(x)
  }_+,
\end{align*}
where $e_{\text{in}(v)}$ is the edge incoming to vertex $v$ from its parent in $\tau$,
$(a)_+=\max(a, 0)$ and $\operatorname{out}(v)$ is the set of edges outgoing from vertex $v$ to its children in $\tau$. Then there is a bijection between functions $x\colon E(\tau) \to \mathbb{Z}_{> 0}$ and multisets $\mathcal{M}$ of pairwise compatible tubes of $L(\tau)$ with the full tube $E(\tau)$ allowed, such that
\begin{equation}\label{eq:IIvdXE}
J_1(x) = |\mathcal{M}|.
\end{equation}
\end{lemma}

\begin{proof}
Given $x\colon E(\tau) \to \mathbb{Z}_{> 0}$, define for each $k \geq 0$ the superlevel edge set
\begin{equation*}
E_{\geq k}(x) := \crl{e \in E(\tau): x_e \geq k}.
\end{equation*}
Let $\Phi(x)$ be the multiset of all connected components of $E_{\geq k}(x)$, as $k$ ranges over $\mathbb{Z}_{\geq 2}$.
Each such component is a connected set of edges of $\tau$, i.\,e., a tube of $L(\tau)$.
These tubes are pairwise compatible: indeed, components at the same level are disjoint and non-adjacent,
and components at different levels are either nested or disjoint and non-adjacent.

The map $\Phi$ is injective by
\begin{equation}\label{eq:aieNGF}
x_e = 1 + \#\crl[\big]{I \in \Phi(x): e \in I},
\end{equation}
where occurrences are counted with multiplicity. Indeed, $e$ belongs to exactly one component of $E_{\geq k}(x)$ for each $k = 2, \ldots, x_e$ and to no component for $k > x_e$; thus the right-hand side of \eqref{eq:aieNGF} is $1 + (x_e - 1) = x_e$.

We now show surjectivity.
Let $\mathcal{M}$ be a finite multiset of pairwise compatible tubes of $L(\tau)$, with the full tube $E(\tau)$ allowed.
Define
\begin{equation*}
x_{\mathcal{M}}(e) = 1 + \#\crl{I \in \mathcal{M}: e \in I}.
\end{equation*}
We argue that $\Phi(x_{\mathcal{M}}) = \mathcal{M}$ by induction on $|\mathcal{M}|$.
The base case $\mathcal{M} = \varnothing$ is trivial. Assume $\mathcal{M} \neq \varnothing$.
Let $C_1, \ldots, C_a$ be the distinct inclusion-maximal tubes appearing in $\mathcal{M}$.
Since they are pairwise compatible and inclusion-maximal, they are disjoint and non-adjacent, and thus they are exactly the connected components of $E_{\geq 2}(x_{\mathcal{M}})$. Let $\mathcal{M}'$ be the multiset obtained from $\mathcal{M}$ by removing one copy of each $C_j$. Then
\begin{equation*}
  x_{\mathcal{M}'}(e) =
  \begin{cases}
    x_{\mathcal{M}}(e) - 1, & e \in C_1 \cup \cdots \cup C_a,\\
    1, & e \notin C_1 \cup \cdots \cup C_a.
  \end{cases}
\end{equation*}
Consequently, for every $k \geq 3$,
\begin{equation*}
E_{\geq k}(x_{\mathcal{M}}) = E_{\geq k - 1}(x_{\mathcal{M}'}).
\end{equation*}
Hence the components appearing in levels $k \geq 3$ of $\Phi(x_{\mathcal{M}})$ are exactly the components appearing in $\Phi(x_{\mathcal{M'}})$. By induction, $\Phi(x_{\mathcal{M}'}) = \mathcal{M}'$. Together with the level $2$ components $C_1, \ldots, C_a$, this gives $\Phi(x_{\mathcal{M}}) = \mathcal{M}$, completing the proof of surjectivity.

Finally, we prove \eqref{eq:IIvdXE}. Every connected component of $E_{\geq k}(x)$ has a unique top vertex (closest to the root). For any $v \in V(\tau)$, the vertex $v$ is the top vertex of a connected component in $E_{\geq k}(x)$ exactly when at least one outgoing edge $e$ of $v$ has $x_e \geq k$ and either $v$ is the root or its incoming edge $e_{\text{in}}(v)$ has $x_{e_{\text{in}}(v)} < k$.
Since $h_v(x) = 1 < k$ at the root and
$h_v(x) = x_{e_{\text{in}}(v)}$ elsewhere, these conditions are
equivalent to
$\max_{e \in \text{out}(v)} x_e \geq k > h_v(x)$.
Thus the number of components of $E_{\geq k}(x)$ is
\begin{equation*}
  \sum_{\substack{v\in V(\tau)\\ \operatorname{out}(v)\neq\varnothing}}
  \mathbf{1}\crl[\Big]{\max_{e \in \operatorname{out}(v)} x_e \geq k > h_v(x)}.
\end{equation*}
Summing over $k \geq 2$ gives
\begin{align*}
  \abs{\Phi(x)}
  =
  \sum_{k \geq 2}
  \sum_{\substack{v\in V(\tau)\\ \operatorname{out}(v)\neq\varnothing}}
  \mathbf{1}\crl[\Big]{\max_{e \in \operatorname{out}(v)} x_e \geq k > h_v(x)}
  = \sum_{\substack{v\in V(\tau)\\ \operatorname{out}(v)\neq\varnothing}} \prn[\bigg]{\max_{e \in \operatorname{out}(v)} x_e - h_v(x)}_+ = J_1(x)
\end{align*}
as desired.
\end{proof}

We now prove a slight generalization of \eqref{eq:ewlkKG}.

\begin{lemma}\label{lem:zNQMmN}
Recall $J_l(x)$ defined in \cref{lem:YKOemB}. For every rooted tree $\tau$, we have
\begin{equation*}
\frac{e_{\tau, l}(\beta)}{(1 - \beta)^{\abs{E(\tau)}}} = \sum_{x\colon E(\tau) \to \mathbb{Z}_{> 0}} \beta^{J_l(x)}.
\end{equation*}
\end{lemma}

\begin{proof}
We argue by induction on the number of vertices of $\tau$.
For $\tau = \bullet$, both sides are equal to one. Suppose $\tau = [\tau_1, \ldots, \tau_{\ell}]$. Since
\begin{equation*}
\abs{E(\tau)} = \ell + \sum_{i = 1}^{\ell} \abs{E(\tau_i)},
\end{equation*}
the definition of $e_{\tau, l}(\beta)$ gives
\begin{equation*}
\frac{e_{\tau, l}(\beta)}{(1 - \beta)^{\abs{E(\tau)}}} = (1 - \beta) \sum_{b \geq 0} \beta^b \prod_{i = 1}^{\ell} \sum_{s = 1}^{l + b} \frac{e_{\tau_i, s}(\beta)}{(1 - \beta)^{\abs{E(\tau_i)}}}.
\end{equation*}
Interchanging the $b$-sum with the $s_i$-sums (permitted by Tonelli's theorem) yields
\begin{align*}
  \frac{e_{\tau, l}(\beta)}{(1 - \beta)^{\abs{E(\tau)}}} &= \sum_{s_1, \ldots, s_{\ell} \geq 1} \prn[\bigg]{(1 - \beta) \sum_{b \geq \max(0, \max_i s_i - l)} \beta^b} \prod_{i = 1}^{\ell} \frac{e_{\tau_i, s_i}(\beta)}{(1 - \beta)^{\abs{E(\tau_i)}}}\\
  &= \sum_{s_1, \ldots, s_{\ell} \geq 1} \beta^{(\max_i s_i - l)_+} \prod_{i = 1}^{\ell} \frac{e_{\tau_i, s_i}(\beta)}{(1 - \beta)^{\abs{E(\tau_i)}}}.
\end{align*}
By induction, this means
\begin{equation*}
  \frac{e_{\tau, l}(\beta)}{(1 - \beta)^{\abs{E(\tau)}}} = \sum_{s_1, \ldots, s_{\ell} \geq 1} \beta^{(\max_i s_i - l)_+} \prod_{i = 1}^{\ell} \prn[\bigg]{\sum_{x\colon E(\tau_i) \to \mathbb{Z}_{> 0}} \beta^{J_{s_i}(x)}}.
\end{equation*}
The right-hand side simplifies to
\begin{equation*}
  \sum_{s_1, \ldots, s_{\ell} \geq 1} \sum_{x_1: E(\tau_1) \to \mathbb{Z}_{> 0}} \cdots \sum_{x_{\ell}: E(\tau_{\ell}) \to \mathbb{Z}_{> 0}} \beta^{(\max_i s_i - l)_+ + \sum_{i = 1}^{\ell} J_{s_i}(x_i)}
  = \sum_{x\colon E(\tau) \to \mathbb{Z}_{> 0}} \beta^{J_l(x)},
\end{equation*}
concluding the induction step.
\end{proof}

The proof of the theorem is now immediate, given below.

\begin{proof}[Proof of \cref{thm:e-h-poly}]
If $\tau=\bullet$, both sides are $1$.  Hence assume $|E(\tau)|\geq 1$.
By \cref{lem:zNQMmN}, we have
\begin{equation*}
\frac{e_{\tau, 1}(\beta)}{(1 - \beta)^{\abs{E(\tau)}}} = \sum_{x\colon E(\tau) \to \mathbb{Z}_{> 0}} \beta^{J_1(x)}.
\end{equation*}
By \cref{lem:YKOemB}, the right-hand side is equal to
\begin{equation*}
\sum_{\mathcal{M}} \beta^{|\mathcal{M}|},
\end{equation*}
where $\mathcal{M}$ ranges over finite multisets of pairwise compatible tubes of $L(\tau)$, with the full tube $E(\tau)$ allowed. By \eqref{eq:xuLaEb}, this is equal to
\begin{equation*}
  \frac{h(P_{\mathcal{B}(L(\tau))}, \beta)}
  {(1 - \beta)^{\abs{E(\tau)}}}.
\end{equation*}
Combining, we have
\begin{equation*}
\frac{e_{\tau, 1}(\beta)}{(1 - \beta)^{\abs{E(\tau)}}} = \frac{h(P_{\mathcal{B}(L(\tau))}, \beta)}{(1 - \beta)^{\abs{E(\tau)}}}
\end{equation*}
as desired.
\end{proof}

Having made the excursion to geometric combinatorics, we can now find the properties of the polynomials $e_{\tau, 1}(\beta)$ in the rich literature on graph-associahedra.

\begin{corollary}\label{cor:e-properties}
  The polynomials $e_{\tau, 1}(\beta)$ are palindromic, monic, $\gamma$-positive, and satisfy the coefficient-wise inequalities
\begin{equation}\label{eq:fFYsfU}
N_{\abs{\tau} - 1}(\beta) \leq_{\text{coeff}} e_{\tau, 1}(\beta) \leq_{\text{coeff}} A_{\abs{\tau} - 1}(\beta).
\end{equation}
\end{corollary}

\begin{proof}
The case $\tau=\bullet$ is trivial, with $N_0(\beta) = A_0(\beta) = 1$.
Assume from now on that $|E(\tau)|\geq 1$.
Since $L(\tau)$ is connected, $P_{\mathcal B(L(\tau))}$ is a simple polytope of dimension $|E(\tau)| - 1$, hence its $h$-polynomial is palindromic and monic \cite[Section 2.2]{PRW}. Since $L(\tau)$ is chordal, $P_{\mathcal B(L(\tau))}$ is a chordal graph-associahedron, hence its $h$-polynomial is $\gamma$-positive \cite[Theorem 11.6]{PRW}.

Finally, by \cite[Theorem~9.2]{BuchstaberVolodin2011}, for every connected graph $G$ on $m$ vertices one has
\begin{equation*}
  h(P_{\mathcal{B}(P_m)}, \beta)
  \leq_{\mathrm{coeff}}
  h(P_{\mathcal{B}(G)}, \beta)
  \leq_{\mathrm{coeff}}
  h(P_{\mathcal{B}(K_m)}, \beta),
\end{equation*}
where $P_m$ is the path on $m$ vertices and $K_m$ is the complete
graph. The graph-associahedron of $P_m$ is the associahedron, whose
$h$-polynomial is $N_m(\beta)$, and the graph-associahedron of
$K_m$ is the permutohedron, whose $h$-polynomial is $A_m(\beta)$.
Substituting $G = L(\tau)$ gives \eqref{eq:fFYsfU}.
\end{proof}

\section{Matching the Fixed-Point and Memory Perspectives}\label{sec:matching}

The result in \cref{th:nomem-coef} can be connected to the invariant manifold perspective discussed in \cref{sec:intro-hb-is-gd-on-manifold,sec:full-batch-hb-is-gd}. Recall that HB on that manifold can be rewritten as \eqref{eq:qcUUpn} where $\bg_h$ satisfies the fixed point equation \eqref{eq:uruw}. Then, write $\bg_h(\btheta)$ as a formal B-series \citep{butcher1972algebraic,hairer1974butcher} to obtain
\begin{equation}\label{eq:xRbpDJ}
  h^2 \bg_h(\btheta) = \sum_{\tau \in \labrooted_{\varnothing}} \frac{h^{\abs{\tau}}}{\abs{\tau}!} g(\tau) \nabla^{\tau} L(\btheta),
\end{equation}
where $g\colon \unlabrooted_{\varnothing} \to \mathbb{R}$ is the coefficient mapping (with the induced mapping $g\colon \labrooted_{\varnothing} \to \mathbb{R}$ denoted by the same symbol), $g(\varnothing) = g(\rootedtree[]) = 0$.
We stress that this is not really an equality but a formal expansion.
The next result is the characterization of the coefficients $g(\tau)$ in \eqref{eq:xRbpDJ}.

\subsection{Matching Power Series}
\begin{proposition}
Let $a\colon \unlabrooted_{\varnothing} \to \mathbb{R}$ be defined by
\begin{gather}
  a(\varnothing) = 1, \quad a(\rootedtree[]) = - (1 - \beta)^{-1}, \label{eq:oOENIx}
  \\
  a(\tau) = \beta g(\tau) \quad \text{for $\abs{\tau} \geq 2$}. \notag
\end{gather}
Then it is the unique solution with initial conditions \eqref{eq:oOENIx} to the following recursion (for $\tau = [\tau_1, \ldots, \tau_{\ell}]$ with $\ell \geq 1$):
\begin{equation}\label{eq:rtYrem}
a(\tau) = \frac{\beta}{(1 - \beta)^2} a(\tau_1) \ldots a(\tau_{\ell}) - \frac{1}{1 - \beta} \sum_{\substack{\mrk \in \mrkset_{\tau} \setminus \crl{\varnothing}\\\abs{\tau_0^{\mrk}} \geq 2}} a(\tau_0^{\mrk}) a(\tau_1^{\mrk}) \ldots a(\tau_{\abs{\mrk}}^{\mrk}).
\end{equation}
\end{proposition}

\begin{proof}
By the composition rule (e.\,g. \citet{faris2021rootedtreegraphs}), we have
\begin{equation}\label{eq:xadAIY}
  h^2 \bg_h\prn[\bigg]{\btheta - \frac{h}{1 - \beta} \nabla L(\btheta) + h^2 \beta \bg_h(\btheta)} = \sum_{\tau \in \labrooted_{\varnothing}} \frac{h^{\abs{\tau}}}{\abs{\tau}!} (a * g)(\tau) \nabla^{\tau} L(\btheta),
\end{equation}
where $a * g$ is the subtree convolution, that is, $(a * g)(\varnothing) = g(\varnothing) = 0$ and
\begin{equation*}
  (a * g)(\tau) = g(\varnothing) a(\tau) + \sum_{\mrk \in \mrkset_{\tau}} g(\tau_0^{\mrk}) a(\tau_1^{\mrk}) \ldots a(\tau_{\abs{\mrk}}^{\mrk}) = \sum_{\mrk \in \mrkset_{\tau}} g(\tau_0^{\mrk}) a(\tau_1^{\mrk}) \ldots a(\tau_{\abs{\mrk}}^{\mrk}).
\end{equation*}
Similarly, by the composition rule,
\begin{equation}\label{eq:LcoekJ}
  - \frac{h}{1 - \beta} \nabla L\prn[\bigg]{\btheta - \frac{h}{1 - \beta} \nabla L(\btheta) + h^2 \beta \bg_h(\btheta)}
  = \sum_{\tau \in \labrooted_{\varnothing}} \frac{h^{\abs{\tau}}}{\abs{\tau}!} (a * l)(\tau) \nabla^{\tau} L(\btheta),
\end{equation}
where $l\colon \unlabrooted_{\varnothing} \to \mathbb{R}$ is defined by $l(\varnothing) = 0$, $l(\rootedtree[]) = - (1 - \beta)^{-1}$, $l(\tau) = 0$ for $\abs{\tau} \geq 2$, which means
\begin{equation*}
(a * l)(\tau) = \sum_{\mrk \in \mrkset_{\tau}} l(\tau_0^{\mrk}) a(\tau_1^{\mrk}) \ldots a(\tau_{\abs{\mrk}}^{\mrk}) = - \frac{1}{1 - \beta} a(\tau_1) \ldots a(\tau_{\ell})
\end{equation*}
for $\tau = [\tau_1, \ldots, \tau_{\ell}]$; in particular, $(a * l)(\rootedtree[]) = - (1 - \beta)^{-1}$. Combining \eqref{eq:xadAIY} and \eqref{eq:LcoekJ} gives
\begin{align*}
  &- \frac{h}{1 - \beta} \nabla L\prn[\bigg]{\btheta - \frac{h}{1 - \beta} \nabla L(\btheta) + h^2 \beta \bg_h(\btheta)} + h^2 \bg_h\prn[\bigg]{\btheta - \frac{h}{1 - \beta} \nabla L(\btheta) + h^2 \beta \bg_h(\btheta)}\\
  &\quad = \sum_{\tau \in \labrooted_{\varnothing}} \frac{h^{\abs{\tau}}}{\abs{\tau}!} \crl[\big]{(a * l)(\tau) + (a * g)(\tau)} \nabla^{\tau} L(\btheta).
\end{align*}
By \eqref{eq:uruw}, this equals
\begin{equation*}
- \frac{h}{1 - \beta} \nabla L(\btheta) + h^2 \beta \bg_h(\btheta) = - \frac{h}{1 - \beta} \nabla L(\btheta) + \sum_{\tau \in \labrooted_{\varnothing}} \frac{h^{\abs{\tau}}}{\abs{\tau}!} \beta g(\tau) \nabla^{\tau} L(\btheta).
\end{equation*}
Matching the coefficients of each rooted tree in the formal B-series identity gives for $\tau = [\tau_1, \ldots, \tau_{\ell}]$ with $\abs{\tau} \geq 2$
\begin{equation*}
(a * l)(\tau) + (a * g)(\tau) = \beta g(\tau),
\end{equation*}
that is,
\begin{equation*}
- \frac{1}{1 - \beta} a(\tau_1) \ldots a(\tau_{\ell}) + \sum_{\mrk \in \mrkset_{\tau} \setminus \crl{\varnothing}} g(\tau_0^{\mrk}) a(\tau_1^{\mrk}) \ldots a(\tau_{\abs{\mrk}}^{\mrk}) + g(\tau) = \beta g(\tau).
\end{equation*}
Hence, the coefficients $g(\tau)$ satisfy the recursion
\begin{equation*}
g(\tau) = \frac{1}{(1 - \beta)^2} a(\tau_1) \ldots a(\tau_{\ell}) - \frac{1}{1 - \beta} \sum_{\mrk \in \mrkset_{\tau} \setminus \crl{\varnothing}} g(\tau_0^{\mrk}) a(\tau_1^{\mrk}) \ldots a(\tau_{\abs{\mrk}}^{\mrk})
\end{equation*}
with $g(\varnothing) = g(\rootedtree[]) = 0$; the same is rewritten in terms of only the $a$ mapping as \eqref{eq:rtYrem}.
\end{proof}

By definition, the right-hand side of
\begin{equation*}
\btheta^{(n + 1)} - \btheta^{(n)} = \prn[\bigg]{- h \frac{1}{1 - \beta} \nabla L(\vtheta) + h^2 \beta \boldsymbol{g}_h(\vtheta)}_{\vtheta = \btheta^{(n)}}
\end{equation*}
has a formal power series expansion (omitting $\vtheta$)
\begin{equation}\label{eq:formal-power-series-invariant-manifold}
  -\frac{h}{1 - \beta} \nabla L
    + \beta \sum_{j = 2}^{\infty} \frac{h^j}{j!}
    \sum_{\tau \in \labrooted \range{1}{j}} g(\tau) \nabla^\tau L
  = \sum_{j = 1}^{\infty} \frac{h^j}{j!}
    \sum_{\tau \in \labrooted \range{1}{j}} a(\tau) \nabla^\tau L .
\end{equation}

For illustration,
\begin{equation*}
a(\rootedtree[[]]) = - \frac{\beta}{(1 - \beta)^3},\quad a(\rootedtree[[[]]]) = a(\rootedtree[[][]]) = - \frac{\beta (1 + \beta)}{(1 - \beta)^5},
\end{equation*}
giving the update
\begin{align*}
  \btheta^{(n + 1)} - \btheta^{(n)} &= - \frac{h}{1 - \beta} \crl[\bigg]{
    \nabla L(\btheta^{(n)}) + \frac{h \beta}{(1 - \beta)^2} \nabla^2 L(\btheta^{(n)}) \nabla L(\btheta^{(n)})\\
  &\quad + \frac{h^2 \beta (1 + \beta)}{2 (1 - \beta)^4} \nabla^3 L(\btheta^{(n)}) \brk[\big]{\nabla L(\btheta^{(n)}), \nabla L(\btheta^{(n)})}\\
  &\quad + \frac{h^2 \beta (1 + \beta)}{(1 - \beta)^4} \nabla^2 L(\btheta^{(n)}) \nabla^2 L(\btheta^{(n)}) \nabla L(\btheta^{(n)}) + \ldots},
\end{align*}
with the dots denoting higher-order terms in $h$.
This is exactly the same update as we saw
in \eqref{eq:SZADJf}, the memoryless approximation of the corresponding order in the large-$n$ limit.
We now show that, expectedly, this is true at all orders: in the limit $n \to \infty$, the memoryless approximation coefficients are characterized by the same power series as in \eqref{eq:formal-power-series-invariant-manifold}.

Taking the limit $n \to \infty$ and formally $\ord \to \infty$ of the right-hand side in
\begin{align*}
  \tilde{\btheta}^{(n + 1)}
  - \tilde{\btheta}^{(n)}
  &=
  \sum_{j = 1}^{\ord}
  h^j \nomemcoef{j}{n}(\tilde{\btheta}^{(n)})
  \\
  &=
  \prn[\bigg]{- h \nabla L^{(n)}(\vtheta)
  - \beta \sum_{j = 1}^{\ord} \frac{h^j}{j!}
  \sum_{\tau \in \labrooted\range{1}{j}}
  \E{\tau}{1}{n}(\vtheta)}_{\vtheta = \tilde{\btheta}^{(n)}}
\end{align*}
in the full-batch case, we obtain the following formal expression:
\begin{equation}\label{eq:formal-power-series-memory}
  -h \nabla L
  - \beta \sum_{j = 1}^{\infty} \frac{h^j}{j!}
  \sum_{\tau \in \labrooted \range{1}{j}} q(\tau) \nabla^\tau L
  = -\frac{h}{1 - \beta} \nabla L
    - \beta \sum_{j = 2}^{\infty} \frac{h^j}{j!}
    \sum_{\tau \in \labrooted \range{1}{j}} q (\tau) \nabla^\tau L,
\end{equation}
where we set
\begin{equation}\label{eq:q-of-tau-def}
  q(\tau) =
  \begin{cases}
    1 &\text{if } \tau = \varnothing,
    \\
    (1 - \beta)^{1 - 2 \abs{\tau}} e_{\tau, 1}(\beta)
    &\text{if } \tau
      \neq \varnothing
  \end{cases}
\end{equation}
with $e_{\tau, 1}(\beta)$ defined in \cref{def:e-coef}.
We have the following proposition.

\begin{proposition}\label{prop:matching-power-series}
Let $q$ be defined as in \eqref{eq:q-of-tau-def} and let
\begin{equation*}
  \hat{a}(\tau) =
  \begin{cases}
    1 &\text{if } \tau = \varnothing,
    \\
    - (1 - \beta)^{-1} &\text{if } \tau = \rootedtree[]
    \\
    - \beta q(\tau)
    &\text{if } \abs{\tau} \geq 2.
  \end{cases}
\end{equation*}
Then $a = \hat{a}$. In particular, the power series in \eqref{eq:formal-power-series-memory} and \eqref{eq:formal-power-series-invariant-manifold} coincide.
\end{proposition}

\begin{proof}
Taking the limit $n \to \infty$ in the full-batch special case of \eqref{eq:lQbzcy} yields
\begin{equation}\label{eq:YDONVu}
q(\tau) = q(\tau_1) \ldots q(\tau_{\ell}) + \beta \sum_{\mrk \in \mrkset_{\tau}} q(\tau_0^{\mrk}) q(\tau_1^{\mrk}) \ldots q(\tau_{\abs{\mrk}}^{\mrk}).
\end{equation}
This means that for every non-empty tree $\tau$ we have
\begin{align}
  (q * \hat{a})(\tau) &= q(\tau) + \sum_{\mrk \in \mrkset_{\tau}} \hat{a}(\tau_0^{\mrk}) q(\tau_1^{\mrk}) \ldots q(\tau_{\abs{\mrk}}^{\mrk}) \notag
  \\
                 &= q(\tau) - q(\tau_1) \ldots q(\tau_{\ell}) - \beta \sum_{\mrk \in \mrkset_{\tau}} q(\tau_0^{\mrk}) q(\tau_1^{\mrk}) \ldots q(\tau_{\abs{\mrk}}^{\mrk}) \label{eq:GLkVIj}
  \\
  &= 0, \label{eq:TZhppP}
\end{align}
where \eqref{eq:GLkVIj} is by the definition of $\hat{a}$ and \eqref{eq:TZhppP} is by \eqref{eq:YDONVu}. We have proven
\begin{equation*}
q * \hat{a} = \varepsilon, \quad \text{where} \quad \varepsilon(\tau) := \begin{cases}
  1, &\text{if } \tau = \varnothing,\\
  0, &\text{otherwise.}
\end{cases}
\end{equation*}

By associativity,
\begin{equation*}
q * (\hat{a} * q) = (q * \hat{a}) * q = \varepsilon * q = q.
\end{equation*}
In other words, for any non-empty tree $\tau$
\begin{equation*}
q(\tau) + (\hat{a} * q)(\tau) + \sum_{\mrk \in \mrkset_{\tau} \setminus \crl{\varnothing}} (\hat{a} * q)(\tau_0^{\mrk})q(\tau_1^{\mrk}) \ldots q(\tau_{\abs{\mrk}}^{\mrk}) = q(\tau),
\end{equation*}
or, equivalently,
\begin{equation*}
(\hat{a} * q)(\tau) = - \sum_{\mrk \in \mrkset_{\tau} \setminus \crl{\varnothing}} (\hat{a} * q)(\tau_0^{\mrk}) q(\tau_1^{\mrk}) \ldots q(\tau_{\abs{\mrk}}^{\mrk}).
\end{equation*}
Additionally, $(\hat{a} * q)(\varnothing) = 1$. A simple induction shows that only $\varepsilon$ satisfies this recursion and initial condition, thus
\begin{equation}\label{eq:njdLyj}
\hat{a} * q = \varepsilon.
\end{equation}

For any $\tau$ with the number of vertices at least $2$, we can rewrite $(\hat{a} * q)(\tau) = 0$
as
\begin{equation*}
  0
  = \hat{a}(\tau)
  + q(\tau)
  + \frac{1}{1 - \beta} \hat{a}(\tau_1) \ldots \hat{a}(\tau_{\ell})
  + \sum_{\substack{\mrk \in \mrkset_{\tau} \setminus \crl{\varnothing} \\ \abs{\tau_0^{\mrk}} \geq 2}} q(\tau_0^{\mrk}) \hat{a}(\tau_1^{\mrk}) \ldots \hat{a}(\tau_{\abs{\mrk}}^{\mrk}).
\end{equation*}
Inserting $q(\tau) = - \hat{a}(\tau) / \beta$ yields
\begin{equation*}
  \hat{a}(\tau)
  = \frac{\beta}{(1 - \beta)^2} \hat{a}(\tau_1) \hat{a}(\tau_2) \ldots \hat{a}(\tau_{\ell})
  - \frac{1}{1 - \beta} \sum_{\substack{\mrk \in \mrkset_{\tau} \setminus \crl{\varnothing} \\ \abs{\tau_0^{\mrk}} \geq 2}} \hat{a}(\tau_0^{\mrk}) \hat{a}(\tau_1^{\mrk}) \ldots \hat{a}(\tau_{\abs{\mrk}}^{\mrk}).
\end{equation*}
This is the same recursion as \eqref{eq:rtYrem}. Since the initial conditions
$\hat{a}(\varnothing) = a(\varnothing) = 1$ and $\hat{a}(\rootedtree[]) = a(\rootedtree[]) = - 1 / (1 - \beta)$ match, we conclude that $\hat{a} = a$.
\end{proof}

This was a fact about formal power series, with no approximation or convergence implications.
The corresponding finite-$\ord$ result gives closeness of trajectories.

\subsection{Matching the Trajectories}
Recall the limiting memoryless approximation coefficients $\nomemcoef{j}{\infty}(\vtheta)$ from \eqref{eq:nomem-coef-limiting}.
By \eqref{eq:nomem-coef-limiting-explicit}, they are given by
\begin{align*}
  \nomemcoef{1}{\infty}(\vtheta) &= -\frac{1}{1 - \beta} \nabla L(\vtheta),
  \\
  \nomemcoef{j}{\infty}(\vtheta)
  &= -\frac{\beta}{j!}
    \sum_{\tau \in \labrooted \range{1}{j}} q(\tau) \nabla^\tau L(\vtheta), \qquad j
    \geq 2,
\end{align*}
and rewritten as
\begin{equation*}
  \nomemcoef{j}{\infty}(\vtheta)
    = \frac{1}{j!}
  \sum_{\tau \in \labrooted \range{1}{j}} a(\tau) \nabla^\tau L(\vtheta), \qquad j
    \geq 1.
  \end{equation*}

  \begin{proposition}\label{prop:fixed-point-and-memory-maps-match}
Assume $L(\cdot) \in C^{(\ord + 1) \vee 4}(\mathbb{R}^d, \mathbb{R})$ and
\begin{equation*}
\sup_{\vtheta} \norm{\nabla^j L(\vtheta)} \leq D_j, \quad j \in \range{1}{(\ord + 1) \vee 4}.
\end{equation*}
Then, for $h$ small enough, the maps
\begin{align*}
  \vF_h(\vtheta)
  &:= \vtheta - \frac{h}{1 - \beta} \nabla L(\vtheta) + h^2 \beta \vg_h(\vtheta),
  \\
  \vP_h(\vtheta) &:= \vtheta + \sum_{j = 1}^{\ord} h^j \nomemcoef{j}{\infty}(\vtheta)
\end{align*}
satisfy
\begin{equation*}
\sup_{\vtheta \in \mathbb{R}^d} \norm{\vF_h(\vtheta) - \vP_h(\vtheta)} \leq \nC{ljqhhq} h^{\ord + 1},
\end{equation*}
where $\oC{ljqhhq}$ depends on $\ord$, $\beta$, and $\crl{D_j}$ but not on $h$.
Consequently, two iterations
\begin{equation*}
\vtheta^{(n + 1)} = \vF_h(\vtheta^{(n)}), \qquad
\tilde{\vtheta}^{(n + 1)} = \vP_h(\tilde{\vtheta}^{(n)}),
\end{equation*}
starting at the same initial position $\vtheta^{(0)} = \tilde{\vtheta}^{(0)} \in \mathbb{R}^d$, satisfy
the global closeness bound
\begin{equation}\label{eq:fixed-point-memory-global-approximation}
\max_{n \in \range{0}{\lfloor T / h \rfloor}}
\norm{\vtheta^{(n)} - \tilde{\vtheta}^{(n)}}
\leq \oC{fixed-point-memory-global} h^{\ord},
\end{equation}
on any finite time interval $[0, T]$, where $\nC{fixed-point-memory-global}$ depends on
$T$, $\ord$, $\beta$, and $\crl{D_j}$ but not on $h$ or $\vtheta^{(0)}$.
\end{proposition}

\begin{remark}\label{rem:fixed-point-memory-global-approximation}
Here $\vtheta^{(n)}$ is the full-batch HB trajectory initialized on the
invariant manifold, that is, with
\begin{equation*}
\vv^{(0)} = -\frac{1}{1 - \beta}\nabla L(\vtheta^{(0)})
+ h\vg_h(\vtheta^{(0)}).
\end{equation*}
The global approximation by the $\vP_h$-based recurrence would in general fail for other initializations, including the standard $\vv^{(0)} = \boldsymbol{0}$: the initial momentum discrepancy can produce an error of order $h$ that persists during further iterations.
The finite-$n$ result, \cref{th:approx}, handles this issue.
\end{remark}

\begin{proof}
By the definition of $\vg_h$, we have
\begin{equation*}
\vF_h(\vF_h(\vtheta)) - (1 + \beta) \vF_h(\vtheta) + \beta \vtheta + h \nabla L(\vF_h(\vtheta)) = 0.
\end{equation*}
Accordingly, define the defect of the truncated map $\vP_h$ by
\begin{align*}
  \vE_h(\vtheta) &:= \vP_h(\vP_h(\vtheta)) - (1 + \beta) \vP_h(\vtheta) + \beta \vtheta + h \nabla L(\vP_h(\vtheta))\\
  &= \sum_{j = 1}^{\ord} h^j \brk[\big]{\nomemcoef{j}{\infty}(\vP_h(\vtheta)) - \beta \nomemcoef{j}{\infty}(\vtheta)} + h \nabla L(\vP_h(\vtheta)).
\end{align*}
Taylor-expand $\nomemcoef{j}{\infty}(\vP_h(\vtheta))$ up to degree $\ord - j$ around $\vtheta$,
and $\nabla L(\vP_h(\vtheta))$ up to degree $\ord - 1$ around $\vtheta$.
By the composition rule, the coefficient before $h^m$ in the resulting expansion of $\vE_h(\vtheta)$, $m \in \range{1}{\ord}$, is
\begin{equation*}
\frac{1}{m!} \sum_{\tau \in \labrooted \range{1}{m}} \brk[\big]{(a * a)(\tau) - (1 + \beta) a(\tau) + a(\tau_1) \ldots a(\tau_{\ell})} \nabla^\tau L(\vtheta).
\end{equation*}
This coefficient is zero. Indeed, for the singleton $\rootedtree[]$, the bracket is
\begin{equation*}
2 a(\rootedtree[]) - (1 + \beta) a(\rootedtree[]) + 1 = 0;
\end{equation*}
for $\abs{\tau} \geq 2$, we have
\begin{align*}
  &(a * a)(\tau) - (1 + \beta) a(\tau) + a(\tau_1) \ldots a(\tau_{\ell})
  \\
  &\quad = (1 - \beta) a(\tau) - \frac{\beta}{1 - \beta} a(\tau_1) \ldots a(\tau_{\ell}) + \sum_{\substack{\mrk \in \mrkset_{\tau} \setminus \crl{\varnothing} \\ \abs{\tau_0^{\mrk}} \geq 2}} a(\tau_0^{\mrk}) a(\tau_1^{\mrk}) \ldots a(\tau_{\abs{\mrk}}^{\mrk}) = 0
\end{align*}
by the recursion \eqref{eq:rtYrem} for $a$.
The $k$th derivatives of $\nomemcoef{j}{\infty}$ involve derivatives of $L$ of order at most $j + k$,
and $\sup_{\vtheta} \norm{\vP_h(\vtheta) - \vtheta} = O(h)$.
Thus, under the assumptions, the remainder after truncating the expansion in $h$ at order $\ord$ is bounded by $O(h^j h^{\ord - j + 1}) + O(h h^{\ord}) = O(h^{\ord + 1})$, and we have
\begin{equation}\label{eq:DIvymL}
\sup_{\vtheta} \norm{\vE_h(\vtheta)} \leq C_E h^{\ord + 1}.
\end{equation}

The fixed-point equation gives
\begin{equation*}
\vF^{-1}_h(\vtheta) = \vtheta + \frac{h}{1 - \beta} \nabla L(\vtheta) - h^2 \vg_h(\vtheta).
\end{equation*}
Further,
by the contraction mapping principle, the equation
\begin{equation*}
\vtheta = \vzeta - \sum_{j = 1}^{\ord} h^j \nomemcoef{j}{\infty}(\vtheta)
\end{equation*}
has exactly one solution in $\vtheta$ for each $\vzeta \in \mathbb{R}^d$ as long as $h$ is small enough, because
$\sum_{j = 1}^{\ord} h^j \nomemcoef{j}{\infty}(\vtheta)$ is $O(h)$-Lipschitz. Thus, $\vP_h(\vtheta)$ is bijective and in particular invertible.

The following two identities are now immediate from definitions:
\begin{align*}
  \vF_h(\vtheta)
  + \beta \vF_h^{-1}(\vtheta) &= (1 + \beta) \vtheta - h \nabla L(\vtheta),
  \\
  \vP_h(\vtheta)
  + \beta \vP_h^{-1}(\vtheta) &= (1 + \beta) \vtheta - h \nabla L(\vtheta)
    + \vE_h(\vP_h^{-1}(\vtheta)) .
\end{align*}
Subtracting them, we get
\begin{equation*}
\vP_h(\vtheta) - \vF_h(\vtheta) = - \beta \prn[\big]{\vP_h^{-1}(\vtheta) - \vF_h^{-1}(\vtheta)} + \vE_h(\vP_h^{-1}(\vtheta)),
\end{equation*}
which after taking supremum norms and using \eqref{eq:DIvymL} yields
\begin{equation}\label{eq:ordPqI}
\sup_{\vtheta} \norm{\vP_h(\vtheta) - \vF_h(\vtheta)} \leq \beta \sup_{\vtheta} \norm{\vP_h^{-1}(\vtheta) - \vF_h^{-1}(\vtheta)} + C_E h^{\ord + 1}.
\end{equation}
Let $K_{\mathrm{inv},h}$ be the Lipschitz constant of $\vF_h^{-1}$.
Recall from \eqref{eq:ljfhhh} and the proof of \cref{lem:t-maps-from-gamma-to-gamma} that it can be bounded as
\begin{equation*}
K_{\mathrm{inv},h} \leq \frac{1}{1 - h \ks} \leq \frac{2}{1 + \beta},
\end{equation*}
for small enough $h$ (so that $h \ks \leq \frac{1 - \beta}{2}$).
Then
\begin{align*}
  \norm{\vP_h^{-1}(\vtheta) - \vF_h^{-1}(\vtheta)}
  &= \norm{\vF_h^{-1}(\vF_h(\vP_h^{-1}(\vtheta))) - \vF_h^{-1}(\vtheta)}
  \\
  &\leq K_{\mathrm{inv},h} \norm{\vF_h(\vP_h^{-1}(\vtheta)) - \vP_h(\vP_h^{-1}(\vtheta))}
  \\
  &\leq K_{\mathrm{inv},h} \sup_{\vtheta} \norm{\vF_h(\vtheta) - \vP_h(\vtheta)}.
\end{align*}
Inserting this into \eqref{eq:ordPqI}, we conclude
\begin{align*}
  \sup_{\vtheta} \norm{\vP_h(\vtheta) - \vF_h(\vtheta)} &\leq \beta K_{\mathrm{inv},h} \sup_{\vtheta} \norm{\vP_h(\vtheta) - \vF_h(\vtheta)} + C_E h^{\ord + 1}
  \\
  &\leq \frac{2 \beta}{1 + \beta} \sup_{\vtheta} \norm{\vP_h(\vtheta) - \vF_h(\vtheta)} + C_E h^{\ord + 1},
\end{align*}
or
\begin{equation*}
\sup_{\vtheta} \norm{\vP_h(\vtheta) - \vF_h(\vtheta)} \leq \frac{1 + \beta}{1 - \beta} C_E h^{\ord + 1}
\end{equation*}
as desired. The global bound follows by a standard argument as in \cref{lem:global-error-bound}.
\end{proof}

\subsection{Finite-\texorpdfstring{$\ord$}{R} Losses}

Recall that the $\ord = 2$ correction to the loss in \eqref{eq:SZADJf} is (up to a scalar multiple),
\begin{equation*}
\sum_{i = 1}^d \partial_i L \, \partial_i L = \norm{\nabla L}^2.
\end{equation*}
This corresponds to attaching a coordinate index $i$ to the single edge of the tree $\rootedtree[[]]$, attaching a first-order derivative $\partial_i L$ to each of the two leaves and summing over $i$.
For $\ord = 3$, the correction is (up to a scalar multiple)
\begin{equation*}
\sum_{i = 1}^d \sum_{j = 1}^d \partial_i L \, \partial_{i j} L \, \partial_j L = \nabla L\trans \nabla^2 L \nabla L,
\end{equation*}
corresponding to attaching coordinate indices $i, j$ to each edge of $\rootedtree[[][]]$, attaching corresponding first- and second-order derivatives to vertices (the leaf incident to edge $i$ receives $\partial_i L$, the leaf incident to edge $j$ receives $\partial_j L$, and the non-leaf vertex incident to both edges receives $\partial_{i j} L$), and summing over all indices.
In general, for an unlabeled unrooted tree $\tau \in \unlab$ define
\begin{equation*}
\Phi_\tau(L) = \sum_{i\colon E \to \range{1}{d}} \prod_{v \in V} \partial_{i_{e_1(v)} \ldots i_{e_{\deg(v)}(v)}} L,
\end{equation*}
where $e_1(v), \ldots, e_{\deg(v)}(v)$ are the edges incident to $v$, the map $i\colon E \to \range{1}{d}$ is a labeling of the edges by coordinate indices, and $(V, E)$ is a representative of the isomorphism class $\tau$ (clearly, $\Phi_\tau(L)$ does not depend on the choice of representative).

\begin{proposition}\label{prop:finite-order-modified-loss}
Fix $\ord \geq 2$ and assume $L \in C^{\ord}(\mathbb{R}^d, \mathbb{R})$.
The following equality holds:
\begin{equation*}
\vtheta + \sum_{m = 1}^{\ord} h^m \nomemcoef{m}{\infty}(\vtheta) = \vtheta - \frac{h}{1 - \beta} \nabla L_h^{[\ord]}(\vtheta),
\end{equation*}
with the modified loss given by
\begin{equation*}
L_h^{[\ord]} := L + \sum_{m = 2}^{\ord} \frac{\beta h^{m - 1}}{(1 - \beta)^{2 m - 2}} \sum_{\tau \in \unlab[m]} \frac{e_{\tau, 1}(\beta)}{\abs{\mathrm{Aut}(\tau)}} \Phi_\tau(L);
\end{equation*}
$\abs{\mathrm{Aut}(\tau)}$ denoting the order of the automorphism group of any representative of the isomorphism class $\tau$.

Moreover, under the assumptions of \cref{prop:fixed-point-and-memory-maps-match}, we have for sufficiently small $h$
\begin{equation*}
\sup_{\vtheta} \norm{\nabla L_h^{[\ord]}(\vtheta) - \nabla L_h(\vtheta)} \leq (1 - \beta) \oC{ljqhhq} h^{\ord},
\end{equation*}
where
\begin{equation*}
L_h := L - h \beta (1 - \beta) G_h
\end{equation*}
is the full-order modified loss from \cref{contrib:manifold}.
In particular, for every reference point $\vtheta_\ast \in \mathbb{R}^d$,
\begin{equation*}
\begin{aligned}
  \abs{\prn{L_h^{[\ord]} - L_h}(\vtheta) - \prn{L_h^{[\ord]} - L_h}(\vtheta_\ast)}
  \leq (1 - \beta) \oC{ljqhhq} h^{\ord} \norm{\vtheta - \vtheta_\ast}, \quad \vtheta \in \mathbb{R}^d.
\end{aligned}
\end{equation*}
Thus, normalizing the losses in such a way that $L_h^{[\ord]}(\vtheta_{*}) = L_h(\vtheta_{*}) = 0$
leads to the bound $O(h^{\ord})$ on their difference uniformly on any fixed bounded set.
\end{proposition}

\begin{proof}
Fix $m \in \range{2}{\ord}$ and $\tau \in \unlab[m]$ with a representative $G = (V, E)$.
For each $r \in V$, let $G_r := (G, r)$ be the rooted tree obtained by choosing $r$ as the root.
Directly differentiating, we obtain
\begin{align*}
  \partial_a \Phi_{\tau}(L) &= \sum_{i\colon E \to \range{1}{d}} \sum_{r \in V} \partial_{a i_{e_1(r)} \ldots i_{e_{\deg(r)}(r)}} L \prod_{v \in V \setminus \crl{r}} \partial_{i_{e_1(v)} \ldots i_{e_{\deg(v)}(v)}} L
  \\
               &= \sum_{r \in V} \sum_{i\colon E \to \range{1}{d}} \partial_{a i_{e_1(r)} \ldots i_{e_{\deg(r)}(r)}} L \prod_{v \in V \setminus \crl{r}} \partial_{i_{e_1(v)} \ldots i_{e_{\deg(v)}(v)}} L
  \\
  &= \sum_{r \in V} \crl[big]{\nabla^{G_r} L}_a,
\end{align*}
or in other words
\begin{equation*}
\nabla \Phi_{\tau}(L) = \sum_{r \in V} \nabla^{G_r} L.
\end{equation*}
Group the sum over $r \in V$ by the orbits of the action of $\mathrm{Aut}(G)$ on $V$ ($r$ and $r'$ lie in the same orbit if and only if $G_r$ and $G_{r'}$ are isomorphic as rooted trees):
\begin{equation*}
\nabla \Phi_{\tau}(L) = \sum_{O \in V / \mathrm{Aut}(G)} \abs{O} \nabla^{G_{r_O}} L,
\end{equation*}
where $r_O \in O$ is a representative of the orbit $O$. The stabilizer group of $r_O$ under the action of the group $\mathrm{Aut}(G)$ is precisely the automorphism group of $G_{r_O}$ as a rooted tree, whose order is $\sigma(G_{r_O})$, which gives by the orbit-stabilizer theorem
\begin{equation*}
\abs{O} = \frac{\abs{\mathrm{Aut}(G)}}{\sigma(G_{r_O})},
\end{equation*}
and we can continue
\begin{equation*}
\nabla \Phi_{\tau}(L) = \sum_{O \in V / \mathrm{Aut}(G)} \frac{\abs{\mathrm{Aut}(G)}}{\sigma(G_{r_O})} \nabla^{G_{r_O}} L.
\end{equation*}
Every rooted isomorphism class
corresponds to exactly one pair consisting of an unrooted
isomorphism class and a vertex orbit. This correspondence gives
\begin{equation*}
  \nabla \prn[\bigg]{
    \sum_{\tau \in \unlab[m]}
    \frac{e_{\tau, 1}(\beta)}{\abs{\mathrm{Aut}(\tau)}} \Phi_\tau(L)
  }
    = \sum_{\rho \in \unlabrooted[m]}
  \frac{e_{\rho, 1}(\beta)}{\sigma(\rho)} \nabla^\rho L.
\end{equation*}
Inserting this into $\nabla L_h^{[\ord]}$ and recalling \eqref{eq:nomem-coef-limiting-explicit}, we have
\begin{align*}
  \nabla L_h^{[\ord]}
  &= \nabla L
    + \sum_{m = 2}^{\ord} \frac{\beta h^{m - 1}}{(1 - \beta)^{2 m - 2}}
    \sum_{\rho \in \unlabrooted[m]} \frac{e_{\rho, 1}(\beta)}{\sigma(\rho)} \nabla^\rho L
  \\
  &= -(1 - \beta) \sum_{m = 1}^{\ord} h^{m - 1} \nomemcoef{m}{\infty}.
\end{align*}

The gradient bound now follows directly from \cref{prop:fixed-point-and-memory-maps-match} since
\begin{equation*}
\nabla L^{[\ord]}_h - \nabla L_h = \frac{1 - \beta}{h} (\vF_h - \vP_h).
\end{equation*}
The scalar-loss bound follows by integrating $\nabla L_h^{[\ord]} - \nabla L_h$
along the segment between $\vtheta_\ast$ and $\vtheta$.
\end{proof}

For example, since
\begin{align*}
  &\Phi_{\rootedtree[[]]}(L) = \norm{\nabla L}^2,
  \\
  &\Phi_{\rootedtree[[][]]}(L) = \nabla L \trans \nabla^2 L \nabla L,
\end{align*}
and the automorphism
orders of the two trees are $2$,
we have
\begin{equation*}
L_h^{[3]} = L
+ \frac{\beta h}{2 (1 - \beta)^2}
  \norm{\nabla L}^2
+ \frac{\beta(1 + \beta)h^2}{2 (1 - \beta)^4}
  \nabla L\trans \nabla^2 L \nabla L.
\end{equation*}
Note that neglecting the third-order derivative in
\begin{equation*}
  \nabla L_h^{[3]}
    = \nabla L + \frac{\beta h}{(1 - \beta)^2} \nabla^2 L \nabla L
    + \frac{\beta(1 + \beta) h^2}{2(1 - \beta)^4}
  \brk[\big]{2(\nabla^2 L)^2 \nabla L + \nabla^3 L[\nabla L, \nabla L]} .
\end{equation*}
would generally result in a non-gradient.
In particular, the principal iteration in \cref{sec:principal-iter}, while explicit and useful to capture instabilities, is typically not gradient descent.

\section{Continuous Time}\label{sec:continuous}

Our main theorems (\cref{th:approx} and \cref{th:nomem-coef}) can be used to obtain useful additional results for the analysis of HB. This section focuses on deriving continuous modified equations of arbitrary approximation order (with rigorous bounds), as well as the principal flow capturing the HB dynamics, generalizing \citet{rosca2023on}. Due to their practical importance, we consider both full-batch and mini-batch implementations.

\subsection{Modified Equation}\label{sec:modified-equation}

The global approximation by a memoryless iteration (\cref{th:approx}) helps find a modified equation or, in other words, an approximation by a continuous flow. Define the \textit{BEA coefficients} $\crl[\big]{\beacoef{j}{n}(\btheta)}_{j = 1}^{\infty}$ by
\begin{equation}\label{eq:bea-coef-def}
  \beacoef{j}{n}(\btheta) = \nomemcoef{j}{n}(\btheta) - \sum_{i = 2}^j \frac{1}{i!} \sum_{\substack{k_1, \ldots, k_i \geq 1\\k_1 + \ldots + k_i = j}} (D^{(n)}_{k_1} \ldots D^{(n)}_{k_{i - 1}} \beacoef{k_i}{n})(\btheta),
\end{equation}
with the $i$th Lie derivative $D^{(n)}_i := \sum_l \beacoefsc{i}{n}{l}(\btheta) \partial_l$. This formula is standard in the literature on backward error analysis applied to numerical methods \citep{hairer2006}.

We now state the continuous approximation result.

\begin{corollary}[Modified equation]\label{cor:modified-eq}
Assume the conditions of \cref{th:approx}. Let $\btheta(t) \equiv \btheta(\ord; t)$ be the unique continuous solution to the piecewise ODE in $\mathcal{D}$
\begin{equation}\label{eq:ode-def}
    \dot{\btheta}(t) = \sum_{i = 0}^{\ord - 1} h^i \beacoef{i + 1}{n}(\btheta(t))
\end{equation}
on $t \in [t_n, t_{n + 1}]$ with the initial condition $\btheta(0) = \btheta^{(0)}$, assumed to exist on $t \in [0, T]$, where we use the shortcut $t_n := n h$ and the derivatives are one-sided at $t_n$, $t_{n + 1}$. Then for each $h \in (0, 1 / 2)$
\begin{equation*}
\sup_{n \in \range{0}{\lfloor T / h \rfloor}} \norm{\btheta^{(n)} - \btheta(t_n)} \leq \oC{mod} h^{\ord},
\end{equation*}
where $\nC{mod}$ is some constant depending on $\beta$, $\ord$, $T$ and the uniform derivative bounds, but is independent of $n$ and $h$.
\end{corollary}

\begin{remark}
The ODE solution is easily proven to exist by classical existence results if $\mathcal{D} = \mathbb{R}^d$,
but since $\mathcal{D}$ is not restricted, we assume this to avoid uninformative technicalities.
\end{remark}

The proof is by Taylor-expanding $\btheta(t)$ around $t_n$ at $t_{n + 1}$ and rearranging; this allows us to get the local error bound, which is then converted into the global error bound. Full details are provided in \cref{sec:proof-of-modified-eq}.

For example, the second-order approximation is
\begin{equation}\label{eq:XbfyrE}
\dot{\btheta}(t) = - \sum_{b = 0}^n \beta^b \nabla L^{(n - b)}(\btheta(t)) + h \beacoef{2}{n}(\btheta(t))
\end{equation}
with
\begin{align*}
  \beacoef{2}{n}(\btheta) = &- \sum_{b = 1}^n \beta^b \nabla^2 L^{(n - b)}(\btheta) \sum_{l = 1}^b \sum_{b' = 0}^{n - l} \beta^{b'} \nabla L^{(n - l - b')}(\btheta)\\
  &- \frac{1}{2} \sum_{b = 0}^n \beta^b \nabla^2 L^{(n - b)} \sum_{b' = 0}^n \beta^{b'} \nabla L^{(n - b')}(\btheta).
\end{align*}

In the simpler full-batch case (where all $L^{(s)} \equiv L$) \eqref{eq:XbfyrE} is rewritten as
\begin{align*}
  \dot{\btheta}(t) &= - \frac{1 - \beta^{n + 1}}{1 - \beta} \nabla L(\btheta(t))\\
  &\quad - h \frac{1 + \beta - 4 (n + 1) (\beta^{n + 1} - \beta^{n + 2}) - \beta^{2 n + 2} - \beta^{2 n + 3}}{2 (1 - \beta)^3} \nabla^2 L(\btheta(t)) \nabla L(\btheta(t)).
\end{align*}
on the segment $t \in [t_n, t_{n + 1}]$. As $n$ becomes large, the coefficients converge to those of the autonomous ODE
\begin{equation}\label{eq:modified-ode-smooth}
\dot{\btheta}(t) = - \frac{1}{1 - \beta} \nabla L(\btheta(t)) - h \frac{1 + \beta}{2 (1 - \beta)^3} \nabla^2 L(\btheta(t)) \nabla L(\btheta(t)).
\end{equation}
The equation found in \citet{kovachki2021continuous} (after fixing a small typo) is the same ODE.
For \eqref{eq:modified-ode-smooth}, initialization on the invariant manifold is enough for the guarantee \eqref{eq:ode-approx-contrib} with $\ord = 2$, but
generally, initialization outside the manifold can lead to an $O(h)$ trajectory difference.
We did not lose any information but gained a guarantee in the case of the common setting $\vv^{(0)} = \boldsymbol{0}$, typically outside the manifold. (A minimally modified argument would work for any $\vv^{(0)}$.)

\begin{remark}\label{rem:not-gf}
Note that \eqref{eq:modified-ode-smooth} can be rewritten as
\begin{equation*}
  \dot{\btheta}(t) = - \frac{1}{1 - \beta} \nabla \crl[\bigg]{
    L + h \frac{1 + \beta}{4 (1 - \beta)^2} \norm{\nabla L}^2
  }(\btheta(t)).
\end{equation*}
This can be seen as gradient flow with a modified loss \citep{farazmand2020multiscale,kovachki2021continuous,ghosh2023implicit}. However, for $\ord = 3$ this is already not true: $\beacoef{3}{n}(\btheta)$ becomes close to
\begin{equation*}
- \frac{1 + 4 \beta + \beta^2}{3 (1 - \beta)^5} \nabla^2 L(\btheta) \nabla^2 L(\btheta) \nabla L(\btheta) - \frac{1 + 10 \beta + \beta^2}{12 (1 - \beta)^5} \nabla^3 L(\btheta) \brk[\big]{\nabla L(\btheta), \nabla L(\btheta)},
\end{equation*}
which is in general not a gradient.
\end{remark}

\subsection{Principal Flow}\label{sec:principle-flow}

In this subsection, assume that $L^{(n)} \equiv L$.
Using the same framework as in \cref{sec:principal-iter}, we can derive a ``full-order'' modified equation up to non-principal terms.

\begin{corollary}[Principal flow]\label{cor:principal-flow}
    Define $\crl{z_m^{(n)}}$ as such coefficients that the principal part of $\beacoef{m}{n}(\btheta)$ is $z_m^{(n)} \crl{\nabla^2 L(\btheta)}^{m - 1} \nabla L(\btheta)$. Then,
    \begin{equation}\label{eq:z-coef-rec}
        \zcoef{m}{n} := \sum_{l = 1}^m \frac{(-1)^{l + 1}}{l} \sum_{\substack{k_1, \ldots, k_l \geq 1\\k_1 + \ldots + k_l = m}} p^{(n)}_{k_1} \ldots p^{(n)}_{k_l},
    \end{equation}
    where
    \begin{equation*}
        p^{(n)}_k := \begin{cases}
        - \frac{1 - \beta^{n + 1}}{1 - \beta}, &\text{if } k = 1,\\
        - \frac{\beta \e{\mathcal{c}_k}{1}{n}(\beta)}{(1 - \beta)^{2 k - 1}}, &\text{if } k \geq 2.
        \end{cases}
    \end{equation*}
    Moreover, the limiting sequence $\zcoef{m}{\infty} := \lim_{n \to \infty} \zcoef{m}{n}$ admits a generating function $\bar{g}_{\beta}(x) := \sum_{k = 0}^{\infty} \zcoef{k + 1}{\infty} x^k$ given by
    \begin{equation*}
        \bar{g}_{\beta}(x) = \frac{1}{x} \ln \prn[\bigg]{\frac{1 + \beta - x + \sqrt{(1 - \beta - x)^2 - 4 \beta x}}{2}}.
    \end{equation*}
\end{corollary}

The proof is given in \cref{sec:proof-of-principal-flow}.

Informally, the ``full-order'' modified equation (called ``principal flow'' by \citet{rosca2023on}) is approximately
\begin{align*}
  \dot{\btheta}(t)
  &= \sum_{k = 0}^{\infty} z_{k + 1}^{(n)} h^k \crl{\nabla^2 L(\btheta(t))}^{k} \nabla L(\btheta(t)) + \np\\
  &\approx \bar{g}_{\beta}\prn[\big]{h \nabla^2 L(\btheta(t))} \nabla L(\btheta(t)) + \np
\end{align*}
for large $n$. Taking $\beta = 0$, we recover the result of \citet{rosca2023on} for GD.

For example, consider the one-dimensional case and $L(\theta) = \theta^2 / 2$. Its gradient is unbounded but we can identify the invariant manifolds directly without appealing to \cref{th:manifold}. The HB iteration \eqref{eq:hb} is solved by
\begin{equation*}
  \begin{pmatrix}
    \theta^{(n)}\\
    v^{(n)}
  \end{pmatrix} =
  \begin{pmatrix}
    1 - h & h \beta\\
    - 1 & \beta
  \end{pmatrix}^n
  \begin{pmatrix}
    \theta^{(0)}\\ v^{(0)}.
  \end{pmatrix}
\end{equation*}
The general solution is
\begin{align*}
  &\theta^{(n)} = \frac{((1 - h - \lambda_-) \theta^{(0)} + h \beta v^{(0)}) \lambda_+^n + ((\lambda_+ + h - 1) \theta^{(0)} - h \beta v^{(0)}) \lambda_-^n}{\lambda_+ - \lambda_-},\\
  &v^{(n)} = \frac{\prn[\big]{(\beta - \lambda_-) v^{(0)} - \theta^{(0)}} \lambda_+^n + \prn[\big]{\theta^{(0)} + (\lambda_+ - \beta) v^{(0)}} \lambda_-^n}{\lambda_+ - \lambda_-}
\end{align*}
with
\begin{equation*}
\lambda_{\pm} = \frac{1 + \beta - h \pm \sqrt{(1 - \beta - h)^2 - 4 \beta h}}{2},
\end{equation*}
excluding for simplicity the case of a double root $\lambda_+ = \lambda_-$ (which happens for $h = (1 \pm \sqrt{\beta})^2$).
For $0 < h < (1 - \sqrt{\beta})^2$, the square root is positive and $0 < \lambda_- < \lambda_+ < 1$.
The $\lambda_+$ eigenspace $\crl{(\theta, v): v = s_+ \theta}$ with $s_+ := (\lambda_+ + h - 1) / (h \beta)$
is an invariant graph: if the initialization is on it, that is,
$v^{(0)} = s_+ \theta^{(0)}$, then
\begin{equation}\label{eq:oMWSQv}
\theta^{(n)} = \lambda_+^n \theta^{(0)}.
\end{equation}
It is  attractive since
$v^{(n + 1)} - s_+ \theta^{(n + 1)} = \lambda_-(v^{(n)} - s_+ \theta^{(n)})$.
The solution of the principal flow $\dot{\theta}(t) = \bar{g}_{\beta}\prn[\big]{h} \theta(t)$ in this case is
\begin{equation}\label{eq:XtfJtO}
\theta(t) = \theta(0) \exp \crl[\bigg]{\frac{t}{h} \ln \prn[\bigg]{\frac{1 + \beta - h + \sqrt{(1 - \beta - h)^2 - 4 \beta h}}{2}}},
\end{equation}
coinciding with \eqref{eq:oMWSQv} at points $t = n h$.

For large step sizes $(1 + \sqrt{\beta})^2 < h < 2 + 2 \beta$, the $\lambda_-$
eigenspace $\crl{(\theta, v): v = s_- \theta}$ with
$s_- := (h + \lambda_- - 1) / (h \beta)$ (blowing up as $h \to 0$),
is invariant and attractive. If the initialization happened there, the solution will be
\begin{equation*}
\theta^{(n)} = \lambda_{-}^n \theta^{(0)}.
\end{equation*}
This also coincides with principal flow \eqref{eq:XtfJtO} by choosing appropriate branches of the complex square root and logarithm. We illustrate both situations in \cref{fig:pf}.

\begin{figure}[htb!]
  \centering
  \begin{subfigure}[t]{0.480\textwidth}
    \includegraphics[width=\linewidth]{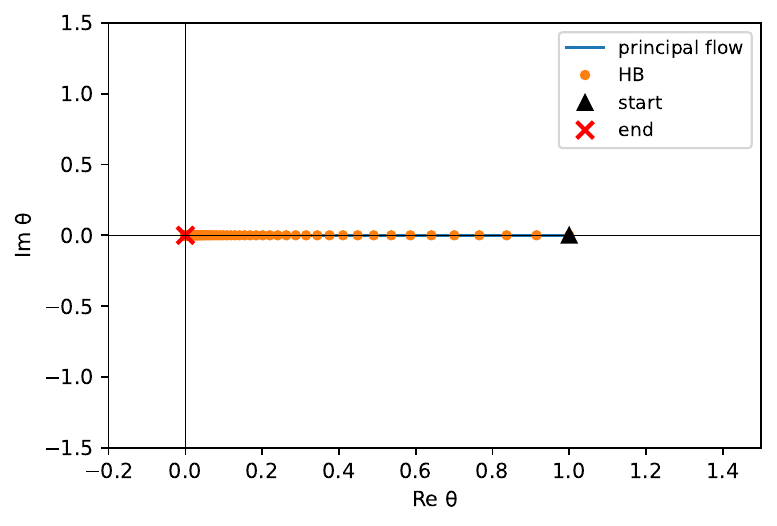}
    \caption{Small $h = 0.02$\label{subfig:real}}
  \end{subfigure}
  \hfill
  \begin{subfigure}[t]{0.480\textwidth}
    \includegraphics[width=\linewidth]{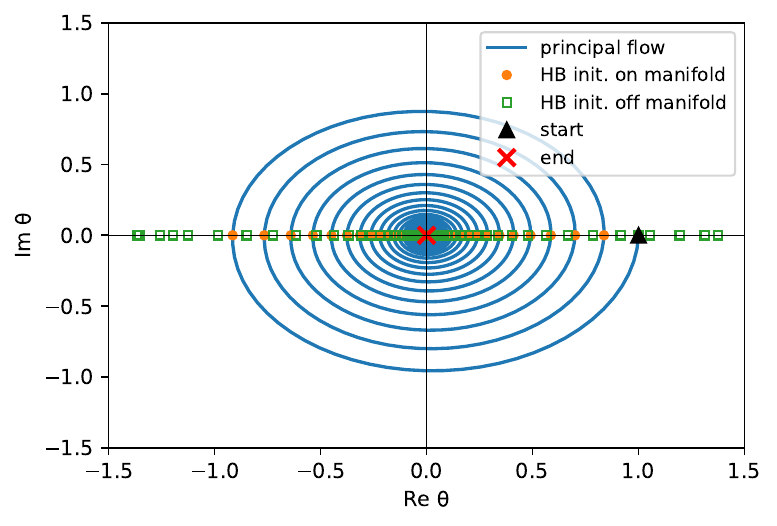}
    \caption{Large $h = 3.38$\label{subfig:complex}}
  \end{subfigure}
  \caption{Principal flow ($\beta = 0.7$, $\theta^{(0)} = 1$) corresponding to the quadratic loss $L(\theta) = \theta^2 / 2$ on the complex plane. It is purely real for small enough step size $h$. For high step sizes, it is not purely real, and the real values $\theta(n h)$ capture the oscillatory behavior of HB (with HB exactly matching the flow if initialized on the attractive invariant manifold).\label{fig:pf}}
\end{figure}

For intermediate step sizes $(1 - \sqrt{\beta})^2 < h < (1 + \sqrt{\beta})^2$, the eigenvalues are not purely real:
\begin{equation*}
\lambda_{\pm} = \sqrt{\beta} e^{\pm i \varphi},
\quad
\cos \varphi = \frac{1 + \beta - h}{2 \sqrt{\beta}},
\quad \varphi \in (0, \pi),
\end{equation*}
and the invariant eigenlines $v = s_{\pm} \theta$ are complex. Taking the square root with a positive imaginary part and the logarithm so that $\log \lambda_+ = \frac{1}{2} \ln \beta + i \varphi$, the complex principal flow has the solution
\begin{equation*}
z(t) = \prn[\bigg]{\theta^{(0)} - i \frac{(1 - h) \theta^{(0)} + h \beta v^{(0)} - \sqrt{\beta} \theta^{(0)} \cos\varphi}{\sqrt{\beta} \sin\varphi}} \exp\crl[\bigg]{\frac{t}{h} \prn[\bigg]{\frac{1}{2} \ln\beta + i \varphi}}
\end{equation*}
whose real part coincides with the HB trajectory at points $t = n h$, that is
\begin{equation*}
  \operatorname{Re} z(nh) = \theta^{(n)}
  = \beta^{n / 2} \brk[\bigg]{\theta^{(0)} \cos(n \varphi) + \frac{(1 - h) \theta^{(0)} + h \beta v^{(0)} - \sqrt{\beta} \theta^{(0)}\cos \varphi}{\sqrt{\beta} \sin \varphi} \sin(n \varphi)}
\end{equation*}
for every $n \geq 0$, and any $v^{(0)}$.

\section{Concluding Remarks}

We studied the theoretical properties of gradient descent with heavy-ball momentum \citep{polyak1964some}, one of the simplest and most commonly used optimization algorithms with memory. We established an approximation of the algorithm by a memoryless iteration to arbitrary precision. In the full-batch case, this memoryless iteration is (roughly speaking) plain gradient descent with a modified loss. This loss modification can be seen as implicit regularization by memory and can sometimes explain good empirical performance \citep{barrett2021implicit,smith2021on,ghosh2023implicit}.
In the stochastic (mini-batch) case, additional implicit regularization by mini-batch noise can be identified. These insights can be of practical importance in machine learning, not just for understanding existing algorithms, but also for proposing new ones, e.\,g., by introducing (or strengthening) similar regularization explicitly \citep{foret2021sharpnessaware,zhao2022penalizinggradientnorm}. In addition, analyzing the terms of the memoryless iteration revealed rich combinatorics hidden inside the algorithm, which may be of independent interest.

\citet{bernstein2024old} note that ``the precise role of EMA [exponential moving averages] is perhaps still an open problem'' in optimization. Even though we focused the presentation specifically on HB, one of the strengths of our techniques is that they can be used to study other algorithms with decaying memory (defined by smooth enough functions of the parameter $\btheta$). Thus, our paper not only takes a step toward the broader goal of understanding memory and its effects on a specific optimization algorithm, but also outlines a more general framework that can be used to analyze other algorithms. For example, our techniques could be used to study other (ubiquitous) numerical methods with memory such as Adam \citep{kingma2017adammethodstochasticoptimization} or Shampoo \citep{gupta2018shampoopreconditionedstochastic}.

\section*{Acknowledgments}

We thank Boris Hanin for insightful comments and discussions. Cattaneo gratefully acknowledges financial support from the National Science Foundation through grants SES-2019432 and SES-2241575, and the John Simon Guggenheim Memorial Foundation through a 2026 Guggenheim Fellowship.

\textbf{AI disclosure.}
The initial draft of this paper was created with minimal technical AI assistance apart from proofreading.
During subsequent improvements, we used ChatGPT for problem-solving assistance as well.
In particular, \cref{thm:e-h-poly} (and subsequent \cref{cor:e-properties}) was discovered empirically and then proven rigorously in a helpful interaction with ChatGPT~5.
ChatGPT~6 assisted in ideation and writeup of
\cref{cor:critical-points,prop:exact-minibatch-gradient,prop:matching-power-series,prop:fixed-point-and-memory-maps-match,prop:finite-order-modified-loss}, made useful literature and structural suggestions.
All the text of the paper, including all the proofs, is human-written.
(Manual incorporation by the authors of specific AI suggestions or completions was allowed, but direct AI editing or large AI generations were not.)

\clearpage
\appendix

\section{Proof of \cref{th:manifold}}\label{sec:manifold-details}

\subsection{Constants}

Fix the constants defining $\Gamma$ as follows:
\begin{equation}\label{eq:manifold-fixed-constants}
\begin{aligned}
  \gamma &:= 1 + \frac{D_1 D_2}{(1 - \beta)^3},\\
  \delta &:= 1 + \frac{D_2^2 + D_1 D_3}{(1 - \beta)^3},\\
  \lambda &:= 1 + \frac{3 D_2 D_3 + D_1 D_4}{(1 - \beta)^3}.
\end{aligned}
\end{equation}
It will also be convenient to define the auxiliary bounds
\begin{align}
  &\kf := D_1 / (1 - \beta) + h \beta \gamma,\label{eq:k1-def}\\
  &\ks := D_2 / (1 - \beta) + h \beta \delta.\label{eq:k2-def}
\end{align}

\subsection{Omitted Lemmas}

\begin{lemma}[$\btheta \mapsto \bzeta$ is bijective]\label{lem:ljrhh}
For small enough $h$ and for any $\boldsymbol{g} \in \Gamma$, the function $\btheta \mapsto \bzeta = \btheta - h (1 - \beta)^{-1} \nabla L(\btheta) + h^2 \beta \boldsymbol{g}(\btheta)$ is a bijection $\mathbb{R}^d \to \mathbb{R}^d$.
\end{lemma}

\begin{proof}
This is proven the same way as in \citet{kovachki2021continuous}.
Specifically, for fixed $\vzeta \in \mathbb{R}^d$ and $\vg \in \Gamma$, the map
\begin{equation*}
\vtheta \mapsto \vzeta + \frac{h}{1 - \beta} \nabla L(\vtheta) - h^2 \beta \vg(\vtheta)
\end{equation*}
is Lipschitz with constant $h D_2 / (1 - \beta) + h^2 \beta \delta = h \ks$, which is less than $1$ for small enough $h$. By the contraction mapping theorem, it has a unique fixed point. The existence and uniqueness of the fixed point is equivalent to the surjectivity and injectivity of the map $\vtheta \mapsto \vzeta$ in the statement.
\end{proof}

\begin{lemma}[$T$ maps from $\Gamma$ to $\Gamma$]\label{lem:t-maps-from-gamma-to-gamma}
For small enough $h$, the operator $T$ defined in~\eqref{eq:def-of-t} maps from $\Gamma$ to $\Gamma$.
\end{lemma}

\begin{proof}
Take $\boldsymbol{g} \in \Gamma$. By \cref{lem:ljrhh}, the function $\btheta \mapsto \bzeta = \btheta - h (1 - \beta)^{-1} \nabla L(\btheta) + h^2 \beta \boldsymbol{g}(\btheta)$ is a bijection. This function is continuously differentiable:
\begin{equation*}
\nabla_{\btheta} \bzeta(\btheta) = I - \frac{h}{1 - \beta} \nabla^2 L(\btheta) + h^2 \beta \nabla \boldsymbol{g}(\btheta).
\end{equation*}
For $h$ small enough, this matrix is invertible. By the inverse mapping theorem, we see that the inverse mapping is continuously differentiable with Jacobian
\begin{equation}\label{eq:ljfhhh}
\nabla_{\bzeta} \btheta(\bzeta) = \prn[\bigg]{I - \frac{h}{1 - \beta} \nabla^2 L(\btheta(\bzeta)) + h^2 \beta \nabla \boldsymbol{g}(\btheta(\bzeta))}^{-1}.
\end{equation}
Therefore, the function $T \boldsymbol{g}$ is also continuously differentiable with derivative
\begin{align*}
  &\nabla_{\bzeta} T \boldsymbol{g}(\bzeta) = \frac{1}{h (1 - \beta)} \nabla_{\bzeta} \crl[\big]{\nabla L(\bzeta) - \nabla L(\btheta(\bzeta))} + \beta \nabla_{\bzeta} \crl[\big]{\boldsymbol{g}(\btheta(\bzeta))}\\
  &\quad = \frac{1}{h (1 - \beta)} \crl[\big]{\nabla^2 L(\bzeta) - \nabla^2 L(\btheta(\bzeta)) \nabla_{\bzeta} \btheta(\bzeta)} + \beta \nabla \boldsymbol{g}(\btheta(\bzeta)) \nabla_{\bzeta} \btheta(\bzeta)\\
  &\quad = \frac{1}{h (1 - \beta)} \crl[\bigg]{\nabla^2 L(\bzeta) - \nabla^2 L(\btheta(\bzeta)) \prn[\bigg]{I - \frac{h}{1 - \beta} \nabla^2 L(\btheta(\bzeta)) + h^2 \beta \nabla \boldsymbol{g}(\btheta(\bzeta))}^{-1}}\\
  &\qquad + \beta \nabla \boldsymbol{g}(\btheta(\bzeta)) \prn[\bigg]{I - \frac{h}{1 - \beta} \nabla^2 L(\btheta(\bzeta)) + h^2 \beta \nabla \boldsymbol{g}(\btheta(\bzeta))}^{-1}\\
  &\quad = \frac{1}{h (1 - \beta)} \nabla^2 L(\bzeta)\\
  &\qquad + \frac{1}{h^2} \prn[\bigg]{- \frac{h}{1 - \beta} \nabla^2 L(\btheta(\bzeta)) + h^2 \beta \nabla \boldsymbol{g}(\btheta(\bzeta))} \prn[\bigg]{I - \frac{h}{1 - \beta} \nabla^2 L(\btheta(\bzeta)) + h^2 \beta \nabla \boldsymbol{g}(\btheta(\bzeta))}^{-1}\\
  &\quad = \frac{\nabla^2 L(\bzeta) - \nabla^2 L(\btheta(\bzeta))}{h (1 - \beta)} + \beta \nabla \boldsymbol{g}(\btheta(\bzeta)) + \frac{1}{h^2} \prn[\bigg]{\frac{h}{1 - \beta} \nabla^2 L(\btheta(\bzeta)) - h^2 \beta \nabla \boldsymbol{g}(\btheta(\bzeta))}\\
  &\qquad - \frac{1}{h^2} \prn[\bigg]{\frac{h}{1 - \beta} \nabla^2 L(\btheta(\bzeta)) - h^2 \beta \nabla \boldsymbol{g}(\btheta(\bzeta))} \prn[\bigg]{I - \frac{h}{1 - \beta} \nabla^2 L(\btheta(\bzeta)) + h^2 \beta \nabla \boldsymbol{g}(\btheta(\bzeta))}^{-1}\\
  &\quad = \frac{\nabla^2 L(\bzeta) - \nabla^2 L(\btheta(\bzeta))}{h (1 - \beta)} + \beta \nabla \boldsymbol{g}(\btheta(\bzeta))\\
  &\qquad + \frac{1}{h^2} \prn[\bigg]{\frac{h}{1 - \beta} \nabla^2 L(\btheta(\bzeta)) - h^2 \beta \nabla \boldsymbol{g}(\btheta(\bzeta))}\\
  &\hspace{10em} \times \crl[\bigg]{I - \prn[\bigg]{I - \frac{h}{1 - \beta} \nabla^2 L(\btheta(\bzeta)) + h^2 \beta \nabla \boldsymbol{g}(\btheta(\bzeta))}^{-1}}.\numberthis\label{eq:derivative-of-t}
\end{align*}
This matrix is symmetric for any $\boldsymbol{g} \in \Gamma$ because any square matrix $A$ commutes with $(I - A)^{-1}$ (provided the latter exists). We have proven that $T \boldsymbol{g}$ is a continuously differentiable function and its Jacobian is symmetric.

\paragraph{Bounding $T \boldsymbol{g}$}
We need to show that the norm of the function $T \boldsymbol{g}$ does not exceed $\gamma$ for any $g \in \Gamma$. We have for any $\bzeta \in \mathbb{R}^d$
\begin{align*}
  \norm{T \boldsymbol{g}(\bzeta)} &\leq \frac{\sup_{\btheta} \norm{\nabla^2 L(\btheta)}}{h (1 - \beta)} \norm{\bzeta - \btheta(\bzeta)} + \beta \norm{\boldsymbol{g}(\btheta(\bzeta))}\\
  \intertext{(where we used the definition of $T$ in \cref{eq:def-of-t})}
                                  &= \frac{\sup_{\btheta} \norm{\nabla^2 L(\btheta)}}{h (1 - \beta)} \norm[\bigg]{- \frac{h}{1 - \beta} \nabla L(\btheta(\bzeta)) + h^2 \beta \boldsymbol{g}(\btheta(\bzeta))} + \beta \norm{\boldsymbol{g}(\btheta(\bzeta))}\\
                                  &\leq \frac{\sup_{\btheta} \norm{\nabla^2 L(\btheta)} \sup_{\btheta} \norm{\nabla L(\btheta)}}{(1 - \beta)^2} + \prn[\bigg]{\beta + h \frac{\beta}{1 - \beta} \sup_{\btheta} \norm{\nabla^2 L(\btheta)}} \gamma\\
                                  &\leq \frac{D_1 D_2}{(1 - \beta)^2} + \prn[\bigg]{\beta + h \frac{\beta}{1 - \beta} D_2} \gamma\\
  &\leq \gamma,\numberthis\label{eq:ldlkjfhh}
\end{align*}
where the last inequality holds for all small enough $h$ because
\begin{equation*}
\frac{D_1 D_2}{(1 - \beta)^2} + \prn[\bigg]{\beta + h \frac{\beta}{1 - \beta} D_2} \gamma \underset{h \to 0}{\longrightarrow} D_1D_2/(1-\beta)^2 + \beta\gamma = \gamma-(1-\beta) < \gamma.
\end{equation*}

\paragraph{Bounding the derivative of $T \boldsymbol{g}$}
Consider the expression ending in~\eqref{eq:derivative-of-t}. We will use
\begin{align*}
  &\norm{\nabla_{\bzeta} T \boldsymbol{g}(\bzeta)} \leq \norm[\bigg]{\frac{\nabla^2 L(\bzeta) - \nabla^2 L(\btheta(\bzeta))}{h (1 - \beta)}}\\
  &\qquad + \norm[\bigg]{\frac{1}{h (1 - \beta)} \nabla^2 L(\btheta(\bzeta)) \crl[\bigg]{I - \prn[\bigg]{I - \frac{h}{1 - \beta} \nabla^2 L(\btheta(\bzeta)) + h^2 \beta \nabla \boldsymbol{g}(\btheta(\bzeta))}^{-1}}}\\
  &\qquad + \norm[\bigg]{\beta \nabla \boldsymbol{g}(\btheta(\bzeta)) \prn[\bigg]{I - \frac{h}{1 - \beta} \nabla^2 L(\btheta(\bzeta)) + h^2 \beta \nabla \boldsymbol{g}(\btheta(\bzeta))}^{-1}}
\end{align*}
and bound each of the three terms in the right-hand side separately.

The first term is easy to bound: \internalComment{$D_3$ is defined so that this is true}
\begin{align*}
  &\norm[\bigg]{\frac{\nabla^2 L(\bzeta) - \nabla^2 L(\btheta(\bzeta))}{h (1 - \beta)}}
    \leq \frac{D_3}{h (1 - \beta)} \norm{\bzeta - \btheta(\bzeta)}\\
  &\quad = \frac{D_3}{h (1 - \beta)} \norm[\bigg]{- \frac{h}{1 - \beta} \nabla L(\btheta(\bzeta)) + h^2 \beta \boldsymbol{g}(\btheta(\bzeta))}\\
  &\quad \leq \frac{D_3}{h (1 - \beta)} \crl[\bigg]{\frac{h D_1}{1 - \beta} + h^2 \beta \gamma} = \frac{D_3 \kf}{1 - \beta}.
\end{align*}

Now we bound
\begin{align*}
  &\norm[\bigg]{\frac{1}{h (1 - \beta)} \nabla^2 L(\btheta(\bzeta)) \crl[\bigg]{I - \prn[\bigg]{I - \frac{h}{1 - \beta} \nabla^2 L(\btheta(\bzeta)) + h^2 \beta \nabla \boldsymbol{g}(\btheta(\bzeta))}^{-1}}}\\
  &\quad \leq \frac{D_2}{h (1 - \beta)} \norm[\bigg]{I - \prn[\bigg]{I - \frac{h}{1 - \beta} \nabla^2 L(\btheta(\bzeta)) + h^2 \beta \nabla \boldsymbol{g}(\btheta(\bzeta))}^{-1}}.
\end{align*}
Notice that the eigenvalues of the matrix
\begin{equation*}
I - \prn[\bigg]{I - \frac{h}{1 - \beta} \nabla^2 L(\btheta(\bzeta)) + h^2 \beta \nabla \boldsymbol{g}(\btheta(\bzeta))}^{-1}
\end{equation*}
are of the form $- h \lambda_i / (1 - h \lambda_i)$ where $\crl{\lambda_i}$ are eigenvalues of the matrix
$(1 - \beta)^{-1} \nabla^2 L(\btheta(\bzeta)) - h \beta \nabla \boldsymbol{g}(\btheta(\bzeta))$.
For small enough $h$, we have
\begin{equation*}
  h \max_i \abs{\lambda_i} \leq h \crl{(1 - \beta)^{-1} D_2 + h \beta \delta} < 1,
\end{equation*}
so we can bound
\begin{align*}
  \frac{\abs{\lambda_i}}{\abs{1 - h \lambda_i}} &\leq \frac{\abs{\lambda_i}}{1 - h \abs{\lambda_i}} \leq \frac{\max_i \abs{\lambda_i}}{1 - h \max_i \abs{\lambda_i}} = \frac{\norm{(1 - \beta)^{-1} \nabla^2 L(\btheta(\bzeta)) - h \beta \nabla \boldsymbol{g}(\btheta(\bzeta))}}{1 - h \norm{(1 - \beta)^{-1} \nabla^2 L(\btheta(\bzeta)) - h \beta \nabla \boldsymbol{g}(\btheta(\bzeta))}}\\
  &\leq \frac{(1 - \beta)^{-1} D_2 + h \beta \delta}{1 - h \crl{(1 - \beta)^{-1} D_2 + h \beta \delta}}.
\end{align*}

Similarly,
\begin{align*}
\norm[\bigg]{\beta \nabla \boldsymbol{g}(\btheta(\bzeta)) \prn[\bigg]{I - \frac{h}{1 - \beta} \nabla^2 L(\btheta(\bzeta)) + h^2 \beta \nabla \boldsymbol{g}(\btheta(\bzeta))}^{-1}} \leq \frac{\beta \delta}{1 - h \crl{(1 - \beta)^{-1} D_2 + h \beta \delta}}.
\end{align*}

We conclude that
\begin{align*}
  &\norm{\nabla_{\bzeta} T \boldsymbol{g}(\bzeta)} \leq \frac{D_3 \kf}{1 - \beta} + \frac{D_2}{1 - \beta} \frac{(1 - \beta)^{-1} D_2 + h \beta \delta}{1 - h \crl{(1 - \beta)^{-1} D_2 + h \beta \delta}} + \frac{\beta \delta}{1 - h \crl{(1 - \beta)^{-1} D_2 + h \beta \delta}}\\
  &\quad = \frac{D_3 \kf \prn[\big]{1 - h \crl[\big]{\frac{D_2}{1 - \beta} + h \beta \delta}} + D_2 \prn[\big]{\frac{D_2}{1 - \beta} + h \beta \delta} + \beta \delta (1 - \beta)}{(1 - \beta) [1 - h \crl{(1 - \beta)^{-1} D_2 + h \beta \delta}]}\\
  &\quad = \frac{(D_2 - h D_3 \kf) \prn[\big]{\frac{D_2}{1 - \beta} + h \beta \delta} + D_3 \kf + \beta \delta (1 - \beta)}{(1 - \beta) [1 - h \crl{(1 - \beta)^{-1} D_2 + h \beta \delta}]}.
\end{align*}

The inequality
\begin{equation}\label{eq:constaint-on-delta}
\frac{(D_2 - h D_3 \kf) \prn[\big]{\frac{D_2}{1 - \beta} + h \beta \delta} + D_3 \kf + \beta \delta (1 - \beta)}{(1 - \beta) [1 - h \crl{(1 - \beta)^{-1} D_2 + h \beta \delta}]} \leq \delta
\end{equation}
can be rewritten as
\begin{align*}
h^2 \beta \delta^2 + \prn[\bigg]{\beta -1 + h \frac{1 + \beta}{1 - \beta} D_2 - \frac{h^2 \beta}{1 - \beta} \kf D_3} \delta + \frac{D_2 (D_2 - h D_3 \kf)}{(1 - \beta)^2} + \frac{D_3 \kf}{1 - \beta} \leq 0.
\end{align*}
This inequality holds for sufficiently small $h$ because the left-hand side converges to
\begin{equation*}
  -(1 - \beta)\delta
  + \frac{D_2^2 + D_1D_3}{(1 - \beta)^2}
  = -(1 - \beta) < 0
\end{equation*}
as $h \to 0$.

\paragraph{Bounding the Lipschitz constant of the derivative of $T \boldsymbol{g}$}
Consider the expression ending in~\eqref{eq:derivative-of-t}.
To declutter notation, denote $A_i := (1 - \beta)^{-1} \nabla^2 L(\btheta(\bzeta_i)) - h \beta \nabla \boldsymbol{g}(\btheta(\bzeta_i))$ for $i \in \crl{1, 2}$.
First, decompose
\begin{align*}
  &\nabla_{\bzeta} T \boldsymbol{g}(\bzeta_1) - \nabla_{\bzeta} T \boldsymbol{g}(\bzeta_2) = \frac{\nabla^2 L(\bzeta_1)}{h (1 - \beta)} - \prn[\bigg]{\frac{\nabla^2 L(\btheta(\bzeta_1))}{h (1 - \beta)} - \beta \nabla \boldsymbol{g}(\btheta(\bzeta_1))} \prn{I - h A_1}^{-1}\\
  &\quad - \frac{\nabla^2 L(\bzeta_2)}{h (1 - \beta)} + \prn[\bigg]{\frac{\nabla^2 L(\btheta(\bzeta_2))}{h (1 - \beta)} - \beta \nabla \boldsymbol{g}(\btheta(\bzeta_2))} \prn{I - h A_2}^{-1}\\
  &\quad = \frac{\nabla^2 L(\bzeta_1) - \nabla^2 L(\bzeta_2)}{h (1 - \beta)}\\
  &\qquad - \prn[\bigg]{\frac{\nabla^2 L(\btheta(\bzeta_1)) - \nabla^2 L(\btheta(\bzeta_2))}{h (1 - \beta)} - \beta \nabla \boldsymbol{g}(\btheta(\bzeta_1)) + \beta \nabla \boldsymbol{g}(\btheta(\bzeta_2))} \prn{I - h A_1}^{-1}\\
  &\qquad - \prn[\bigg]{\frac{\nabla^2 L(\btheta(\bzeta_2))}{h (1 - \beta)} - \beta \nabla \boldsymbol{g}(\btheta(\bzeta_2))} \crl[\big]{\prn{I - h A_1}^{-1} - \prn{I - h A_2}^{-1}}\\
  &\quad = T_1 - T_2 - T_3,
\end{align*}
where
\begin{align*}
  &T_1 := \frac{\nabla^2 L(\bzeta_1) - \nabla^2 L(\bzeta_2)}{h (1 - \beta)} \crl[\big]{I - (I - h A_1)^{-1}},\\
  &T_2 := \prn[\bigg]{\frac{\nabla^2 L(\btheta(\bzeta_1)) - \nabla^2 L(\bzeta_1) - \nabla^2 L(\btheta(\bzeta_2)) + \nabla^2 L(\bzeta_2)}{h (1 - \beta)} - \beta \nabla \boldsymbol{g}(\btheta(\bzeta_1)) + \beta \nabla \boldsymbol{g}(\btheta(\bzeta_2))} \times\\
  &\qquad \times \prn{I - h A_1}^{-1},\\
  &T_3 := \prn[\bigg]{\frac{\nabla^2 L(\btheta(\bzeta_2))}{h (1 - \beta)} - \beta \nabla \boldsymbol{g}(\btheta(\bzeta_2))} \crl[\big]{\prn{I - h A_1}^{-1} - \prn{I - h A_2}^{-1}}.
\end{align*}

To bound $T_1$, use $\norm{I - (I - h A_1)^{-1}} \leq h \frac{(1 - \beta)^{-1} D_2 + h \beta \delta}{1 - h \crl{(1 - \beta)^{-1} D_2 + h \beta \delta}}$ and conclude
\begin{align*}
&\norm{T_1} \leq \frac{D_3 (1 - \beta)^{-1 } \brk{(1 - \beta)^{-1} D_2 + h \beta \delta}}{1 - h \crl{(1 - \beta)^{-1} D_2 + h \beta \delta}} \norm{\bzeta_1 - \bzeta_2}.
\end{align*}

To bound $T_2$, note that
\begin{align*}
  &\nabla_{\bzeta} \crl[\big]{\nabla^2 L(\btheta(\bzeta)) - \nabla^2 L(\bzeta)} = \nabla^3 L(\btheta(\bzeta)) (I - h A)^{-1} - \nabla^3 L(\bzeta)\\
  &\quad = \prn[\big]{\nabla^3 L(\btheta(\bzeta)) - \nabla^3 L(\bzeta)} (I - h A)^{-1} + \nabla^3 L(\bzeta) \prn[\big]{(I - h A)^{-1} - I}
\end{align*}
where $A := (1 - \beta)^{-1} \nabla^2 L(\btheta(\bzeta)) - h \beta \nabla \boldsymbol{g}(\btheta(\bzeta))$.
Therefore,
\begin{align*}
  &\norm{\nabla^2 L(\btheta(\bzeta_1)) - \nabla^2 L(\bzeta_1) - \nabla^2 L(\btheta(\bzeta_2)) + \nabla^2 L(\bzeta_2)}\\
  &\quad\leq \sup_{\bzeta} \norm[\big]{\nabla_{\bzeta} \crl[\big]{\nabla^2 L(\btheta(\bzeta)) - \nabla^2 L(\bzeta)}} \norm{\bzeta_1 - \bzeta_2}\\
  &\quad \leq \frac{D_4 \sup_{\bzeta} \norm{\btheta(\bzeta) - \bzeta}}{1 - h \crl{(1 - \beta)^{-1} D_2 + h \beta \delta}} \norm{\bzeta_1 - \bzeta_2} + h \frac{D_3 \crl{(1 - \beta)^{-1} D_2 + h \beta \delta}}{1 - h \crl{(1 - \beta)^{-1} D_2 + h \beta \delta}} \norm{\bzeta_1 - \bzeta_2}\\
  &\quad \leq h \frac{D_4 \kf + D_3 \crl{(1 - \beta)^{-1} D_2 + h \beta \delta}}{1 - h \crl{(1 - \beta)^{-1} D_2 + h \beta \delta}} \norm{\bzeta_1 - \bzeta_2},
\end{align*}
where the last inequality is obtained by bounding $\sup_{\bzeta} \norm{\btheta(\bzeta) - \bzeta}$ like for \cref{eq:ldlkjfhh}. Since also
\begin{align*}
  \beta \norm{\nabla \boldsymbol{g}(\btheta(\bzeta_1)) - \nabla \boldsymbol{g}(\btheta(\bzeta_2))} &\leq \beta \lambda \norm{\btheta(\bzeta_1) - \btheta(\bzeta_2)} \leq \beta \lambda \sup_{\bzeta} \norm{\nabla_{\bzeta} \btheta(\bzeta)} \norm{\bzeta_1 - \bzeta_2}\\
  &\leq \frac{\beta \lambda}{1 - h \crl{(1 - \beta)^{-1} D_2 + h \beta \delta}} \norm{\bzeta_1 - \bzeta_2},
\end{align*}
we conclude
\begin{align*}
\norm{T_2} \leq \frac{(1 - \beta)^{-1} \brk{D_4 \kf + D_3 \crl{(1 - \beta)^{-1} D_2 + h \beta \delta}} + \beta \lambda}{\prn[\big]{1 - h \crl{(1 - \beta)^{-1} D_2 + h \beta \delta}}^2} \norm{\bzeta_1 - \bzeta_2}.
\end{align*}

To bound $T_3$, use
\begin{align*}
  &\norm{(I - h A_1)^{-1} - (I - h A_2)^{-1}} = h \norm{(I - h A_1)^{-1} \brk{A_1 - A_2} (I - h A_2)^{-1}}\\
  &\quad \leq \frac{h}{\prn[\big]{1 - h \crl{(1 - \beta)^{-1} D_2 + h \beta \delta}}^2}\\
  &\qquad \times \norm[\bigg]{\frac{\nabla^2 L(\btheta(\bzeta_1)) - \nabla^2 L(\btheta(\bzeta_2))}{1 - \beta} - h \beta \crl[\big]{\nabla \boldsymbol{g}(\btheta(\bzeta_1)) - \nabla \boldsymbol{g}(\btheta(\bzeta_2))}}\\
  &\quad \leq \frac{h}{\prn[\big]{1 - h \crl{(1 - \beta)^{-1} D_2 + h \beta \delta}}^2} \brk[\bigg]{\frac{D_3}{1 - \beta} + h \beta \lambda} \norm{\btheta(\bzeta_1) - \btheta(\bzeta_2)}\\
  &\quad \leq \frac{h \brk[\big]{(1 - \beta)^{-1} D_3 + h \beta \lambda}}{\prn[\big]{1 - h \crl{(1 - \beta)^{-1} D_2 + h \beta \delta}}^3} \norm{\bzeta_1 - \bzeta_2},
\end{align*}
and conclude
\begin{equation*}
\norm{T_3} \leq \frac{\prn[\big]{(1 - \beta)^{-1} D_2 + h \beta \delta} \brk[\big]{(1 - \beta)^{-1} D_3 + h \beta \lambda}}{\prn[\big]{1 - h \crl{(1 - \beta)^{-1} D_2 + h \beta \delta}}^3} \norm{\bzeta_1 - \bzeta_2}.
\end{equation*}

Combining, we get
\begin{align*}
  &\norm{\nabla_{\bzeta} T \boldsymbol{g}(\bzeta_1) - \nabla_{\bzeta} T \boldsymbol{g}(\bzeta_2)}\\
  &\quad \leq \prn[\bigg]{\frac{D_3 (1 - \beta)^{-1 } \brk{(1 - \beta)^{-1} D_2 + h \beta \delta}}{1 - h \crl{(1 - \beta)^{-1} D_2 + h \beta \delta}}\\
  &\qquad + \frac{(1 - \beta)^{-1} \brk{D_4 \kf + D_3 \crl{(1 - \beta)^{-1} D_2 + h \beta \delta}} + \beta \lambda}{\prn[\big]{1 - h \crl{(1 - \beta)^{-1} D_2 + h \beta \delta}}^2}\\
  &\qquad + \frac{\prn[\big]{(1 - \beta)^{-1} D_2 + h \beta \delta} \brk[\big]{(1 - \beta)^{-1} D_3 + h \beta \lambda}}{\prn[\big]{1 - h \crl{(1 - \beta)^{-1} D_2 + h \beta \delta}}^3}} \norm{\bzeta_1 - \bzeta_2}.
\end{align*}

It is left to note that
\begin{align*}
&\frac{D_3 (1 - \beta)^{-1 } \brk{(1 - \beta)^{-1} D_2 + h \beta \delta}}{1 - h \crl{(1 - \beta)^{-1} D_2 + h \beta \delta}}\\
  &\qquad + \frac{(1 - \beta)^{-1} \brk{D_4 \kf + D_3 \crl{(1 - \beta)^{-1} D_2 + h \beta \delta}} + \beta \lambda}{\prn[\big]{1 - h \crl{(1 - \beta)^{-1} D_2 + h \beta \delta}}^2}\\
  &\qquad + \frac{\prn[\big]{(1 - \beta)^{-1} D_2 + h \beta \delta} \brk[\big]{(1 - \beta)^{-1} D_3 + h \beta \lambda}}{\prn[\big]{1 - h \crl{(1 - \beta)^{-1} D_2 + h \beta \delta}}^3} \leq \lambda
\end{align*}
for all sufficiently small $h$ since the left-hand side converges to
\begin{equation*}
  \frac{3 D_2 D_3 + D_1 D_4}{(1 - \beta)^2} + \beta \lambda
  = \lambda - (1 - \beta) < \lambda
\end{equation*}
as $h \to 0$.
\end{proof}

\begin{lemma}[$T$ is a contraction]\label{lem:contraction-of-t}
For any $\boldsymbol{g}_1, \boldsymbol{g}_2 \in \Gamma$, we have
\begin{equation*}
\norm{T \boldsymbol{g}_1 - T \boldsymbol{g}_2}_{\Gamma} \leq \alpha \norm{\boldsymbol{g}_1 - \boldsymbol{g}_2}_{\Gamma}
\end{equation*}
with some $\alpha < 1$.
\end{lemma}

\begin{proof}
It is proven in \cite[Lemma 15]{kovachki2021continuous} that for $h$ small enough $h \ks < 1$ and
\begin{equation}\label{eq:jjfhhbc}
\sup_{\bzeta \in \mathbb{R}^d}\norm{T \boldsymbol{g}_1(\bzeta) - T \boldsymbol{g}_2(\bzeta)} \leq \alpha_1 \sup_{\btheta \in \mathbb{R}^d} \norm{\boldsymbol{g}_1(\btheta) - \boldsymbol{g}_2(\btheta)}
\end{equation}
with
\begin{equation*}
  \alpha_1 := \beta + h \beta \frac{D_2 + h D_3 \kf}{1 - \beta} + \frac{h^2 \beta}{1 - h \ks} \prn[\Big]{\beta \delta + \frac{(D_2 + h D_3 \kf)(D_2 / (1 - \beta) + h \beta \delta) + D_3 \kf}{1 - \beta}}.
\end{equation*}

Next, fix $\bzeta \in \mathbb{R}^d$, $\boldsymbol{g}_1, \boldsymbol{g}_2 \in \Gamma$. Let $\btheta_1$ be the preimage of $\bzeta$ under the mapping $\btheta \mapsto \bzeta_1(\btheta) = \btheta - h (1 - \beta)^{-1} \nabla L(\btheta) + h^2 \beta \boldsymbol{g}_1(\btheta)$, and $\btheta_2$ the preimage of $\bzeta$ under the mapping $\btheta \mapsto \bzeta_2(\btheta) = \btheta - h (1 - \beta)^{-1} \nabla L(\btheta) + h^2 \beta \boldsymbol{g}_2(\btheta)$. Recall that the form of the derivative of $T \boldsymbol{g}_i(\bzeta)$ is given in~\eqref{eq:derivative-of-t}. Decompose
\begin{align*}
  &\nabla_{\bzeta} T \boldsymbol{g}_1(\bzeta) - \nabla_{\bzeta} T \boldsymbol{g}_2(\bzeta)\\
  &\quad = \prn[\bigg]{\frac{\nabla^2 L(\btheta_2)}{h (1 - \beta)} - \beta \nabla \boldsymbol{g}_2(\btheta_2)} (I - h A_2)^{-1} - \prn[\bigg]{\frac{\nabla^2 L(\btheta_1)}{h (1 - \beta)} - \beta \nabla \boldsymbol{g}_1(\btheta_1)} (I - h A_1)^{-1}\\
  &\quad = \underbrace{\brk[\bigg]{\frac{\nabla^2 L(\btheta_2) - \nabla^2 L(\btheta_1)}{h (1 - \beta)} - \beta \crl[\big]{\nabla \boldsymbol{g}_2(\btheta_2) - \nabla \boldsymbol{g}_1(\btheta_1)}} (I - h A_2)^{-1}}_{T_1}\\
  &\qquad + \underbrace{\brk[\bigg]{\frac{\nabla^2 L(\btheta_1)}{h (1 - \beta)} - \beta \nabla \boldsymbol{g}_1(\btheta_1)} \prn[\big]{(I - h A_2)^{-1} - (I - h A_1)^{-1}}}_{T_2},
\end{align*}
where $A_i = (1 - \beta)^{-1} \nabla^2 L(\btheta_i) - h \beta \nabla \boldsymbol{g}_i(\btheta_i)$, $i \in \crl{1, 2}$.

To bound $T_1$, first note
\begin{equation}\label{eq:lkjfhh}
  \norm[\bigg]{\frac{\nabla^2 L(\btheta_1) - \nabla^2 L(\btheta_2)}{h (1 - \beta)}} \leq \frac{D_3}{h (1 - \beta)} \norm{\btheta_1 - \btheta_2}.
\end{equation}
Since
\begin{equation*}
\bzeta = \btheta_1 - \frac{h}{1 - \beta} \nabla L(\btheta_1) + h^2 \beta \boldsymbol{g}_1(\btheta_1) = \btheta_2 - \frac{h}{1 - \beta} \nabla L(\btheta_2) + h^2 \beta \boldsymbol{g}_2(\btheta_2),
\end{equation*}
we have
\begin{align*}
  \norm{\btheta_1 - \btheta_2} &\leq \frac{h}{1 - \beta} \norm{\nabla L(\btheta_1) - \nabla L(\btheta_2)} + h^2 \beta \norm{\boldsymbol{g}_1(\btheta_1) - \boldsymbol{g}_2(\btheta_2)}\\
                               &\leq \frac{h D_2}{1 - \beta} \norm{\btheta_1 - \btheta_2} + h^2 \beta \norm{\boldsymbol{g}_1(\btheta_1) - \boldsymbol{g}_1(\btheta_2)} + h^2 \beta \norm{\boldsymbol{g}_1(\btheta_2) - \boldsymbol{g}_2(\btheta_2)}\\
  &\leq \prn[\bigg]{\frac{h D_2}{1 - \beta} + h^2 \beta \delta} \norm{\btheta_1 - \btheta_2} + h^2 \beta \norm{\boldsymbol{g}_1(\btheta_2) - \boldsymbol{g}_2(\btheta_2)}.
\end{align*}
For $h$ small enough, $\frac{h D_2}{1 - \beta} + h^2 \beta \delta < 1$, and we obtain
\begin{equation}\label{eq:jfhhhfb}
\norm{\btheta_1 - \btheta_2} \leq \frac{h^2 \beta}{1 - \frac{h D_2}{1 - \beta} - h^2 \beta \delta} \norm{\boldsymbol{g}_1(\btheta_2) - \boldsymbol{g}_2(\btheta_2)}.
\end{equation}
Continuing~\eqref{eq:lkjfhh}, we get
\begin{align*}
  \norm[\bigg]{\frac{\nabla^2 L(\btheta_1) - \nabla^2 L(\btheta_2)}{h (1 - \beta)}} &\leq \frac{D_3}{h (1 - \beta)} \frac{h^2 \beta}{1 - \frac{h D_2}{1 - \beta} - h^2 \beta \delta} \sup_{\btheta} \norm{\boldsymbol{g}_1(\btheta) - \boldsymbol{g}_2(\btheta)}\\
  &= \frac{(1 - \beta)^{-1} h \beta D_3}{1 - \frac{h D_2}{1 - \beta} - h^2 \beta \delta} \sup_{\btheta} \norm{\boldsymbol{g}_1(\btheta) - \boldsymbol{g}_2(\btheta)}.
\end{align*}
Next, use
\begin{align*}
  \norm{\nabla \boldsymbol{g}_1(\btheta_1) - \nabla \boldsymbol{g}_2(\btheta_2)} &\leq \norm{\nabla \boldsymbol{g}_1(\btheta_1) - \nabla \boldsymbol{g}_1(\btheta_2)} + \norm{\nabla \boldsymbol{g}_1(\btheta_2) - \nabla \boldsymbol{g}_2(\btheta_2)}\\
                                                                       &\leq \lambda \norm{\btheta_1 - \btheta_2} + \sup_{\btheta} \norm{\nabla \boldsymbol{g}_1(\btheta) - \nabla \boldsymbol{g}_2(\btheta)}\\
  &\leq \frac{h^2 \beta \lambda}{1 - \frac{h D_2}{1 - \beta} - h^2 \beta \delta} \sup_{\btheta} \norm{\boldsymbol{g}_1(\btheta) - \boldsymbol{g}_2(\btheta)} + \sup_{\btheta} \norm{\nabla \boldsymbol{g}_1(\btheta) - \nabla \boldsymbol{g}_2(\btheta)}
\end{align*}
with the last inequality by \cref{eq:jfhhhfb}. As previously, using $\norm{(I - h A_1)^{-1}} \leq \brk{1 - h \crl{(1 - \beta)^{-1} D_2 + h \beta \delta}}^{-1}$, we can now conclude
\begin{align*}
  \norm{T_1} &\leq \prn[\bigg]{\frac{(1 - \beta)^{-1} h \beta D_3}{\prn[\big]{1 - \frac{h D_2}{1 - \beta} - h^2 \beta \delta}^2} + \frac{h^2 \beta^2 \lambda}{\prn[\big]{1 - \frac{h D_2}{1 - \beta} - h^2 \beta \delta}^2}} \sup_{\btheta} \norm{\boldsymbol{g}_1(\btheta) - \boldsymbol{g}_2(\btheta)}\\
  &\quad + \frac{\beta}{1 - \frac{h D_2}{1 - \beta} - h^2 \beta \delta} \sup_{\btheta} \norm{\nabla \boldsymbol{g}_1(\btheta) - \nabla \boldsymbol{g}_2(\btheta)}.
\end{align*}

To bound $T_2$, use
\begin{align*}
  &\norm[\big]{(I - h A_2)^{-1} - (I - h A_1)^{-1}} = h \norm{(I - h A_2)^{-1} \brk{A_2 - A_1} (I - h A_1)^{-1}}\\
  &\quad \leq \frac{h}{\prn[\big]{1 - h \crl{(1 - \beta)^{-1} D_2 + h \beta \delta}}^2} \norm{A_2 - A_1}\\
  &\quad = \frac{h}{\prn[\big]{1 - h \crl{(1 - \beta)^{-1} D_2 + h \beta \delta}}^2} \norm[\bigg]{\frac{\nabla^2 L(\btheta_2) - \nabla^2 L(\btheta_1)}{1 - \beta} - h \beta \crl[\big]{\nabla \boldsymbol{g}_2(\btheta_2) - \nabla \boldsymbol{g}_1(\btheta_1)}}\\
  &\quad \leq \frac{h}{\prn[\big]{1 - h \crl{(1 - \beta)^{-1} D_2 + h \beta \delta}}^2} \brk[\bigg]{\frac{(1 - \beta)^{-1} h^2 \beta D_3 + h^3 \beta^2 \lambda}{1 - \frac{h D_2}{1 - \beta} - h^2 \beta \delta} \sup_{\btheta} \norm{\boldsymbol{g}_1(\btheta) - \boldsymbol{g}_2(\btheta)}\\
  &\quad + h \beta \sup_{\btheta} \norm{\nabla \boldsymbol{g}_1(\btheta) - \nabla \boldsymbol{g}_2(\btheta)}}\\
  &\quad \leq \frac{(1 - \beta)^{-1} h^3 \beta D_3 + h^4 \beta^2 \lambda}{\prn[\big]{1 - h \crl{(1 - \beta)^{-1} D_2 + h \beta \delta}}^3} \sup_{\btheta} \norm{\boldsymbol{g}_1(\btheta) - \boldsymbol{g}_2(\btheta)}\\
  &\quad + \frac{h^2 \beta}{\prn[\big]{1 - h \crl{(1 - \beta)^{-1} D_2 + h \beta \delta}}^2} \sup_{\btheta} \norm{\nabla \boldsymbol{g}_1(\btheta) - \nabla \boldsymbol{g}_2(\btheta)},
\end{align*}
and conclude
\begin{align*}
  \norm{T_2} &\leq \frac{\prn[\big]{(1 - \beta)^{-1} D_2 + h \beta \delta} \brk[\big]{(1 - \beta)^{-1} h^2 \beta D_3 + h^3 \beta^2 \lambda}}{\prn[\big]{1 - h \crl{(1 - \beta)^{-1} D_2 + h \beta \delta}}^3} \sup_{\btheta} \norm{\boldsymbol{g}_1(\btheta) - \boldsymbol{g}_2(\btheta)}\\
  &\quad + \frac{\prn[\big]{(1 - \beta)^{-1} D_2 + h \beta \delta} h \beta}{\prn[\big]{1 - h \crl{(1 - \beta)^{-1} D_2 + h \beta \delta}}^2} \sup_{\btheta} \norm{\nabla \boldsymbol{g}_1(\btheta) - \nabla \boldsymbol{g}_2(\btheta)}.
\end{align*}

Combining the bounds on $T_1$ and $T_2$ gives
\begin{equation*}
\norm[\big]{\nabla_{\bzeta} T \boldsymbol{g}_1(\bzeta) - \nabla_{\bzeta} T \boldsymbol{g}_2(\bzeta)} \leq \tilde{\alpha}_1 \sup_{\btheta} \norm{\boldsymbol{g}_1(\btheta) - \boldsymbol{g}_2(\btheta)} + \tilde{\alpha}_2 \sup_{\btheta} \norm{\nabla \boldsymbol{g}_1(\btheta) - \nabla \boldsymbol{g}_2(\btheta)},
\end{equation*}
where
\begin{align*}
  &\tilde{\alpha}_1 := \frac{(1 - \beta)^{-1} h \beta D_3}{\prn[\big]{1 - h \crl{(1 - \beta)^{-1} D_2 + h \beta \delta}}^2} + \frac{h^2 \beta^2 \lambda}{\prn[\big]{1 - h \crl{(1 - \beta)^{-1} D_2 + h \beta \delta}}^2}\\
  &\qquad + \frac{\prn[\big]{(1 - \beta)^{-1} D_2 + h \beta \delta} \brk[\big]{(1 - \beta)^{-1} h^2 \beta D_3 + h^3 \beta^2 \lambda}}{\prn[\big]{1 - h \crl{(1 - \beta)^{-1} D_2 + h \beta \delta}}^3},\\
  &\quad = \frac{\beta h \crl{D_3 / (1 - \beta) + h \beta \lambda}}{(1 - h \ks)^3},\\
  &\tilde{\alpha}_2 := \frac{\beta}{1 - h \crl{(1 - \beta)^{-1} D_2 + h \beta \delta}} + \frac{\prn[\big]{(1 - \beta)^{-1} D_2 + h \beta \delta} h \beta}{\prn[\big]{1 - h \crl{(1 - \beta)^{-1} D_2 + h \beta \delta}}^2} = \frac{\beta}{(1 - h \ks)^2}.
\end{align*}

Combining with \eqref{eq:jjfhhbc}, we get
\begin{align*}
  &\sup_{\bzeta} \norm{T \boldsymbol{g}_1(\bzeta) - T \boldsymbol{g}_2(\bzeta)} + \sup_{\bzeta} \norm[\big]{\nabla_{\bzeta} T \boldsymbol{g}_1(\bzeta) - \nabla_{\bzeta} T \boldsymbol{g}_2(\bzeta)}\\
  &\quad \leq (\alpha_1 + \tilde{\alpha}_1) \sup_{\btheta} \norm{\boldsymbol{g}_1(\btheta) - \boldsymbol{g}_2(\btheta)} + \tilde{\alpha}_2 \sup_{\btheta} \norm{\nabla \boldsymbol{g}_1(\btheta) - \nabla \boldsymbol{g}_2(\btheta)}.
\end{align*}
Since as $h \to 0$
\begin{equation*}
  \alpha_1 \to \beta,
  \quad
  \tilde{\alpha}_1 \to 0,
  \quad
  \tilde{\alpha}_2 \to \beta,
\end{equation*}
we have for small enough $h$ that
\begin{equation*}
  \alpha := (\alpha_1 + \tilde{\alpha}_1) \vee \tilde{\alpha}_2
  \leq \frac{1 + \beta}{2} < 1.\qedhere
\end{equation*}
\end{proof}

\begin{lemma}[Exponential attractivity]\label{lem:exponential-attractivity}
  \Cref{eq:hb} satisfies
  \begin{equation*}
    \norm[\bigg]{\bv^{(n)} + \frac{1}{1 - \beta} \nabla L(\btheta^{(n)}) - h \bg_h(\btheta^{(n)})} \leq \rho_h^n \norm[\bigg]{\bv^{(0)} + \frac{1}{1 - \beta} \nabla L(\btheta^{(0)}) - h \bg_h(\btheta^{(0)})},
  \end{equation*}
  where $\rho_h := \beta \prn{1 + h D_2 / (1 - \beta) + h^2 \delta} < 1$ for all sufficiently small $h$.
\end{lemma}

\begin{proof}
Define
\begin{equation*}
\bV_h(\btheta) := - \frac{1}{1 - \beta} \nabla L(\btheta) + h \bg_h(\btheta),
\end{equation*}
so that $(\btheta^{(n)}, \bV_h(\btheta^{(n)}))$ lies on the invariant manifold. Define the variables obtained by applying the HB update from this point:
\begin{equation*}
  \widehat{\bv}^{(n + 1)} := \beta \bV_h(\btheta^{(n)}) - \nabla L(\btheta^{(n)}),
  \quad
  \widehat{\btheta}^{(n + 1)} := \btheta^{(n)} + h \widehat{\bv}^{(n + 1)}.
\end{equation*}
They also lie on the invariant manifold, that is, $\widehat{\bv}^{(n + 1)} = \bV_h(\widehat{\btheta}^{(n + 1)})$. Moreover, the ordinary variables are expressed through the hatted variables as
\begin{equation*}
  \bv^{(n + 1)} = \widehat{\bv}^{(n + 1)} + \beta \be^{(n)},
  \quad
  \btheta^{(n + 1)} = \widehat{\btheta}^{(n + 1)} + h \beta \be^{(n)},
\end{equation*}
with the error
\begin{equation*}
\be^{(n)} := \bv^{(n)} - \bV_h(\btheta^{(n)}).
\end{equation*}
Therefore,
\begin{align*}
  \norm{\be^{(n + 1)}}
  &\leq \beta \norm{\be^{(n)}} + \norm{\bV_h(\widehat{\btheta}^{(n + 1)}) - \bV_h(\widehat{\btheta}^{(n + 1)} + h \beta \be^{(n)})}
  \\
  &\leq \beta \prn[\bigg]{1 + \frac{h D_2}{1 - \beta} + h^2 \delta} \norm{\be^{(n)}},
\end{align*}
where we used that $\bV_h$ is Lipschitz with constant $D_2 / (1 - \beta) + h \delta$.
\end{proof}

\section{Proof of \cref{th:approx}}\label{sec:proof-of-approx-th}

In the proof of this theorem, $O(\cdot)$ will denote a term that is bounded (in absolute value or in norm) by a constant times the argument, where the constant cannot depend on $n$ or $h$ but can depend on other fixed values (like $\ord$, $T$, $\beta$).

\begin{lemma}\label{lem:liuyf}
  For all $\btheta \in \mathcal{D}$ and all $m \in \range{1}{\ord}$, we have
\begin{equation}\label{eq:nomem-coef-bounds}
\nabla^{r} \nomemcoef{m}{n}(\btheta) = O(1), \quad \nabla^{r} \histcoef{m}{n}{a}(\btheta) = O(a^m),\quad r \in \range{0}{2 \ord - m}.
\end{equation}
\end{lemma}

\begin{proof}
Note that if $\histcoef{l + 1}{n}{s}(\btheta)$ actually appears in the sum on the right of \eqref{eq:jfjl}, then $l + 1$ cannot exceed $m - 1$, because $l + 1 \geq m$ would mean $i \leq 0$ and this term is actually absent (there are no history terms if $i = 0$). Therefore, $\histcoef{m}{n}{a}(\btheta)$ is defined by $\crl[\big]{\histcoef{j}{n}{s}(\btheta)}_{1 \leq j \leq m - 1, 1 \leq s \leq a}$ and $\crl[\big]{\nomemcoef{j}{n - s}(\btheta)}_{1 \leq j \leq m}$. By the same logic, from~\eqref{eq:nomem-coef-def}, $\nomemcoef{m}{n}(\btheta)$ is defined through $\crl[\big]{\histcoef{j}{n}{k}(\btheta)}_{1 \leq j \leq m - 1, 1 \leq k \leq n}$.

So, we prove~\eqref{eq:nomem-coef-bounds} by induction in $m$. For $m = 1$ it is easy to check. To prove that the statement holds for $m$ assuming that it holds for all $m' < m$, we notice that since there is exponential averaging in the definition of $\nomemcoef{m}{n}(\btheta)$ and the derivatives of the mini-batch losses are bounded, we have
\begin{align*}
  &\abs{\nomemcoef{m}{n}(\btheta)} \lesssim \sum_{k = 1}^n \beta^k \sum_{\substack{i, l \geq 0\\i + l = m - 1}} \sum_{(i_0, \ldots, i_l) \in \mathcal{K}_{i, l}} k^{i_0 + \ldots + (l + 1) i_l}\\
  &\quad \lesssim \sum_{k = 1}^n \beta^k \sum_{\substack{i, l \geq 0\\i + l = m - 1}} k^{i + l} \lesssim \sum_{k = 1}^n \beta^k k^{m - 1} = O(1).
\end{align*}
Adding $\nabla^{r}$ in front of $\nomemcoef{m}{n}(\btheta)$ influences only the number of terms in the sum (but it remains bounded) and constant coefficients in the bounds above (the induction assumption is chosen such that we can still apply it if we add derivatives to the history terms), so the statement $\nabla^{r} \nomemcoef{m}{n}(\btheta) = O(1)$ also holds.

Using this and bounds on the previous history terms (and their derivatives), we can also see
\begin{align*}
&\abs{\histcoef{m}{n}{a}(\btheta)} \lesssim \sum_{s = 1}^a \sum_{\substack{j \geq 1, i, l \geq 0\\i + j + l = m}} \sum_{(k_0, \ldots, k_l) \in \mathcal{K}_{i, l}} s^{k_0 + \ldots + (l + 1) k_l}\\
&\quad \lesssim \sum_{s = 1}^a \sum_{\substack{j \geq 1, i, l \geq 0\\i + j + l = m}} s^{i + l} \lesssim \sum_{s = 1}^a (1 + \ldots + s^{m - 1}) \lesssim \sum_{s = 1}^a s^{m - 1} \lesssim a^m.
\end{align*}
Again, adding $\nabla^{r}$ in front of $\histcoef{m}{n}{a}(\btheta)$ only influences the number of terms in the sum (though it remains bounded) and constant coefficients in these bounds.
This proves $\nabla^{r} \histcoef{m}{n}{a}(\btheta) = O(a^m)$, completing the induction step.
\end{proof}

\begin{lemma}\label{lem:qVNejQ}
We have $\tilde{\btheta}^{(n - k)} - \tilde{\btheta}^{(n)} = O(k h)$ for any $k \in \range{1}{n}$.
\end{lemma}

\begin{proof}
This follows from the fact that the right-hand side in \eqref{eq:memoryless-approx-contrib} is $O(h)$, because $\nomemcoef{j}{n}(\btheta) = O(1)$ by \cref{lem:liuyf} and
\begin{equation*}
\sum_{j = 0}^{\ord - 1} h^j = \frac{1 - h^{\ord}}{1 - h} \leq \frac{1}{1 - h} \labrel{1}[=] O(1),
\end{equation*}
where \labrel{1} is because $h$ is bounded away from $1$.
\end{proof}

Recall that, by definition, \eqref{eq:onestep-error-bound} will follow from the following bound \eqref{eq:kwPSbq} on the error introduced by removing memory.
By Taylor's theorem, we have
\begin{align*}
  &\nabla L^{(n - k)}\prn[\big]{\tilde{\btheta}^{(n - k)}} = \sum_{i = 0}^{\ord - 1} \frac{1}{i!} \nabla^{i + 1} L^{(n - k)}\prn[\big]{\tilde{\btheta}^{(n)}} \prn[\big]{\underbrace{\tilde{\btheta}^{(n - k)} - \tilde{\btheta}^{(n)}, \ldots, \tilde{\btheta}^{(n - k)} - \tilde{\btheta}^{(n)}}_{\text{$i$ times}}} + \text{Rem}^{(n - k)},
\end{align*}
where
\begin{multline*}
  \text{Rem}^{(n - k)} = \frac{1}{(\ord - 1) !} \int_0^1 (1 - t)^{\ord - 1}\\
  \times \nabla^{\ord + 1} L^{(n - k)}\prn[\big]{\tilde{\btheta}^{(n)} + t \prn[\big]{\tilde{\btheta}^{(n - k)} - \tilde{\btheta}^{(n)}}} \prn[\big]{\underbrace{\tilde{\btheta}^{(n - k)} - \tilde{\btheta}^{(n)}, \ldots, \tilde{\btheta}^{(n - k)} - \tilde{\btheta}^{(n)}}_{\text{$\ord$ times}}}\,\mathrm{d} t.
\end{multline*}
Using the boundedness of derivatives of $L^{(n - k)}(\cdot)$ and \cref{lem:qVNejQ}, we can write
\begin{align*}
  &\nabla L^{(n - k)}\prn[\big]{\tilde{\btheta}^{(n - k)}} = \sum_{i = 0}^{\ord - 1} \frac{1}{i!} \nabla^{i + 1} L^{(n - k)}\prn[\big]{\tilde{\btheta}^{(n)}} \prn[\big]{\underbrace{\tilde{\btheta}^{(n - k)} - \tilde{\btheta}^{(n)}, \ldots, \tilde{\btheta}^{(n - k)} - \tilde{\btheta}^{(n)}}_{\text{$i$ times}}} + O(k^{\ord} h^{\ord}).\numberthis\label{eq:oiriy}
\end{align*}

We will now express the difference $\tilde{\btheta}^{(n - k)} - \tilde{\btheta}^{(n)}$ in \eqref{eq:oiriy} through $\tilde{\btheta}^{(n)}$.

\begin{proof}[Proof of \cref{lem:history-expansion}]
We induct over $r$.

For $r = 1$ this is the assertion of \cref{lem:qVNejQ}. Further, $r \geq 2$.

Now make the induction assumption that \eqref{eq:lkjru} holds with $r$ replaced by $r - j$ for any $j \in \range{1}{r - 1}$.

Like in the proof of \cref{lem:qVNejQ}, using $\nomemcoef{r}{n}(\btheta) = O(1)$ by \cref{lem:liuyf} and the fact that $h$ is bounded away from $1$, we write for $s \in \range{1}{k}$
\begin{equation*}
\tilde{\btheta}^{(n - s + 1)} = \tilde{\btheta}^{(n - s)} + \sum_{j = 1}^{r - 1} h^j \nomemcoef{j}{n - s}\prn[\big]{\tilde{\btheta}^{(n - s)}} + O(h^r).
\end{equation*}
Inserting this, we get
\begin{align*}
  &\tilde{\btheta}^{(n - k)} - \tilde{\btheta}^{(n)} = \sum_{s = 1}^k \crl[\big]{\tilde{\btheta}^{(n - s)} - \tilde{\btheta}^{(n - s + 1)}}\\
  &\quad = - \sum_{s = 1}^k \crl[\bigg]{\sum_{j = 1}^{r - 1} h^j \nomemcoef{j}{n - s}\prn[\big]{\tilde{\btheta}^{(n - s)}} + O(h^r)}\\
  &\quad = - \sum_{s = 1}^k \sum_{j = 1}^{r - 1} h^j \nomemcoef{j}{n - s}\prn[\big]{\tilde{\btheta}^{(n - s)}} + O(k h^r).\numberthis\label{eq:wpoii}
\end{align*}

By Taylor's theorem,
\begin{align*}
  \nomemcoef{j}{n - s}\prn[\big]{\tilde{\btheta}^{(n - s)}}
  &= \sum_{i = 0}^{r - 1 - j} \frac{1}{i!} \nabla^i \nomemcoef{j}{n - s}\prn[\big]{\tilde{\btheta}^{(n)}} \prn[\big]{\underbrace{\tilde{\btheta}^{(n - s)} - \tilde{\btheta}^{(n)}}_{\text{$i$ times}}}\\
  &\quad + \frac{1}{(r - 1 - j)!} \int_{0}^1 (1 - t)^{r - 1 - j} \\
  &\qquad \times \nabla^{r - j} \nomemcoef{j}{n - s}\prn[\big]{\tilde{\btheta}^{(n)} + t \prn[\big]{\tilde{\btheta}^{(n - s)} - \tilde{\btheta}^{(n)}}} \prn[\big]{\underbrace{\tilde{\btheta}^{(n - s)} - \tilde{\btheta}^{(n)}}_{\text{$r - j$ times}}}\,\mathrm{d} t\\
  &\labrel{1}[=] \sum_{i = 0}^{r - 1 - j} \frac{1}{i!} \nabla^i \nomemcoef{j}{n - s}\prn[\big]{\tilde{\btheta}^{(n)}} \prn[\big]{\underbrace{\tilde{\btheta}^{(n - s)} - \tilde{\btheta}^{(n)}}_{\text{$i$ times}}} + O\prn[\big]{s^{r - j} h^{r - j}}\\
  &\labrel{2}[=] \sum_{i = 0}^{r - 1 - j} \frac{1}{i!} \nabla^i \nomemcoef{j}{n - s}(\tilde{\btheta}^{(n)}) \prn[\big]{\underbrace{\tilde{\btheta}^{(n - s)} - \tilde{\btheta}^{(n)}}_{\text{$i$ times}}} + O\prn[\big]{k^{r - j} h^{r - j}},\numberthis\label{eq:kjlkfy}
\end{align*}
where in \labrel{1} we used that the derivatives of $\nomemcoef{j}{s}(\btheta)$ are bounded (\cref{lem:liuyf}) and \cref{lem:qVNejQ}; \labrel{2} is just because $s \leq k$.
By the induction assumption,
\begin{equation*}
\tilde{\btheta}^{(n - s)} = \tilde{\btheta}^{(n)} + \sum_{l = 1}^{r - 1 - j} h^l \histcoef{l}{n}{s}\prn[\big]{\tilde{\btheta}^{(n)}} + O(k^{r - j} h^{r - j}),
\end{equation*}
which we insert into \eqref{eq:kjlkfy}, giving
\begin{align*}
 &\nomemcoef{j}{n - s}\prn[\big]{\tilde{\btheta}^{(n - s)}}\\
 &\quad = \sum_{i = 0}^{r - 1 - j} \frac{1}{i!} \nabla^i \nomemcoef{j}{n - s}\prn[\big]{\tilde{\btheta}^{(n)}} \prn[\bigg]{\underbrace{\sum_{l = 1}^{r - 1 - j} h^l \histcoef{l}{n}{s}\prn[\big]{\tilde{\btheta}^{(n)}} + O(k^{r - j} h^{r - j})}_{\text{$i$ times}}} + O\prn[\big]{k^{r - j} h^{r - j}}\\
 &\quad \labrel{1}[=] \sum_{i = 0}^{r - 1 - j} \frac{1}{i!} \nabla^i \nomemcoef{j}{n - s}\prn[\big]{\tilde{\btheta}^{(n)}} \prn[\bigg]{\underbrace{\sum_{l = 1}^{r - 1 - j} h^l \histcoef{l}{n}{s}\prn[\big]{\tilde{\btheta}^{(n)}}}_{\text{$i$ times}}} + O\prn[\big]{k^{r - j} h^{r - j}}\\
 &\quad = \sum_{i = 0}^{r - 1 - j} \frac{h^i}{i!} \nabla^i \nomemcoef{j}{n - s}\prn[\big]{\tilde{\btheta}^{(n)}} \prn[\bigg]{\underbrace{\sum^{r - 2 - j}_{l = 0} h^l \histcoef{l + 1}{n}{s}\prn[\big]{\tilde{\btheta}^{(n)}}}_{\text{$i$ times}}} + O\prn[\big]{k^{r - j} h^{r - j}}\\
 &\quad \labrel{2}[=] \sum_{i = 0}^{r - 1 - j} \sum_{l = 0}^{r - 1 - j - i} h^{i + l}\\
 &\qquad\quad \times \sum_{(i_0, \ldots, i_l) \in \mathcal{K}_{i, l}} \frac{1}{i_0! \ldots i_l!} \nabla^i \nomemcoef{j}{n - s}\prn[\big]{\tilde{\btheta}^{(n)}} \prn[\bigg]{\underbrace{\histcoef{1}{n}{s}\prn[\big]{\tilde{\btheta}^{(n)}}}_{\text{$i_0$ times}}, \ldots, \underbrace{\histcoef{l + 1}{n}{s}\prn[\big]{\tilde{\btheta}^{(n)}}}_{\text{$i_l$ times}}}\\
 &\qquad + O\prn[\big]{k^{r - j} h^{r - j}}
\end{align*}
where in \labrel{1} we used that components of $\nabla^i \nomemcoef{j}{n - s}(\btheta)$ are bounded (\cref{lem:liuyf}) and for $i \geq 1$ we have $(k h)^{i (r - j)} = O \prn[\big]{(k h)^{r - j}}$ because $k h$ does not exceed $T$\internalComment{!!}; in \labrel{2} we also used that $\histcoef{l}{n}{s}(\btheta) = O(s^l)$ by \cref{lem:liuyf}.
Inserting this into \eqref{eq:wpoii} gives
\begin{align*}
 &\tilde{\btheta}^{(n - k)} - \tilde{\btheta}^{(n)}\\
 &\quad = - \sum_{s = 1}^k \sum_{j = 1}^{r - 1} h^j \nomemcoef{j}{n - s}\prn[\big]{\tilde{\btheta}^{(n - s)}} + O(k h^{r})\\
 &\quad = - \sum_{s = 1}^k \sum_{j = 1}^{r - 1} h^j \sum_{i = 0}^{r - 1 - j} \sum_{l = 0}^{r - 1 - j - i} h^{i + l} \times \\
 &\qquad \quad \times \sum_{(i_0, \ldots, i_l) \in \mathcal{K}_{i, l}} \frac{1}{i_0! \ldots i_l!} \nabla^i \nomemcoef{j}{n - s}\prn[\big]{\tilde{\btheta}^{(n)}} \prn[\bigg]{\underbrace{\histcoef{1}{n}{s}\prn[\big]{\tilde{\btheta}^{(n)}}}_{\text{$i_0$ times}}, \ldots, \underbrace{\histcoef{l + 1}{n}{s}\prn[\big]{\tilde{\btheta}^{(n)}}}_{\text{$i_l$ times}}}\\
 &\qquad - \sum_{s = 1}^k \sum_{j = 1}^{r - 1} h^j O\prn[\big]{k^{r - j} h^{r - j}}
   + O(k h^{r})\\
 &\quad = - \sum_{m = 1}^{r - 1} h^m \sum_{s = 1}^k \sum_{\substack{j \geq 1, i, l \geq 0\\i + j + l = m}} \sum_{(i_0, \ldots, i_l) \in \mathcal{K}_{i, l}} \frac{1}{i_0! \ldots i_l!} \nabla^i \nomemcoef{j}{n - s}\prn[\big]{\tilde{\btheta}^{(n)}} \prn[\bigg]{\underbrace{\histcoef{1}{n}{s}\prn[\big]{\tilde{\btheta}^{(n)}}}_{\text{$i_0$ times}}, \ldots, \underbrace{\histcoef{l + 1}{n}{s}\prn[\big]{\tilde{\btheta}^{(n)}}}_{\text{$i_l$ times}}}\\
 &\qquad + O(k^{r} h^{r})\\
 &\quad \labrel{1}[=] \sum_{m = 1}^{r - 1} h^m \histcoef{m}{n}{k}\prn[\big]{\tilde{\btheta}^{(n)}} + O(k^{r} h^{r}),
\end{align*}
where in \labrel{1} we used \eqref{eq:jfjl}.
We have completed the induction step for \eqref{eq:lkjru}.
\internalComment{
For example,
\begin{equation*}
\histcoef{3}{n}{k}(\btheta) = - \sum_{s = 1}^k \brk[\bigg]{\nomemcoef{3}{n - s}(\btheta) + \nabla \nomemcoef{2}{n - s}(\btheta) \histcoef{1}{n}{s}(\btheta) + \nabla \nomemcoef{1}{n - s}(\btheta) \histcoef{2}{n}{s}(\btheta) + \frac{1}{2} \nabla^2 \nomemcoef{1}{n - s}(\btheta) \prn[\big]{\histcoef{1}{n}{s}(\btheta), \histcoef{1}{n}{s}(\btheta)}}.
\end{equation*}
}
\end{proof}

We will now conclude the proof of~\eqref{eq:kwPSbq}. Inserting \eqref{eq:lkjru} into \eqref{eq:oiriy} gives
\begin{align*}
 &\nabla L^{(n - k)}\prn[\big]{\tilde{\btheta}^{(n - k)}}\\
 &\quad = \sum_{i = 0}^{\ord - 1} \frac{1}{i!} \nabla^{i + 1} L^{(n - k)}\prn[\big]{\tilde{\btheta}^{(n)}} \prn[\bigg]{\underbrace{\sum_{m = 1}^{\ord - 1} h^m \histcoef{m}{n}{k}\prn[\big]{\tilde{\btheta}^{(n)}} + O(k^{\ord} h^{\ord})}_{\text{$i$ times}}} + O(k^{\ord} h^{\ord})\\
 &\quad \labrel{1}[=] \sum_{i = 0}^{\ord - 1} \frac{1}{i!} \nabla^{i + 1} L^{(n - k)}\prn[\big]{\tilde{\btheta}^{(n)}} \prn[\bigg]{\underbrace{\sum_{m = 1}^{\ord - 1} h^m \histcoef{m}{n}{k}\prn[\big]{\tilde{\btheta}^{(n)}}}_{\text{$i$ times}}} + O(k^{\ord} h^{\ord})\\
 &\quad = \sum_{i = 0}^{\ord - 1} \frac{h^i}{i!} \nabla^{i + 1} L^{(n - k)}\prn[\big]{\tilde{\btheta}^{(n)}} \prn[\bigg]{\underbrace{\sum_{m = 0}^{\ord - 2} h^m \histcoef{m + 1}{n}{k}\prn[\big]{\tilde{\btheta}^{(n)}}}_{\text{$i$ times}}} + O(k^{\ord} h^{\ord})\\
 &\quad = \sum_{m = 0}^{\ord - 1} h^m \sum_{\substack{i, l \geq 0\\i + l = m}} \sum_{(i_0, \ldots, i_l) \in \mathcal{K}_{i, l}} \frac{1}{i_0! \ldots i_l!} \nabla^{i + 1} L^{(n - k)}\prn[\big]{\tilde{\btheta}^{(n)}} \prn[\bigg]{\underbrace{\histcoef{1}{n}{k}\prn[\big]{\tilde{\btheta}^{(n)}}}_{\text{$i_0$ times}}, \ldots, \underbrace{\histcoef{l + 1}{n}{k}\prn[\big]{\tilde{\btheta}^{(n)}}}_{\text{$i_l$ times}}}\\
  &\qquad + O(k^{\ord} h^{\ord})
\end{align*}
where in \labrel{1} we used that for $i \geq 1$ we have $(k h)^{i \ord} = O \prn[\big]{(k h)^{\ord}}$ because $k h$ does not exceed $T$.\internalComment{!!} Using exponential summation gives
\begin{align*}
 &- \sum_{k = 0}^n \beta^k \nabla L^{(n - k)}\prn[\big]{\tilde{\btheta}^{(n - k)}}\\
 &\quad = - \nabla L^{(n)}\prn[\big]{\tilde{\btheta}^{(n)}}\\
 &\qquad - \sum_{m = 0}^{\ord - 1} h^m \sum_{k = 1}^n \beta^k \sum_{\substack{i, l \geq 0\\i + l = m}} \sum_{(i_0, \ldots, i_l) \in \mathcal{K}_{i, l}} \frac{1}{i_0! \ldots i_l!} \times\\
  &\qquad \qquad \qquad \times \nabla^{i + 1} L^{(n - k)}\prn[\big]{\tilde{\btheta}^{(n)}} \prn[\bigg]{\underbrace{\histcoef{1}{n}{k}\prn[\big]{\tilde{\btheta}^{(n)}}}_{\text{$i_0$ times}}, \ldots, \underbrace{\histcoef{l + 1}{n}{k}\prn[\big]{\tilde{\btheta}^{(n)}}}_{\text{$i_l$ times}}} + O(h^{\ord})\\
 &\quad = - \sum_{k = 0}^n \beta^k \nabla L^{(n - k)}\prn[\big]{\tilde{\btheta}^{(n)}}\\
 &\qquad - \sum_{m = 1}^{\ord - 1} h^m \sum_{k = 1}^n \beta^k \sum_{\substack{i, l \geq 0\\i + l = m}} \sum_{(i_0, \ldots, i_l) \in \mathcal{K}_{i, l}} \frac{1}{i_0! \ldots i_l!} \times \\
  &\qquad \qquad \qquad \times \nabla^{i + 1} L^{(n - k)}\prn[\big]{\tilde{\btheta}^{(n)}} \prn[\bigg]{\underbrace{\histcoef{1}{n}{k}\prn[\big]{\tilde{\btheta}^{(n)}}}_{\text{$i_0$ times}}, \ldots, \underbrace{\histcoef{l + 1}{n}{k}\prn[\big]{\tilde{\btheta}^{(n)}}}_{\text{$i_l$ times}}} + O(h^{\ord})\\
 &\quad = \sum_{m = 0}^{\ord - 1} h^m \nomemcoef{m + 1}{n}\prn[\big]{\tilde{\btheta}^{(n)}} + O(h^{\ord}),
\end{align*}
as desired. Since \eqref{eq:kwPSbq} was sufficient for \eqref{eq:onestep-error-bound}, we have proven \eqref{eq:onestep-error-bound} as well.

\begin{lemma}\label{lem:global-error-bound}
\Cref{eq:global-error-bound} follows from \eqref{eq:onestep-error-bound}.
\end{lemma}

\begin{proof}
The argument is standard. Define the error at the $n$th step
\begin{equation*}
\be^{(n)} := \btheta^{(n)} - \tilde{\btheta}^{(n)}
\end{equation*}
and local error
\begin{equation*}
\bdelta^{(n)} := \tilde{\btheta}^{(n + 1)} - \tilde{\btheta}^{(n)} + h \sum_{k = 0}^n \beta^k \nabla L^{(n - k)}\prn[\big]{\tilde{\btheta}^{(n - k)}},
\end{equation*}
where $n \in \range{0}{\lfloor T / h \rfloor}$. By definition and Taylor's theorem, we have
\begin{align*}
  \be^{(n + 1)} &= \be^{(n)} - h \sum_{k = 0}^n \beta^k \prn[\Big]{
                  \nabla L^{(n - k)}\prn[\big]{\btheta^{(n - k)}} - \nabla L^{(n - k)}\prn[\big]{\tilde{\btheta}^{(n - k)}}} - \bdelta^{(n)}\\
  &= \be^{(n)} - h \sum_{k = 0}^n \beta^k \int_0^1 \nabla^2 L^{(n - k)}\prn[\big]{\tilde{\btheta}^{(n - k)} + t \prn[\big]{\btheta^{(n - k)} - \tilde{\btheta}^{(n - k)}} } \prn[\big]{\btheta^{(n - k)} - \tilde{\btheta}^{(n - k)}} \,\mathrm{d} t - \bdelta^{(n)}.
\end{align*}

Since the derivatives of the loss are bounded, this implies
\begin{equation*}
\norm{\be^{(n + 1)}} \leq \norm{\be^{(n)}} + h \oC{lglpr} \sum_{k = 0}^n \beta^k \norm{\be^{(n - k)}} + \norm{\bdelta^{(n)}}
\end{equation*}
with some constant $\nC{lglpr}$. Denoting $s^{(n)} := \max_{0 \leq k \leq n} \norm{\be^{(k)}}$, we have
\begin{equation*}
  s^{(n + 1)} \leq s^{(n)} + h \oC{lgl} s^{(n)} + \norm{\bdelta^{(n)}} = (1 + h \oC{lgl}) s^{(n)} + \norm{\bdelta^{(n)}} \labrel{1}[\leq] (1 + h \oC{lgl}) s^{(n)} + \oC{clocal} h^{\ord + 1},
\end{equation*}
with some constant $\nC{lgl}$, where \labrel{1} is by \eqref{eq:onestep-error-bound}. Applying this inequality iteratively, we obtain
\begin{align*}
  s^{(n)} &\leq (1 + h \oC{lgl})^n s^{(0)} + \frac{(1 + h \oC{lgl})^n - 1}{\oC{lgl}} \oC{clocal} h^{\ord}\\
          &= \frac{(1 + h \oC{lgl})^n - 1}{\oC{lgl}} \oC{clocal} h^{\ord} \labrel{1}[\leq] \frac{e^{\oC{lgl} n h} - 1}{\oC{lgl}} \oC{clocal} h^{\ord} \leq \oC{jjqh} e^{\oC{lgl} n h} h^{\ord} \leq \oC{jjqh} e^{\oC{lgl} T} h^{\ord},
\end{align*}
where in \labrel{1} we used the inequality $1 + x \leq e^x$ for any $x \geq 0$, $\nC{jjqh}$ is some constant. It is left to apply it with $n = \lfloor T / h \rfloor$ and put $\oC{cglobal} = \oC{jjqh} e^{\oC{lgl} T}$.
\end{proof}

We have proven both \eqref{eq:onestep-error-bound} and \eqref{eq:global-error-bound}, concluding the proof of \cref{th:approx}.

\section{Proof of \cref{th:nomem-coef}}\label{sec:omitted-details-nomem-coef}

\begin{proof}[Proof of \cref{lem:sum-over-markings-of-one-tree}]
We prove this by induction in $m$.

\paragraph{Induction base}
The singleton tree has only the empty marking, and
\begin{align*}
  \Emark{\rootedtree[]}{\varnothing}{a}{n - l \to n}
  &= \E{\rootedtree[]}{a}{n - l}
  \\
  &= \sum_{b = 0}^{n - l - a} \beta^b \nabla L^{(n - l - a - b)}
  = \E{\rootedtree[]}{l + a}{n}.
\end{align*}

\paragraph{Induction step} Let $r$ be the root of $\tau$, $v_1, \ldots, v_{\ell}$ its children, $\tau_1, \ldots, \tau_{\ell}$ the corresponding subtrees rooted at the children, $V_1, \ldots, V_{\ell}$ their vertex sets (partitioning $\range{1}{m} \setminus \crl{r}$). From the definition,
\begin{equation*}
\Emark{\tau}{\mrk}{a}{n - l \to n} = \sum_{b = 0}^{n - l - a} \beta^b \nabla^{\ell + 1} L^{(n - l - a - b)} \brk[\bigg]{\sum_{l' = 1}^l \E{\tau_s}{l'}{n} \inde{\text{$v_s \in \mrk$}} + \sum_{l' = 1}^{a + b} \Emark{\tau_s}{\mrk \cap V_s}{l'}{n - l \to n} \inde{\text{$v_s \notin \mrk$}}}_{s = 1}^{\ell},
\end{equation*}
which means
\begin{align*}
  &\sum_{\substack{\mrk \in \mrkset_{\tau}\\v_1 \notin \mrk}} \Emark{\tau}{\mrk}{a}{n - l \to n} = \sum_{\substack{\mrk \in \mrkset_{\tau}\\v_1 \notin \mrk}} \sum_{b = 0}^{n - l - a} \beta^b \times \\
  &\qquad \quad \times \nabla^{\ell + 1} L^{(n - l - a - b)} \brk[\bigg]{
  \sum_{l' = 1}^{a + b} \Emark{\tau_1}{\mrk \cap V_1}{l'}{n - l \to n},
  \ldots,
  \sum_{l' = 1}^l \E{\tau_\ell}{l'}{n} \inde{\text{$v_\ell \in \mrk$}} + \sum_{l' = 1}^{a + b} \Emark{\tau_{\ell}}{\mrk \cap V_\ell}{l'}{n - l \to n} \inde{\text{$v_\ell \notin \mrk$}}
  }\\
  &\quad =  \sum_{\mrk_1 \in \mrkset_{\tau_1}} \sum_{\mrk_2 \in \mrkset_{\tau \setminus \tau_1}} \sum_{b = 0}^{n - l - a} \beta^b \times \\
  &\qquad \quad \times \nabla^{\ell + 1} L^{(n - l - a - b)} \brk[\bigg]{
    \sum_{l' = 1}^{a + b} \Emark{\tau_1}{\mrk_1}{l'}{n - l \to n},
    \ldots,
    \sum_{l' = 1}^l \E{\tau_\ell}{l'}{n} \inde{\text{$v_\ell \in \mrk_2$}} + \sum_{l' = 1}^{a + b} \Emark{\tau_{\ell}}{\mrk_2 \cap V_{\ell}}{l'}{n - l \to n} \inde{\text{$v_\ell \notin \mrk_2$}}
    }\\
  &\quad = \sum_{\mrk_2 \in \mrkset_{\tau \setminus \tau_1}} \sum_{b = 0}^{n - l - a} \beta^b \times\\
  &\qquad \quad \times \nabla^{\ell + 1} L^{(n - l - a - b)} \brk[\bigg]{\sum_{l' = 1}^{a + b} \sum_{\mrk_1 \in \mrkset_{\tau_1}} \Emark{\tau_1}{\mrk_1}{l'}{n - l \to n}, \ldots, \sum_{l' = 1}^l \E{\tau_\ell}{l'}{n} \inde{\text{$v_\ell \in \mrk_2$}} + \sum_{l' = 1}^{a + b} \Emark{\tau_{\ell}}{\mrk_2 \cap V_\ell}{l'}{n - l \to n} \inde{\text{$v_\ell \notin \mrk_2$}}}.
\end{align*}
Here, we used the one-to-one correspondence between markings $\mrk$ not containing $v_1$ and pairs $(\mrk_1, \mrk_2)$, where $\mrk_1$ is a marking of $\tau_1$ and $\mrk_2$ is a marking of the tree $\tau \setminus \tau_1$.
Now we apply the induction hypothesis (since $\tau_1$ is a smaller tree) and replace $\sum_{\mrk_1 \in \mrkset_{\tau_1}} \Emark{\tau_1}{\mrk_1}{l'}{n - l \to n}$ with $\E{\tau_1}{l + l'}{n}$:
\begin{align*}
  &\sum_{\substack{\mrk \in \mrkset_{\tau}\\v_1 \notin \mrk}} \Emark{\tau}{\mrk}{a}{n - l \to n} = \sum_{\mrk_2 \in \mrkset_{\tau \setminus \tau_1}} \sum_{b = 0}^{n - l - a} \beta^b\\
  &\qquad \times \nabla^{\ell + 1} L^{(n - l - a - b)} \brk[\bigg]{\sum_{l' = 1}^{a + b} \E{\tau_1}{l + l'}{n}, \ldots, \sum_{l' = 1}^l \E{\tau_\ell}{l'}{n} \inde{\text{$v_\ell \in \mrk_2$}} + \sum_{l' = 1}^{a + b} \Emark{\tau_{\ell}}{\mrk_2 \cap V_{\ell}}{l'}{n - l \to n} \inde{\text{$v_\ell \notin \mrk_2$}}}.
\end{align*}

We deal with the case $v_1 \in \mrk$ similarly, obtaining
\begin{align*}
  &\sum_{\substack{\mrk \in \mrkset_{\tau}\\v_1 \in \mrk}} \Emark{\tau}{\mrk}{a}{n - l \to n} = \sum_{\mrk_2 \in \mrkset_{\tau \setminus \tau_1}} \sum_{b = 0}^{n - l - a} \beta^b\\
  &\qquad \times \nabla^{\ell + 1} L^{(n - l - a - b)} \brk[\bigg]{\sum_{l' = 1}^l \E{\tau_1}{l'}{n}, \ldots, \sum_{l' = 1}^l \E{\tau_\ell}{l'}{n} \inde{\text{$v_\ell \in \mrk_2$}} + \sum_{l' = 1}^{a + b} \Emark{\tau_{\ell}}{\mrk_2 \cap V_\ell}{l'}{n - l \to n} \inde{\text{$v_\ell \notin \mrk_2$}}}.
\end{align*}
Adding the latter two equations and using $\sum_{l' = 1}^{a + b} \E{\tau_1}{l + l'}{n} + \sum_{l' = 1}^l \E{\tau_1}{l'}{n} = \sum_{l' = 1}^{l + a + b} \E{\tau_1}{l'}{n}$ gives
\begin{align*}
  &\sum_{\substack{\mrk \in \mrkset_{\tau}}} \Emark{\tau}{\mrk}{a}{n - l \to n} = \sum_{\mrk_2 \in \mrkset_{\tau \setminus \tau_1}} \sum_{b = 0}^{n - l - a} \beta^b\\
  &\qquad \times \nabla^{\ell + 1} L^{(n - l - a - b)} \brk[\bigg]{\sum_{l' = 1}^{l + a + b} \E{\tau_1}{l'}{n}, \ldots, \sum_{l' = 1}^l \E{\tau_\ell}{l'}{n} \inde{\text{$v_\ell \in \mrk_2$}} + \sum_{l' = 1}^{a + b} \Emark{\tau_{\ell}}{\mrk_2 \cap V_\ell}{l'}{n - l \to n} \inde{\text{$v_\ell \notin \mrk_2$}}}.
\end{align*}
Continuing the same argument with $v_2, \ldots, v_{\ell}$, we will arrive at
\begin{align*}
  \sum_{\substack{\mrk \in \mrkset_{\tau}}} \Emark{\tau}{\mrk}{a}{n - l \to n} = \sum_{b = 0}^{n - l - a} \beta^b \nabla^{\ell + 1} L^{(n - l - a - b)} \brk[\bigg]{
    \sum_{l' = 1}^{l + a + b} \E{\tau_1}{l'}{n},
    \ldots,
    \sum_{l' = 1}^{l + a + b} \E{\tau_\ell}{l'}{n}
    }.
\end{align*}
By definition, the right-hand side is equal to $\E{\tau}{l + a}{n}$, completing the induction step.
\end{proof}

\begin{proof}[Proof of \cref{lem:kjjhhag}]
Consider the following operation. Choose a partition of $\range{1}{m}$ into $i + 1$ disjoint non-empty sets $(V_0, \crl{V_1, \ldots, V_i})$, where $0 \leq i \leq m - 1$ (one of the sets is privileged, but the order in the other ones does not matter); choose a labeled rooted tree $\tau_0 \in \labrooted[V_0]$ with root $v_0 \in V_0$, another labeled rooted tree $\tau_1 \in \labrooted[V_1]$ with root $v_1 \in V_1$, and so on, up $\tau_i \in \labrooted[V_i]$ with root $v_i \in V_i$.

Now, assign to each of $v_1, \ldots, v_i$ a parent among vertices of $\tau_0$ in all possible ways, by choosing all mappings from $\crl{v_1, \ldots, v_i}$ to $V_0$. Write the following expression corresponding to the family of all such assignments:
\begin{equation}\label{eq:expr-corr-to-parent-assignm}
\nabla^i \E{\tau_0}{a}{n - l} \brk[\bigg]{\sum_{l' = 1}^l \E{\tau_1}{l'}{n}, \ldots, \sum_{l' = 1}^l \E{\tau_i}{l'}{n}}.
\end{equation}
This way, we have constructed a family of labeled rooted trees $\tau \in \labrooted\range{1}{m}$ with roots of $\tau_1, \ldots, \tau_i$ marked.

For example, consider $m = 3$. When $i = 0$, there is only one partition $V_0 = \crl{1, 2, 3}$ and nine corresponding marked trees (with no vertices marked):
\begin{gather*}
  \bigrootedtree[1[2][3]]\quad \bigrootedtree[1[2[3]]]\quad \bigrootedtree[1[3[2]]]
  \\
  \bigrootedtree[2[3][1]]\quad \bigrootedtree[2[1[3]]]\quad \bigrootedtree[2[3[1]]]
  \\
  \bigrootedtree[3[1][2]]\quad \bigrootedtree[3[1[2]]]\quad \bigrootedtree[3[2[1]]]
\end{gather*}
When $i = 1$, there are six partitions listed below with corresponding marked trees
  \begin{align*}
    &V_0 = \crl{1}, V_1 = \crl{2, 3}\quad \bigrootedtree[1[2,text=red[3]]]\quad \bigrootedtree[1[3,text=red[2]]]\\
    &V_0 = \crl{2}, V_1 = \crl{1, 3}\quad \bigrootedtree[2[1,text=red[3]]]\quad \bigrootedtree[2[3,text=red[1]]]\\
    &V_0 = \crl{3}, V_1 = \crl{1, 2}\quad \bigrootedtree[3[1,text=red[2]]]\quad \bigrootedtree[3[2,text=red[1]]]\\
    &V_0 = \crl{1, 2}, V_1 = \crl{3}\quad
      \bigrootedtree[1[2][3,text=red]]\quad
      \bigrootedtree[1[2[3,text=red]]]\quad
      \bigrootedtree[2[1][3,text=red]]\quad
      \bigrootedtree[2[1[3,text=red]]]\\
    &V_0 = \crl{1, 3}, V_1 = \crl{2}\quad
      \bigrootedtree[1[3][2,text=red]]\quad
      \bigrootedtree[1[3[2,text=red]]]\quad
      \bigrootedtree[3[1][2,text=red]]\quad
      \bigrootedtree[3[1[2,text=red]]]\\
    &V_0 = \crl{2, 3}, V_1 = \crl{1}\quad
      \bigrootedtree[2[3][1,text=red]]\quad
      \bigrootedtree[2[3[1,text=red]]]\quad
      \bigrootedtree[3[2][1,text=red]]\quad
      \bigrootedtree[3[2[1,text=red]]]
  \end{align*}
When $i = 2$, there are three partitions with corresponding marked trees
\begin{align*}
  &V_0 = \crl{1}, V_1 = \crl{2}, V_2 = \crl{3}\quad \bigrootedtree[1[2,text=red][3,text=red]]\\
  &V_0 = \crl{2}, V_1 = \crl{1}, V_2 = \crl{3}\quad \bigrootedtree[2[1,text=red][3,text=red]]\\
  &V_0 = \crl{3}, V_1 = \crl{1}, V_2 = \crl{2}\quad \bigrootedtree[3[1,text=red][2,text=red]]
\end{align*}

Sum the expression \eqref{eq:expr-corr-to-parent-assignm} over the choices of a partition and the trees $\tau_0, \ldots, \tau_i$:
\begin{equation*}
\sum_{i = 0}^{m - 1} \sum_{\substack{(V_0, \crl{V_1\ldots, V_i})\\\text{partition of $\range{1}{m}$}}} \sum_{\tau_0 \in \labrooted[V_0], \ldots, \tau_i \in \labrooted[V_i]} \nabla^i \E{\tau_0}{a}{n - l} \brk[\bigg]{\sum_{l' = 1}^l \E{\tau_1}{l'}{n}, \ldots, \sum_{l' = 1}^l \E{\tau_i}{l'}{n}}.
\end{equation*}
Since all marked trees in $\labrooted\range{1}{m}$ can be constructed this way and are counted exactly once in this sum, this equals
\begin{equation*}
\sum_{\tau \in \labrooted\range{1}{m}} \sum_{\mrk \in \mrkset_{\tau}} \Emark{\tau}{\mrk}{a}{n - l \to n}.
\end{equation*}
But $\sum_{\mrk \in \mrkset_{\tau}} \Emark{\tau}{\mrk}{a}{n - l \to n}$ is $\E{\tau}{l + a}{n}$ by \eqref{eq:sum-over-markings-of-one-tree}. Hence, we have obtained
\begin{equation*}
\sum_{i = 0}^{m - 1} \sum_{\substack{(V_0, \crl{V_1\ldots, V_i})\\\text{partition of $\range{1}{m}$}}} \sum_{\tau_0 \in \labrooted[V_0], \ldots, \tau_i \in \labrooted[V_i]} \nabla^i \E{\tau_0}{a}{n - l} \brk[\bigg]{\sum_{l' = 1}^l \E{\tau_1}{l'}{n}, \ldots, \sum_{l' = 1}^l \E{\tau_i}{l'}{n}} = \sum_{\tau \in \labrooted\range{1}{m}} \E{\tau}{l + a}{n}.
\end{equation*}
Recall that $m! / \sigma(\tau)$ labeled rooted trees correspond to the same unlabeled rooted tree.
Hence, this can be rewritten as
\begin{equation}\label{eq:ljdshh}
  \begin{aligned}
    &\sum_{i = 0}^{m - 1} \sum_{\substack{(V_0, \crl{V_1\ldots, V_i})\\\text{partition of $\range{1}{m}$}}} \sum_{\tau_0 \in \unlabrooted[\abs{V_0}], \ldots, \tau_i \in \unlabrooted[\abs{V_i}]} \frac{\abs{V_0}!}{\sigma(\tau_0)}\\
    &\qquad \times \nabla^i \E{\tau_0}{a}{n - l} \brk[\bigg]{\frac{\abs{V_1}!}{\sigma(\tau_1)} \sum_{l' = 1}^l \E{\tau_1}{l'}{n}, \ldots, \frac{\abs{V_i}!}{\sigma(\tau_i)} \sum_{l' = 1}^l \E{\tau_i}{l'}{n}} = \sum_{\tau \in \unlabrooted[m]} \frac{m!}{\sigma(\tau)} \E{\tau}{l + a}{n}.
    \end{aligned}
\end{equation}

Let us fix $j$, $(k_0, \ldots, k_{m - j - i}) \in \mathcal{K}_{i, m - j - i}$ and count the number of partitions with $\abs{V_0} = j$, where among $\crl{V_1, \ldots, V_i}$ there are $k_0$ sets of size $1$, $\ldots$, $k_{m - j - i}$ sets of size $m - j - i + 1$. First, we choose the elements of $V_0$ in $\binom{m}{j}$ ways, then we order the remaining elements in $(m - j)!$ ways and assign the first $k_0$ elements in this ordering as singletons, the next $k_1$ pairs as two-sets, and so on. Notice that each partition of $V \setminus V_0$ will be counted $k_0! \ldots k_{m - j - i}! 1!^{k_0} \ldots (m - j - i + 1)!^{k_{m - j - i}}$ times (because the order of sets with the same length does not matter, and the order within each set does not matter). So, the required number of partitions is
\begin{align*}
  &\binom{m}{j} \frac{(m - j)!}{k_0! \ldots k_{m - j - i}! 1!^{k_0} \ldots (m - j - i + 1)!^{k_{m - j - i}}}\\
  &\quad =  \frac{m!}{j! k_0! \ldots k_{m - j - i}! 1!^{k_0} \ldots (m - j - i + 1)!^{k_{m - j - i}}}.
\end{align*}
Thus, \eqref{eq:ljdshh} can be rewritten as
\begin{align*}
  &\sum_{i = 0}^{m - 1} \sum_{j = 1}^{m - i} \sum_{(k_0, \ldots, k_{m - j - i}) \in \mathcal{K}_{i, m - j - i}} \frac{m!}{k_0! \ldots k_{m - j - i}!} \sum_{\tau_0 \in \unlabrooted[j]} \frac{1}{\sigma(\tau_0)} \\
  &\qquad \times \nabla^i \E{\tau_0}{a}{n - l} \brk[\bigg]{
  \underbrace{\sum_{\tau \in \unlabrooted[1]}\frac{1}{\sigma(\tau)} \sum_{l' = 1}^l \E{\tau}{l'}{n}}_{\text{$k_0$ times}},
  \ldots,
  \underbrace{\sum_{\tau \in \unlabrooted[m - j - i + 1]}\frac{1}{\sigma(\tau)} \sum_{l' = 1}^l \E{\tau}{l'}{n}}_{\text{$k_{m - j - i}$ times}}
  } = \sum_{\tau \in \unlabrooted[m]} \frac{m!}{\sigma(\tau)} \E{\tau}{l + a}{n},
\end{align*}
where $j! 1!^{k_0} \ldots (m - j - i + 1)!^{k_{m - j - i}}$ canceled out with $\abs{V_0}! \abs{V_1}! \ldots \abs{V_i}!$. Dividing both sides by $m!$ completes the proof.
\end{proof}

\begin{proof}[Proof of \cref{th:nomem-coef}]
We prove \eqref{eq:nomemcoef-m-n-ind} and \eqref{eq:histcoef-m-n-k-ind} by induction over $m \geq 2$.

For $m = 2$, they are already verified above. Note also that the second statement holds for $m = 1$ as well:
\begin{equation*}
\histcoef{1}{n}{k} = \sum_{l = 1}^k \E{\rootedtree[]}{l}{n}.
\end{equation*}

By definition,
\begin{align*}
  &\nomemcoef{m}{n}
  =
  - \beta \sum_{b = 0}^{n - 1} \beta^b \sum_{\ell = 0}^{m - 1} \sum_{(i_0, \ldots, i_{m - 1 - \ell}) \in \mathcal{K}_{\ell, m - 1 - \ell}} \frac{1}{i_0! \ldots i_{m - 1 - \ell}!} \times \\
  &\qquad \qquad \times \nabla^{\ell + 1} L^{(n - 1 - b)} \prn[\bigg]{\underbrace{\histcoef{1}{n}{b + 1}}_{\text{$i_0$ times}}, \ldots, \underbrace{\histcoef{m - \ell}{n}{b + 1}}_{\text{$i_{m - 1 - \ell}$ times}}}
\end{align*}
Insert the induction hypothesis (recall that \eqref{eq:histcoef-m-n-k-ind} holds for $m = 1$ too):
\begin{align*}
  &\nomemcoef{m}{n}
  =
  - \beta \sum_{b = 0}^{n - 1} \beta^b \sum_{\ell = 0}^{m - 1} \sum_{(i_0, \ldots, i_{m - 1 - \ell}) \in \mathcal{K}_{\ell, m - 1 - \ell}} \frac{1}{i_0! \ldots i_{m - 1 - \ell}!}\\
  &\qquad \times \nabla^{\ell + 1} L^{(n - 1 - b)} \prn[\bigg]{\underbrace{\sum_{l = 1}^{b + 1} \sum_{\tau \in \unlabrooted[1]} \frac{1}{\sigma(\tau)} \E{\tau}{l}{n}}_{\text{$i_0$ times}}, \ldots, \underbrace{\sum_{l = 1}^{b + 1} \sum_{\tau \in \unlabrooted[m - \ell]} \frac{1}{\sigma(\tau)} \E{\tau}{l}{n}}_{\text{$i_{m - 1 - \ell}$ times}}}
\end{align*}

The sum over $\unlabrooted[s]$, repeated $i_{s - 1}$ times, generates a list of $i_{s - 1}$-tuples of trees with $s$ vertices. On this list, each multiset of $i_{s - 1}$ trees with multiplicities $\mu_1^s, \ldots, \mu^s_{\abs{\unlabrooted[s]}}$ appears $\binom{i_{s - 1}}{\mu_1^s \ldots \mu^s_{\abs{\unlabrooted[s]}}}$ times.
Therefore, in the large sum above each multiset of $\ell$ trees with the total number of vertices $m - 1$ and matching the vertex-count multiplicities $i_0, \ldots, i_{m - 1 - \ell}$ (equivalently, each tree $\tau$ with $m$ vertices whose root has $\ell$ children matching these vertex-count multiplicities) appears
\begin{equation*}
\prod_{s = 1}^{m - \ell} \frac{i_{s - 1}!}{\mu_1^s(\tau)! \ldots \mu_{\abs{\unlabrooted[s]}}^s(\tau)!} = \frac{i_0! \ldots i_{m - 1 - \ell}!}{\prod_{s = 1}^{m - \ell} \prn[\big]{\mu_1^s(\tau)! \ldots \mu_{\abs{\unlabrooted[s]}}^s(\tau)!}}
\end{equation*}
times, and we can rewrite
\begin{align*}
  &\nomemcoef{m}{n}
  =
    - \beta \sum_{b = 0}^{n - 1} \beta^b \sum_{\ell = 0}^{m - 1} \sum_{(i_0, \ldots, i_{m - 1 - \ell}) \in \mathcal{K}_{\ell, m - 1 - \ell}}
\sum_{\substack{\tau = [\tau_1, \ldots, \tau_{\ell}] \in \unlabrooted[m]:\\ \text{$\tau$ matches $i_0, \ldots, i_{m - 1 - \ell}$}}}
    \frac{1}{\prod_{s = 1}^{m - \ell} \prn[\big]{\mu_1^s(\tau)! \ldots \mu_{\abs{\unlabrooted[s]}}^s(\tau)!}}\\
  &\qquad \times \nabla^{\ell + 1} L^{(n - 1 - b)} \prn[\bigg]{\frac{1}{\sigma(\tau_1)} \sum_{l = 1}^{b + 1} \E{\tau_1}{l}{n}, \ldots, \frac{1}{\sigma(\tau_{\ell})} \sum_{l = 1}^{b + 1} \E{\tau_{\ell}}{l}{n}}.
\end{align*}
Using \eqref{eq:recursive-formula-for-sym-coef} simplifies this further:
\begin{align*}
  &\nomemcoef{m}{n}
  =
    - \beta \sum_{b = 0}^{n - 1} \beta^b \sum_{\ell = 0}^{m - 1} \sum_{(i_0, \ldots, i_{m - 1 - \ell}) \in \mathcal{K}_{\ell, m - 1 - \ell}}
\sum_{\substack{\tau = [\tau_1, \ldots, \tau_{\ell}] \in \unlabrooted[m]:\\ \text{$\tau$ matches $i_0, \ldots, i_{m - 1 - \ell}$}}}
    \frac{1}{\sigma(\tau)}\\
  &\qquad \times \nabla^{\ell + 1} L^{(n - 1 - b)} \prn[\bigg]{\sum_{l = 1}^{b + 1} \E{\tau_1}{l}{n}, \ldots, \sum_{l = 1}^{b + 1} \E{\tau_{\ell}}{l}{n}}\\
  &\quad = - \beta \sum_{b = 0}^{n - 1} \beta^b \sum_{\ell = 0}^{m - 1}
    \sum_{\substack{\tau = [\tau_1, \ldots, \tau_{\ell}] \in \unlabrooted[m]}} \frac{1}{\sigma(\tau)} \nabla^{\ell + 1} L^{(n - 1 - b)} \prn[\bigg]{\sum_{l = 1}^{b + 1} \E{\tau_1}{l}{n}, \ldots, \sum_{l = 1}^{b + 1} \E{\tau_{\ell}}{l}{n}}\\
  &\quad = - \beta \sum_{\tau \in \unlabrooted[m]} \frac{1}{\sigma(\tau)} \E{\tau}{1}{n}.
\end{align*}

We have proven that under the induction hypothesis for smaller $m$, \eqref{eq:nomemcoef-m-n-ind} holds.

By definition of the history terms in \eqref{eq:jfjl} and the induction hypothesis, to prove \eqref{eq:histcoef-m-n-k-ind} it is enough to show \eqref{eq:aux-fact-in-nonmem-coef}.

By the induction hypothesis and \eqref{eq:nomemcoef-m-n-ind} already proven, the left-hand side of \eqref{eq:aux-fact-in-nonmem-coef} is
\begin{align}
  &\sum_{i = 0}^{m - 1}
    \sum_{(k_0, \ldots, k_{m - i - 1}) \in \mathcal{K}_{i, m - i - 1}} \frac{1}{k_0! \ldots k_{m - i - 1}!}
  \label{eq:yyhqt1}\\
  &\qquad \times \sum_{b = 0}^{n - l} \beta^b \nabla^{i + 1} L^{(n - l - b)}
  \brk[\bigg]{
    \underbrace{
      \sum_{l' = 1}^l \sum_{\tau \in \unlabrooted[1]} \frac{1}{\sigma(\tau)} \E{\tau}{l'}{n}
    }_{\text{$k_0$ times}},
    \ldots,
    \underbrace{
      \sum_{l' = 1}^l \sum_{\tau \in \unlabrooted[m - i]} \frac{1}{\sigma(\tau)} \E{\tau}{l'}{n}
    }_{\text{$k_{m - i - 1}$ times}}
    }
  \label{eq:yyhqt2}\\
  &\qquad + \beta
  \sum_{j = 2}^m \sum_{i = 0}^{m - j}
    \sum_{(k_0, \ldots, k_{m - i - j}) \in \mathcal{K}_{i, m - i - j}} \frac{1}{k_0! \ldots k_{m - i - j}!}
  \label{eq:yyhqt3}\\
  &\qquad \times \nabla^i \sum_{\tau_0 \in \unlabrooted[j]} \frac{1}{\sigma(\tau_0)} \E{\tau_0}{1}{n - l}
  \brk[\bigg]{
    \underbrace{
      \sum_{l' = 1}^l \sum_{\tau \in \unlabrooted[1]} \frac{1}{\sigma(\tau)} \E{\tau}{l'}{n}
    }_{\text{$k_0$ times}},
    \ldots,
    \underbrace{
      \sum_{l' = 1}^l \sum_{\tau \in \unlabrooted[m - i - j + 1]} \frac{1}{\sigma(\tau)} \E{\tau}{l'}{n}
    }_{\text{$k_{m - i - j}$ times}}
    }
  \label{eq:yyhqt4}\\
  &\quad = \sum_{i = 0}^{m - 1}
      \sum_{(k_0, \ldots, k_{m - i - 1}) \in \mathcal{K}_{i, m - i - 1}} \frac{1}{k_0! \ldots k_{m - i - 1}!}
 \label{eq:yyhqt5}\\
  &\qquad \times \sum_{b = 0}^{n - l} \beta^b \nabla^{i + 1} L^{(n - l - b)}
  \brk[\bigg]{
    \underbrace{
      \sum_{l' = 1}^l \sum_{\tau \in \unlabrooted[1]} \frac{1}{\sigma(\tau)} \E{\tau}{l'}{n}
    }_{\text{$k_0$ times}},
    \ldots,
    \underbrace{
      \sum_{l' = 1}^l \sum_{\tau \in \unlabrooted[m - i]} \frac{1}{\sigma(\tau)} \E{\tau}{l'}{n}
    }_{\text{$k_{m - i - 1}$ times}}
    }
  \label{eq:yyhqt6}\\
  &\qquad - \beta
  \sum_{i = 0}^{m - 1}
    \sum_{(k_0, \ldots, k_{m - i - 1}) \in \mathcal{K}_{i, m - i - 1}} \frac{1}{k_0! \ldots k_{m - i - 1}!}
  \label{eq:yyhqt7}\\
  &\qquad \times \nabla^i \E{\rootedtree[]}{1}{n - l}
  \brk[\bigg]{
    \underbrace{
      \sum_{l' = 1}^l \sum_{\tau \in \unlabrooted[1]} \frac{1}{\sigma(\tau)} \E{\tau}{l'}{n}
    }_{\text{$k_0$ times}},
    \ldots,
    \underbrace{
      \sum_{l' = 1}^l \sum_{\tau \in \unlabrooted[m - i]} \frac{1}{\sigma(\tau)} \E{\tau}{l'}{n}
    }_{\text{$k_{m - i - 1}$ times}}
    }
  \label{eq:yyhqt8}\\
  &\qquad + \beta
  \sum_{j = 1}^m \sum_{i = 0}^{m - j}
    \sum_{(k_0, \ldots, k_{m - i - j}) \in \mathcal{K}_{i, m - i - j}} \frac{1}{k_0! \ldots k_{m - i - j}!}
  \label{eq:yyhqt9}\\
  &\qquad \times \sum_{\tau_0 \in \unlabrooted[j]} \frac{1}{\sigma(\tau_0)} \nabla^i \E{\tau_0}{1}{n - l}
  \brk[\bigg]{
    \underbrace{
      \sum_{l' = 1}^l \sum_{\tau \in \unlabrooted[1]} \frac{1}{\sigma(\tau)} \E{\tau}{l'}{n}
    }_{\text{$k_0$ times}},
    \ldots,
    \underbrace{
      \sum_{l' = 1}^l \sum_{\tau \in \unlabrooted[m - i - j + 1]} \frac{1}{\sigma(\tau)} \E{\tau}{l'}{n}
    }_{\text{$k_{m - i - j}$ times}}
    }.\label{eq:yyhqt10}
\end{align}
Inserting $\E{\rootedtree[]}{1}{n - l} = \sum_{b = 0}^{n - l - 1} \beta^b \nabla L^{(n - l - 1 - b)}$, we see that the sum in \cref{eq:yyhqt5,eq:yyhqt6,eq:yyhqt7,eq:yyhqt8} evaluates to
\begin{align*}
  &\sum_{i = 0}^{m - 1}
    \sum_{(k_0, \ldots, k_{m - i - 1}) \in \mathcal{K}_{i, m - i - 1}} \frac{1}{k_0! \ldots k_{m - i - 1}!}
  \\
  &\qquad \times \nabla^{i + 1} L^{(n - l)}
    \brk[\bigg]{
    \underbrace{
    \sum_{l' = 1}^l \sum_{\tau \in \unlabrooted[1]} \frac{1}{\sigma(\tau)} \E{\tau}{l'}{n}
    }_{\text{$k_0$ times}},
    \ldots,
    \underbrace{
    \sum_{l' = 1}^l \sum_{\tau \in \unlabrooted[m - i]} \frac{1}{\sigma(\tau)} \E{\tau}{l'}{n}
    }_{\text{$k_{m - i - 1}$ times}}
    }\\
  &\quad = \sum_{i = 0}^{m - 1}
    \sum_{\substack{\tau = [\tau_1, \ldots, \tau_i] \in \unlabrooted[m]}} \frac{1}{\sigma(\tau)} \nabla^{i + 1} L^{(n - l)}
    \brk[\bigg]{
    \sum_{l' = 1}^l \E{\tau_1}{l'}{n},
    \ldots,
    \sum_{l' = 1}^l \E{\tau_i}{l'}{n}
  },
\end{align*}
By \cref{lem:kjjhhag}, the sum in \cref{eq:yyhqt9,eq:yyhqt10} is
\begin{equation*}
\beta \sum_{\tau \in \unlabrooted[m]} \frac{1}{\sigma(\tau)} \E{\tau}{l + 1}{n}.
\end{equation*}

We have obtained that the left-hand side of \eqref{eq:aux-fact-in-nonmem-coef} is equal to
\begin{align*}
\sum_{i = 0}^{m - 1}
    \sum_{\substack{\tau \in \unlabrooted[m]\\ \tau = [\tau_1, \ldots, \tau_i]}} \frac{1}{\sigma(\tau)} \nabla^{i + 1} L^{(n - l)}
    \brk[\bigg]{
    \sum_{l' = 1}^l \E{\tau_1}{l'}{n},
    \ldots,
    \sum_{l' = 1}^l \E{\tau_i}{l'}{n}
  }
  +
  \beta \sum_{\tau \in \unlabrooted[m]} \frac{1}{\sigma(\tau)} \E{\tau}{l + 1}{n}.
\end{align*}
Combining this with \eqref{eq:nJdufq}, we see that the left-hand side of \eqref{eq:aux-fact-in-nonmem-coef} is equal to the right-hand side of \eqref{eq:aux-fact-in-nonmem-coef}. This completes the induction step and the whole proof of
\cref{th:nomem-coef}.
\end{proof}

\section{Proof of Corollaries}

\subsection{Proof of \cref{cor:modified-eq}}\label{sec:proof-of-modified-eq}

\begin{lemma}\label{lem:bea-terms-are-bounded}
  For all $\btheta \in \mathcal{D}$ and all $m \in \range{1}{\ord}$, we have
\begin{equation*}
\nabla^{r} \beacoef{m}{n}(\btheta) = O(1),\quad r \in \range{0}{2 \ord - m}.
\end{equation*}
\end{lemma}

\begin{proof}
If we express $\beacoef{m}{n}(\btheta)$ through $\crl[\big]{\nomemcoef{j}{n}}$ and their derivatives, the derivatives will be of the form $\nabla^l \nomemcoef{j}{n}$ where $l + j \leq m$. Therefore, the derivatives in $\nabla^{r} \beacoef{m}{n}(\btheta)$ for $r \in \range{0}{2 \ord - m}$ will be of the form $\nabla^l \nomemcoef{j}{n}$ where $l + j \leq 2 \ord$. So the result follows immediately from \cref{lem:liuyf}.
\end{proof}

\begin{lemma}\label{lem:hb-no-memory-iteration}
For all $r \in \range{1}{\ord}$
  \begin{equation}\label{eq:hb-no-memory-iteration}
    \btheta(t_{n + 1}) = \btheta(t_n) + \sum_{j = 1}^{r} h^j \nomemcoef{j}{n}(\btheta(t_n)) + O(h^{r + 1}).
  \end{equation}
\end{lemma}

\begin{proof}
Differentiating, we get the exact equality on $t \in [t_n, t_{n + 1}]$
\begin{equation*}
  \frac{\mathrm{d}^i \btheta}{\mathrm{d} t^i}(t) = \sum_{k_1, \ldots, k_i = 1}^{\ord} h^{k_1 + \ldots + k_i - i}
  (D_{k_1}^{(n)} \ldots D_{k_{i - 1}}^{(n)} \beacoef{k_i}{n}) (\btheta(t)),\quad1 \leq i \leq \ord + 1.
\end{equation*}

Taylor's theorem gives
\begin{align*}
  &\btheta(t_{n + 1})
    = \btheta(t_n)
    + \sum_{i = 1}^{\ord} \frac{h^i}{i!}
    \frac{\mathrm{d}^i \btheta}{\mathrm{d} t^i}(t_n^+)
    + \frac{h^{\ord + 1}}{\ord !} \int_0^1 (1 - s)^{\ord}
    \frac{\mathrm{d}^{\ord + 1} \btheta} {\mathrm{d} t^{\ord + 1}}
    (t_n + s h)\,\mathrm{d} s.
  \\
  &\quad
    = \btheta(t_n)
    + \sum_{i = 1}^{\ord} \frac{1}{i!}
    \sum_{k_1, \ldots, k_i = 1}^{\ord} h^{k_1 + \ldots + k_i}
    (D_{k_1}^{(n)} \ldots D_{k_{i - 1}}^{(n)} \beacoef{k_i}{n}) (\btheta(t_n))
  \\
  &\qquad
    + \frac{1}{\ord !}
    \sum_{k_1, \ldots, k_{\ord + 1} = 1}^{\ord} h^{k_1 + \cdots + k_{\ord + 1}} \int_0^1
    (1 - s)^{\ord}
    \prn[\big]{D_{k_1}^{(n)} \cdots D_{k_{\ord}}^{(n)} \beacoef{k_{\ord + 1}}{n}}
    (\btheta(t_n + s h))\,\mathrm{d} s
  \\
  &\quad
    = \btheta(t_n)
    + \sum_{i = 1}^{\ord} \frac{1}{i!}
    \sum_{\substack{1 \leq k_1, \ldots, k_i \leq \ord
      \\
      i \leq k_1 + \ldots + k_i \leq \ord}} h^{k_1 + \ldots + k_i}
  (D_{k_1}^{(n)} \ldots D_{k_{i - 1}}^{(n)} \beacoef{k_i}{n}) (\btheta(t_n))
  \\
  &\qquad
    + \sum_{i = 1}^{\ord} \frac{1}{i!}
    \sum_{\substack{1 \leq k_1, \ldots, k_i \leq \ord
      \\
      \ord + 1 \leq k_1 + \ldots + k_i \leq i \ord}} h^{k_1 + \ldots + k_i}
  (D_{k_1}^{(n)} \ldots D_{k_{i - 1}}^{(n)} \beacoef{k_i}{n}) (\btheta(t_n))
  \\
  &\qquad
    + \frac{1}{\ord !}
    \sum_{k_1, \ldots, k_{\ord + 1} = 1}^{\ord} h^{k_1 + \cdots + k_{\ord + 1}} \int_0^1
    (1 - s)^{\ord}
    \prn[\big]{D_{k_1}^{(n)} \cdots D_{k_{\ord}}^{(n)} \beacoef{k_{\ord + 1}}{n}}
    (\btheta(t_n + s h))\,\mathrm{d} s
  \\
  &\quad \labrel{1}[=] \btheta(t_n)
    + \sum_{i = 1}^{\ord} \frac{1}{i!} \sum_{j = i}^{\ord} h^j
    \sum_{\substack{k_1, \ldots, k_i \geq 1
      \\
      k_1 + \ldots + k_i = j}}
  (D^{(n)}_{k_1} \ldots D^{(n)}_{k_{i - 1}} \beacoef{k_i}{n}) (\btheta(t_n))
  + O(h^{\ord + 1})
  \\
  &\quad
    = \btheta(t_n)
    + \sum_{j = 1}^{\ord} h^j \sum_{i = 1}^j \frac{1}{i!}
    \sum_{\substack{k_1, \ldots, k_i \geq 1
      \\
      k_1 + \ldots + k_i = j}}
  (D^{(n)}_{k_1} \ldots D^{(n)}_{k_{i - 1}} \beacoef{k_i}{n}) (\btheta(t_n))
  + O(h^{\ord + 1})
  \\
  &\quad \labrel{2}[=] \btheta(t_n)
    + \sum_{j = 1}^{\ord} h^j \nomemcoef{j}{n}(\btheta(t_n)) + O(h^{\ord + 1}),
\end{align*}
where in \labrel{1} we used \cref{lem:bea-terms-are-bounded} and in \labrel{2} we used~\eqref{eq:bea-coef-def}.

It is left to use the boundedness of $\nomemcoef{j}{n}(\btheta)$ in the region of interest (\cref{lem:liuyf})\internalComment{and that $h$ is bounded away from $1$}.
\end{proof}

\begin{lemma}\label{lem:yxdKct}
The following global bound holds:
\begin{equation*}
\sup_{n \in \range{0}{\lfloor T / h \rfloor}} \norm[\big]{\tilde{\btheta}^{(n)} - \btheta(t_n)} \leq \oC{mod-proof-gl} h^{\ord},
\end{equation*}
where $\nC{mod-proof-gl}$ is some constant.
\end{lemma}

\begin{proof}
Note that $\btheta(0) = \tilde{\btheta}^{(0)}$, the local error bound is already proven in \cref{lem:hb-no-memory-iteration}, and each $\nomemcoef{j}{n}(\cdot)$ is Lipschitz because their derivatives are bounded by \cref{lem:liuyf}.
So this result follows by the same standard argument as for \cref{lem:global-error-bound} with $\tilde{\btheta}^{(n)}$ replaced by $\btheta(t_n)$ and $\btheta^{(n)}$ replaced by $\tilde{\btheta}^{(n)}$.
\end{proof}

To conclude the proof of \cref{cor:modified-eq}, it is left to combine \cref{lem:yxdKct} with \eqref{eq:global-error-bound}.

\subsection{Proof of \cref{cor:principal-flow}}\label{sec:proof-of-principal-flow}

\begin{lemma}\label{lem:principal-terms-in-bea}
Neglecting non-principal terms, the following formula holds:
\begin{equation*}
\beacoef{m}{n}(\btheta) = \sum_{l = 1}^m \frac{(-1)^{l + 1}}{l} \sum_{\substack{k_1, \ldots, k_l \geq 1\\k_1 + \ldots + k_l = m}} \nabla \nomemcoef{k_1}{n} \ldots \nabla \nomemcoef{k_{l - 1}}{n} \nomemcoef{k_l}{n}(\btheta) + \np.
\end{equation*}
\end{lemma}

\begin{proof}
In other words, we need to show
\begin{align*}
  \beacoef{m}{n}(\btheta) &= \sum_{\mathcal{s}: w(\mathcal{s}) = m} C_{l(\mathcal{s})} \mathcal{s} + \np,
\end{align*}
with
\begin{equation}\label{eq:def-of-tree-coefs}
  C_l := \begin{cases}
    1 & \text{if $l = 1$,}\\
    (-1)^{l + 1} \sum_{i = 2}^l \frac{(-1)^i}{i!} \sum_{\substack{l_1, \ldots, l_i \geq 1\\l_1 + \ldots + l_i = l}} \frac{1}{l_1 \ldots l_i} & \text{if $l \geq 2$,}
  \end{cases}
\end{equation}
where by $\mathcal{s}$ we denote expressions of the form $\nabla \nomemcoef{k_1}{n} \ldots \nabla \nomemcoef{k_{l - 1}}{n} \nomemcoef{k_l}{n}(\btheta)$, by $w(\mathcal{s})$ their weight, which is defined as $k_1 + \ldots + k_l$, and by $l(\mathcal{s})$ their length, which is the number of nodes $\nomemcoef{k_i}{n}$ ($l$ in this case). For two such expressions $\mathcal{s}_1$ and $\mathcal{s}_2$, we will write $\mathcal{s}_1 \mathcal{s}_2$ for their concatenation, for example if $\mathcal{s}_1 = \nomemcoef{3}{n}(\btheta)$ and $\mathcal{s}_2 = \nabla \nomemcoef{1}{n}(\btheta) \nomemcoef{2}{n}(\btheta)$, then
$\mathcal{s}_1 \mathcal{s}_2 = \nabla \nomemcoef{3}{n}(\btheta) \nabla \nomemcoef{1}{n}(\btheta) \nomemcoef{2}{n}(\btheta)$.

We will argue by induction over $m$. For $m = 1$ the statement is obvious. Ignoring non-principal terms, we can write
\begin{equation*}
\beacoef{m}{n}(\btheta) = \nomemcoef{m}{n}(\btheta) - \sum_{i = 2}^m \frac{1}{i!} \sum_{\substack{k_1, \ldots, k_i \geq 1\\k_1 + \ldots + k_i = m}} \nabla \beacoef{k_1}{n} \ldots \nabla \beacoef{k_{i - 1}}{n} \beacoef{k_i}{n}(\btheta) + \np.
\end{equation*}
Now using the induction assumption, we rewrite it as
\begin{align*}
  \beacoef{m}{n}(\btheta) &= \nomemcoef{m}{n}(\btheta) - \sum_{i = 2}^m \frac{1}{i!} \sum_{\substack{k_1, \ldots, k_i \geq 1\\k_1 + \ldots + k_i = m}} \sum_{\substack{\mathcal{s}_1, \ldots, \mathcal{s}_i:\\w(\mathcal{s}_1) = k_1, \ldots, w(\mathcal{s}_i) = k_i}} C_{l(\mathcal{s}_1)} \ldots C_{l(\mathcal{s}_i)} \mathcal{s}_1 \ldots \mathcal{s}_i + \np\\
                     &\labrel{1}[=] \nomemcoef{m}{n}(\btheta) - \sum_{i = 2}^m \frac{1}{i!} \sum_{\substack{k_1, \ldots, k_i \geq 1\\k_1 + \ldots + k_i = m}} \sum_{\substack{\mathcal{s}_1, \ldots, \mathcal{s}_i:\\w(\mathcal{s}_1) = k_1, \ldots, w(\mathcal{s}_i) = k_i}} \frac{(-1)^{l(\mathcal{s}_1) + \ldots + l(\mathcal{s}_i) + i}}{l(\mathcal{s}_1) \ldots l(\mathcal{s}_i)} \mathcal{s}_1 \ldots \mathcal{s}_i + \np\\
                     &= \nomemcoef{m}{n}(\btheta) -  \sum_{\mathcal{s}: w(\mathcal{s}) = m} \sum_{i = 2}^m \frac{(-1)^i}{i!} \sum_{\substack{\mathcal{s}_1, \ldots, \mathcal{s}_i\text{ non-empty}:\\\mathcal{s} = \mathcal{s}_1\ldots \mathcal{s}_i}} \frac{(-1)^{l(\mathcal{s}_1) + \ldots + l(\mathcal{s}_i)}}{l(\mathcal{s}_1) \ldots l(\mathcal{s}_i)} \mathcal{s} + \np\\
  &\labrel{2}[=] \sum_{\mathcal{s}: w(\mathcal{s}) = m} C_{l(\mathcal{s})} \mathcal{s} + \np.
\end{align*}
where in \labrel{1} we used that $C_l = (-1)^{l + 1} / l$ which is proven below, in \labrel{2} we used the definition of $C_l$ given in \eqref{eq:def-of-tree-coefs}.

It is left to prove $C_l = (-1)^{l + 1} / l$, or, equivalently, for $l \geq 2$ we have
\begin{equation}\label{eq:jhfhnbc}
\sum_{i = 1}^l \frac{(-1)^i}{i!} \sum_{\substack{l_1, \ldots, l_i \geq 1\\l_1 + \ldots + l_i = l}} \frac{1}{l_1 \ldots l_i} \overset{?}{=} 0.
\end{equation}
To do this, note that
\begin{equation*}
- \ln(1 - x) = \sum_{n = 1}^{\infty} \frac{x^n}{n},
\end{equation*}
which means that the coefficient before $x^l$ in the power series of $\brk[\big]{- \ln(1 - x)}^i$ is
\begin{equation*}
\sum_{\substack{l_1, \ldots, l_i \geq 1\\l_1 + \ldots + l_i = l}} \frac{1}{l_1 \ldots l_i},
\end{equation*}
and therefore the left-hand side of~\eqref{eq:jhfhnbc} is the coefficient before $x^l$ in the power series
\begin{equation*}
\sum_{i = 1}^{\infty} \frac{(-1)^i}{i!} \brk[\big]{- \ln(1 - x)}^i = \exp \crl{\ln(1 - x)} - 1 = -x,
\end{equation*}
which is zero for $l \geq 2$.
\end{proof}

\Cref{eq:z-coef-rec} is immediate from \cref{lem:principal-terms-in-bea} and \cref{cor:principal-coef}. Taking the limit as $n \to \infty$,
\begin{equation*}
    \zcoef{m}{\infty} := \lim_{n \to \infty} \zcoef{m}{n} = \sum_{l = 1}^m \frac{(-1)^{l + 1}}{l} \sum_{\substack{k_1, \ldots, k_l \geq 1\\k_1 + \ldots + k_l = m}} p^{(\infty)}_{k_1} \ldots p^{(\infty)}_{k_l},
\end{equation*}
where
\begin{equation*}
    p^{(\infty)}_k := \begin{cases}
        - (1 - \beta)^{-1}, &\text{if } k = 1,\\
        - \frac{\beta e_{\mathcal{c}_k, 1}(\beta)}{(1 - \beta)^{2 k - 1}}, &\text{if } k \geq 2.
        \end{cases}
\end{equation*}

The generating function $\bar{g}_{\beta}(x) := \sum_{k = 0}^{\infty} \zcoef{k + 1}{\infty} x^k$ satisfies
\begin{equation*}
    \brk{\bar{g}_{\beta}(x)}_k = \sum_{l = 0}^k \frac{(-1)^l}{l + 1} \brk[\big]{\crl{-1 - \beta g_{\beta}(x)}^{l + 1}}_{k - l},
\end{equation*}
where $\brk{g(x)}_k$ denotes the coefficient before $x^k$ in the power series of $g(x)$. Multiplying both sides by $x^k$, summing over $k$ and changing the order of summation gives
\begin{align*}
  \sum_{k = 0}^{\infty} \brk{\bar{g}_{\beta}(x)}_k x^k &= - \sum_{k = 0}^{\infty} \sum_{l = 0}^k \frac{1}{l + 1} \brk[\big]{\crl{1 + \beta g_{\beta}(x)}^{l + 1}}_{k - l} x^k\\
                                           &= - \sum_{l = 0}^{\infty} \frac{1}{l + 1} x^l \sum_{k = l}^{\infty} \brk[\big]{\crl{1 + \beta g_{\beta}(x)}^{l + 1}}_{k - l} x^{k - l}\\
                                           &= - \sum_{l = 0}^{\infty} \frac{1}{l + 1} x^l \crl{1 + \beta g_{\beta}(x)}^{l + 1} = - \frac{1}{x} \sum_{l = 1}^{\infty} \frac{1}{l} \crl{x + \beta x g_{\beta}(x)}^l\\
  &= \frac{1}{x} \ln \prn[\bigg]{\frac{1 + \beta - x + \sqrt{(1 - \beta - x)^2 - 4 \beta x}}{2}},
\end{align*}
as desired.

\section{Averaging over Dataset Permutations}

Let $\batchcntDef := n + 1$ be the batch count.

Recall that
\begin{equation*}
L^{(k)}(\btheta) = \frac{1}{\batchsize} \sum_{r = k \batchsize + 1}^{k \batchsize + \batchsize} \ell_{\pi(r)}(\btheta),\quad k \in \range{0}{n},
\end{equation*}
where $\crl{\ell_s}_{s = 1}^{(n + 1) \batchsize}$ are per-sample losses and $\pi$ is a random permutation of $\range{1}{(n + 1) \batchsize}$, distributed uniformly over all $((n + 1) \batchsize)!$ such permutations.

Note that
\begin{align*}
  \Eletpi \nabla^2 (\ell_{\pi(1)} - L) \nabla (\ell_{\pi(2)} - L) &= - \frac{\nabla \trace \mSigma}{2 (\batchcnt \batchsize - 1)},\\
  \Eletpi \nabla^2 (\ell_{\pi(1)} - L) \nabla (\ell_{\pi(1)} - L) &= \frac{\nabla \trace \mSigma}{2},
\end{align*}
where $\mSigma$ is the empirical covariance matrix of per-sample gradients \eqref{eq:empirical-covariance}.

\begin{lemma}\label{lem:lpaUyF}
The following expressions hold:
\begin{equation*}
  \Eletpi \nabla^2 L^{(0)} \nabla L^{(1)} = \nabla^2 L \nabla L - \frac{\nabla \trace \mSigma}{2 (\batchcnt \batchsize - 1)},\quad
  \Eletpi \nabla^2 L^{(0)} \nabla L^{(0)} = \nabla^2 L \nabla L + \frac{\batchcnt - 1}{2 (\batchcnt \batchsize - 1)} \nabla \trace \mSigma.
\end{equation*}
\end{lemma}

\begin{proof}
Indeed,
\begin{align*}
  &\Eletpi \nabla^2 L^{(0)} \nabla L^{(1)} = \Eletpi \frac{1}{\batchsize} \sum_{r = 1}^{\batchsize} \nabla^2 \ell_{\pi(r)} \frac{1}{\batchsize} \sum_{s = \batchsize + 1}^{2 \batchsize} \nabla \ell_{\pi(s)}\\
  &\quad = \Eletpi \nabla^2 \ell_{\pi(1)} \nabla \ell_{\pi(2)} = \nabla^2 L \nabla L + \Eletpi \nabla^2 (\ell_{\pi(1)} - L) \nabla (\ell_{\pi(2)} - L)\\
  &\quad = \nabla^2 L \nabla L - \frac{\nabla \trace \mSigma}{2 (\batchcnt \batchsize - 1)}
\end{align*}
and
\begin{align*}
  &\Eletpi \nabla^2 L^{(0)} \nabla L^{(0)} = \Eletpi \frac{1}{\batchsize} \sum_{r = 1}^{\batchsize} \nabla^2 \ell_{\pi(r)} \frac{1}{\batchsize} \sum_{s = 1}^{\batchsize} \nabla \ell_{\pi(s)}\\
  &\quad = \frac{1}{\batchsize} \Eletpi \nabla^2 \ell_{\pi(1)} \nabla \ell_{\pi(1)} + \frac{\batchsize - 1}{\batchsize} \Eletpi \nabla^2 \ell_{\pi(1)} \nabla \ell_{\pi(2)}\\
  &\quad = \frac{1}{\batchsize} \prn[\bigg]{\nabla^2 L \nabla L + \frac{\nabla \trace \mSigma}{2}} + \frac{\batchsize - 1}{\batchsize} \prn[\bigg]{\nabla^2 L \nabla L - \frac{\nabla \trace \mSigma}{2 (\batchcnt \batchsize - 1)}}\\
  &\quad = \nabla^2 L \nabla L + \frac{\batchcnt - 1}{2 (\batchcnt \batchsize - 1)} \nabla \trace \mSigma
\end{align*}
completing the proof.
\end{proof}

\begin{lemma}\label{lem:SwTWil}
We have
\begin{align*}
  \Eletpi \nomemcoef{2}{n}(\btheta) &= \prn[\big]{\twoeq(\beta) + \twoneq(\beta)} \nabla^2 L \nabla L + \prn[\bigg]{\twoeq(\beta) \frac{\batchcnt - 1}{2 (\batchcnt \batchsize - 1)} - \twoneq(\beta) \frac{1}{2 (\batchcnt \batchsize - 1)}} \nabla \trace \mSigma\\
  &= \prn[\bigg]{- \frac{\beta}{(1 - \beta)^3} + o_n(1)} \nabla^2 L \nabla L + \prn[\bigg]{- \frac{\beta}{2 (1 - \beta)^2 (1 + \beta)} + o_n(1)} \frac{\nabla \trace \mSigma}{\batchsize},
\end{align*}
where $o_n(1)$ are terms that go to zero as $n \to \infty$ (for fixed $\beta$) regardless of $\batchsize \in \range{1}{n + 1}$, and
\begin{align}
  \twoeq(\beta) &:= - \frac{\beta \brk{1 - \beta^n (1 + \beta) + \beta^{2 n + 1}}}{(1 - \beta)^2 (1 + \beta)},\label{eq:twoeq-def}\\
  \twoneq(\beta) &:= \frac{- 2 \beta^2 + 2 n (1 - \beta^2) \beta^{n + 1} + 2 \beta^{2 n + 2}}{(1 - \beta)^3 (1 + \beta)}.\label{eq:twoneq-def}
\end{align}
\end{lemma}

\begin{proof}
Using \cref{lem:lpaUyF}, we can write
\begin{align*}
  &\Eletpi \nomemcoef{2}{n}(\btheta) = - \beta \Eletpi \sum_{b = 0}^{n - 1} \beta^{b} \sum_{l' = 1}^{b + 1} \sum_{b' = 0}^{n - l'} \beta^{b'} \nabla^2 L^{(n - 1 - b)} \nabla L^{(n - l' - b')}\\
  &\quad = \twoeq(\beta) \Eletpi \nabla^2 L^{(0)} \nabla L^{(0)} + \twoneq(\beta) \Eletpi \nabla^2 L^{(0)} \nabla L^{(1)}\\
  &\quad = \twoeq(\beta) \prn[\bigg]{\nabla^2 L \nabla L + \frac{\batchcnt - 1}{2 (\batchcnt \batchsize - 1)} \nabla \trace \mSigma} + \twoneq(\beta) \prn[\bigg]{\nabla^2 L \nabla L - \frac{\nabla \trace \mSigma}{2 (\batchcnt \batchsize - 1)}}\\
  &\quad = \prn[\big]{\twoeq(\beta) + \twoneq(\beta)} \nabla^2 L \nabla L + \prn[\bigg]{\twoeq(\beta) \frac{\batchcnt - 1}{2 (\batchcnt \batchsize - 1)} - \twoneq(\beta) \frac{1}{2 (\batchcnt \batchsize - 1)}} \nabla \trace \mSigma\\
  &\quad = \prn[\bigg]{- \frac{\beta}{(1 - \beta)^3} + o_n(1)} \nabla^2 L \nabla L + \prn[\bigg]{- \frac{\beta}{2 (1 - \beta)^2 (1 + \beta)} + o_n(1)} \frac{\nabla \trace \mSigma}{\batchsize},
\end{align*}
where $\twoeq(\beta)$ and $\twoneq(\beta)$ can be calculated as
\begin{align*}
  \twoeq(\beta) &:= - \beta \sum_{b = 0}^{n - 1} \beta^b \sum_{l' = 1}^{b + 1} \beta^{b + 1 - l'} = - \frac{\beta \brk{1 - \beta^n (1 + \beta) + \beta^{2 n + 1}}}{(1 - \beta)^2 (1 + \beta)} = \underset{n \to \infty}{\longrightarrow} - \frac{\beta}{(1 - \beta)^2 (1 + \beta)},\\
  \twoneq(\beta) &:= - \beta \sum_{b = 0}^{n - 1} \beta^b \sum_{l' = 1}^{b + 1} \sum_{b' = 0}^{n - l'} \beta^{b'} - \twoeq(\beta)\\
  &= \frac{- 2 \beta^2 + 2 n (1 - \beta^2) \beta^{n + 1} + 2 \beta^{2 n + 2}}{(1 - \beta)^3 (1 + \beta)} \underset{n \to \infty}{\longrightarrow} - \frac{2 \beta^2}{(1 - \beta)^3 (1 + \beta)}.\qedhere
\end{align*}
\end{proof}

\printbibliography

\end{document}